\documentclass{article}

\usepackage[english]{babel}

\usepackage[letterpaper,top=2cm,bottom=2cm,left=3cm,right=3cm,marginparwidth=1.75cm]{geometry}
\usepackage{parskip}

\usepackage{algorithm}
\usepackage{algorithmic}
\usepackage{booktabs}

\usepackage{amsmath}
\usepackage{amsfonts,amssymb}
\usepackage{graphicx}
\usepackage[colorlinks=true, allcolors=blue]{hyperref}
\usepackage{natbib}
\usepackage{amsthm}
\usepackage[T1]{fontenc}
\usepackage{xcolor}

\usepackage{makecell}
\usepackage{multirow}
\usepackage{booktabs}
\usepackage{threeparttable}
\usepackage{bm}
\usepackage{wrapfig}

\usepackage{subfigure}

\newtheorem{lemma}{Lemma}
\newtheorem{theorem}{Theorem}

\newtheorem{remark}{Remark}

\ifx\assumption\undefined
\newtheorem{assumption}{Assumption}

\fi
\theoremstyle{remark}

\newcommand{\Norm}[1]{\left\|#1\right\|}

\newcommand{\dotprod}[2]{\left\langle#1,#2\right\rangle}
\newcommand{\Abs}[1]{\left\vert #1 \right\vert}
\newcommand{\tr}[1]{{\rm tr}\left( #1 \right)}

\renewcommand{\vec}[1]{{\rm vec}\left( #1 \right)}
\newcommand{\argmax}[1]{\underset{ #1 }{\mathrm{argmax}}\, }
\newcommand{\argmin}[1]{\underset{ #1 }{\mathrm{argmin}}\, }

\newcommand{\EE}{\mathbb{E}}
\newcommand{\cO}{\mathcal{O}}

\newcommand{\RR}{\mathbb{R}}

\newcommand{\cA}{\mathcal{A}}
\newcommand{\cL}{\mathcal{L}}

\newcommand{\Lp}{\mathcal{L}_{\mathrm{pre}}}
\newcommand{\Lforget}{\mathcal{L}_{\mathrm{forget}}}
\newcommand{\Lforgeth}{\hat{\mathcal{L}}_{\mathrm{forget}}}
\newcommand{\Lsft}{\mathcal{L}_{\mathrm{SFT}}}

\allowdisplaybreaks[4]

\title{\textbf{Optimizer-Model Consistency: Full Finetuning with the Same Optimizer as Pretraining Forgets Less}}

\author{Yuxing Liu$^{1}$ \qquad Jianyu Wang$^{2}$ \qquad Tong Zhang$^{1}$ \\ \\
{
$^{1}$UIUC}\footnote{All experiments were conducted by the university.} 
\quad {$^{2}$Apple } 
}

\date{}

\begin{document}
\maketitle

\begin{abstract}
    \noindent
    Optimizers play an important role in both pretraining and finetuning stages when training large language models (LLMs). 
    In this paper, we present an observation that full finetuning with the same optimizer as in pretraining achieves a better learning-forgetting tradeoff, i.e., forgetting less while achieving the same or better performance on the new task, than other optimizers and, possibly surprisingly, LoRA, during the supervised finetuning (SFT) stage. We term this phenomenon optimizer-model consistency.
    To better understand it, through controlled experiments and theoretical analysis, we show that: 1) optimizers can shape the models by having regularization effects on the activations, leading to different landscapes around the pretrained checkpoints; 2) in response to this regularization effect, the weight update in SFT should follow some specific structures to lower forgetting of the knowledge learned in pretraining, which can be obtained by using the same optimizer.
    Moreover, we specifically compare Muon and AdamW when they are employed throughout the pretraining and SFT stages and find that Muon performs worse when finetuned for reasoning tasks. With a synthetic language modeling experiment, we demonstrate that this can come from Muon's strong tendency towards rote memorization, which may hurt pattern acquisition with a small amount of data, as for SFT.
\end{abstract}

\section{Introduction}

Large language models (LLMs) have achieved remarkable success under the pretrain-finetune training paradigm \citep{radford2019language,ouyang2022instructgpt,grattafiori2024llama,yang2025qwen3}. In the pretraining stage, the models are trained on an enormous mixed corpus to obtain general language ability.
With the pretrained checkpoint in hand, we typically start the finetuning stage with supervised finetuning (SFT), where models are trained upon the pretrained checkpoints with high-quality but relatively small amounts of data to learn the skills in specific downstream tasks, e.g., math and code.
SFT is a significantly different training stage from pretraining, since models not only need to learn novel knowledge and skills, but also preserve the general language ability learned in pretraining. We call balancing the two factors the learning-forgetting tradeoff.

Training algorithms, or optimizers, play a central role in the learning-forgetting tradeoff. Among the optimizers, the current workhorse for LLM training is undoubtedly AdamW~\citep{kingma2014adam,loshchilov2017decoupled} and its variants, with Muon~\citep{jordan2024muon} and other recently emerged matrix-structured optimizers as strong competitors applied in frontier model training~\citep{team2026kimi,zeng2025glm,deepseekai2026deepseekv4}. Specifically in SFT, parameter-efficient training methods are also powerful, with LoRA~\citep{hu2022lora} as the most popular and representative one. The rich diversity in training algorithms makes us wonder: 
\textit{Which training algorithms achieve the best learning-forgetting tradeoff in the SFT stage?}

In this paper, we provide an answer to the question: full finetuning with the same (family of) optimizer as pretraining actually leads to the best learning-forgetting tradeoff compared to other optimizers and LoRA. As presented in Figure~\ref{fig:llama_pareto_optimal_comparison}, we plot a Pareto frontier where each point corresponds to a learning rate choice and a training period for a specific training method. 
We can observe that the solid blue line presenting full finetuning with AdamW (the same optimizer as pretraining) is at the uppermost and rightmost position of the figure, which implies that it has the least forgetting while achieving the same or even more learning compared to LoRA and Muon.
We find this observation that full finetuning with the same optimizer as pretraining in the SFT stage gives the best learning-forgetting tradeoff consistently holds. We term this phenomenon \textbf{optimizer-model consistency}.

\begin{figure}
    \centering 
   \includegraphics[width=0.8\linewidth]{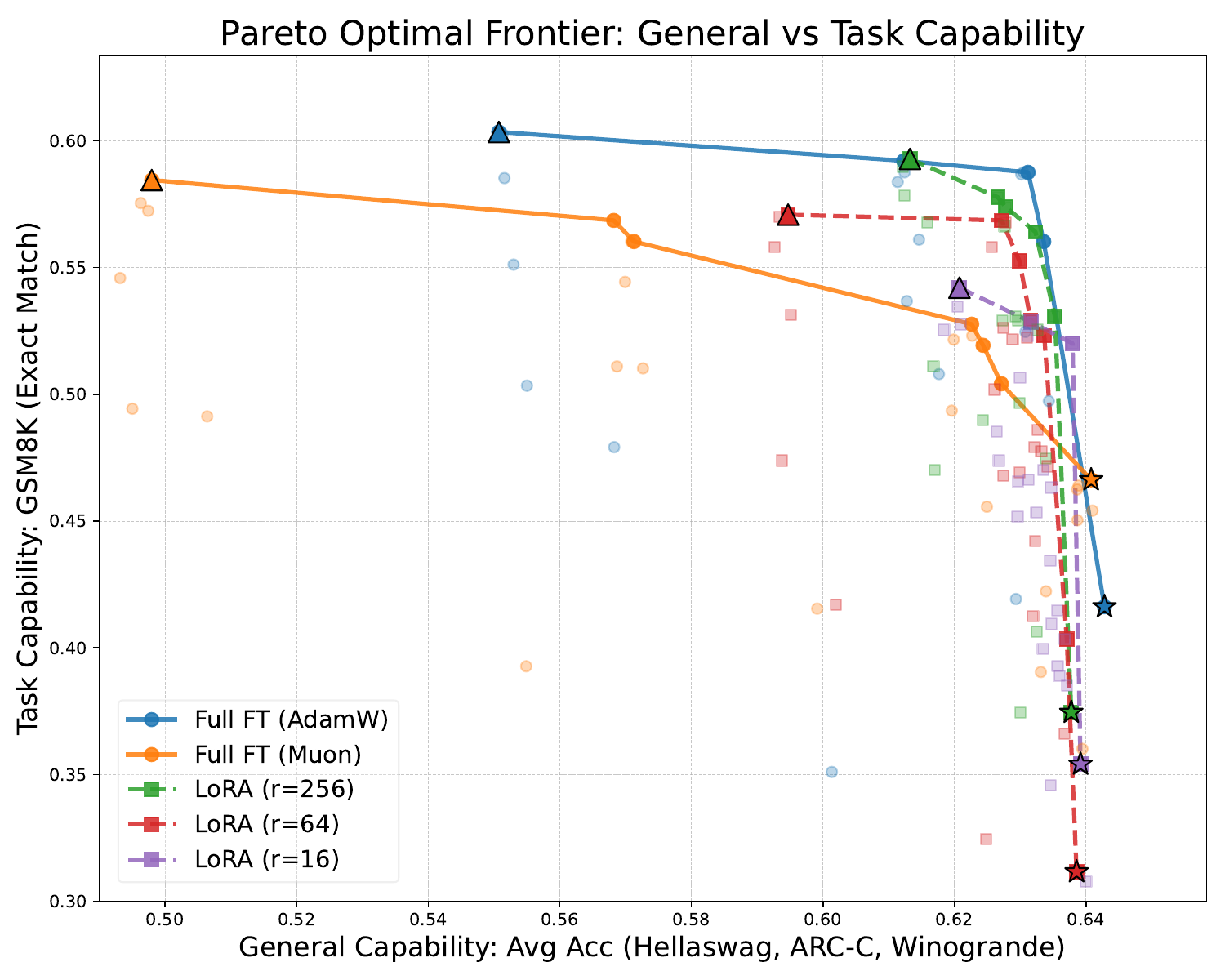}
  \vskip-0.2cm
  \caption{The Pareto frontier of different optimizers and LoRA finetuning Llama-2-7B with the MetaMathQA dataset for $1$ epoch. The $x$ and $y$ axes measure how much general knowledge the model forgets and how much math ability the model learns during SFT, respectively (higher is better). Full finetuning with the same optimizer as pretraining (AdamW) achieves the best learning-forgetting tradeoff, i.e., forgetting less when learning the same or better, compared to LoRA and other optimizers. Please refer to Section~\ref{sec:experiments_same_optimizer_forgets_less} to see detailed experiment settings.} \vspace{-0.5cm}
  \label{fig:llama_pareto_optimal_comparison}
\end{figure}

To understand why optimizer-model consistency occurs, we start by considering the difference in the models pretrained by different optimizers. 
Since LLM training is generally nonconvex, different (families of) optimizers consistently favor solutions with specific properties.
We discuss the properties in weights and activations, and consider the optimizer framework presented in \citet{pethick2025training} with matrix induced norms, showing that the activation properties of the model can originate from the optimizer's regularization effect. Then, based on the understanding of the activation regularization effect, we provide a theoretical explanation for optimizer-model consistency under specific assumptions modeling the observations.

From the understanding of optimizer-model consistency, we know that we should employ the same (family of) optimizers in SFT to have a better learning-forgetting tradeoff.
It also implies that the learning-forgetting tradeoff when an optimizer is employed throughout pretraining and SFT is a comprehensive metric for an optimizer in LLM training compared to only pretraining performance.
We conduct a case study for Muon and AdamW in this way and observe that though Muon generally provides a stronger pretrained checkpoint, its performance after SFT varies across different tasks, and can possibly be no better than AdamW if finetuned on reasoning tasks like math. 
We further conduct a synthetic experiment to suggest that a possible reason is that Muon may be better at memorization and potentially weaker at learning patterns due to its update formula, which is particularly significant in tasks with only a small amount of data, like SFT.
Therefore, a better finetuning method may be an interesting direction for improving Muon's overall performance in LLM training.

We summarize our contributions:
\vspace{-0.2cm}
\begin{itemize}
    \item We demonstrate the \textbf{optimizer-model consistency} phenomenon: full finetuning with the same optimizer as pretraining leads to a better learning-forgetting tradeoff than full finetuning with other optimizers or LoRA in the SFT stage. Specifically comparing with LoRA, we observe that full finetuning can surpass LoRA in both learning and forgetting when a learning rate sweep is properly considered (Section~\ref{sec:experiments_same_optimizer_forgets_less}).
    
    \item We show that an optimizer generally applies regularization to the activations related to the norm on which it is derived. We build a theoretical explanation for the optimizer-model consistency phenomenon under simplified assumptions by connecting the activation properties and how SFT update affects the forgetting of pretraining knowledge (Section~\ref{sec:understanding_phenomenon}).
    \item We specifically compare Muon and AdamW when they are employed throughout pretraining and SFT and find that Muon can lose its advantages when finetuning on reasoning tasks like math. We demonstrate that this may be because Muon tends to memorize rather than learn patterns through a synthetic language modeling experiment (Section~\ref{sec:muon_better_memorization}).
\end{itemize}

\section{Full Finetuning with the
Same Optimizer Forgets Less}\label{sec:experiments_same_optimizer_forgets_less}
In this section, we present the optimizer-model consistency phenomenon in two settings: finetuning GPT-2-small~\citep{ouyang2022instructgpt} and Llama-2-7B~\citep{touvron2023llama}. 
To better compare the SFT performance across different optimizers under different checkpoints, we conduct a full pretrain-finetune training pipeline on the GPT-2 model. We employ OpenWebText~\citep{gokaslan2019openWeb} as the pretraining dataset, and consider three different types of tasks for SFT: MetaMathQA~\citep{yu2023metamath} for math, Alpaca~\citep{alpaca} for instruction following, Magicoder-Evol-Instruct-110K~\citep{wei2023magicoder} for coding. 

For GPT-2 experiments, we evaluate the OpenWebText validation loss as a measure for forgetting, i.e., how much the models preserve the pretraining knowledge, and the loss on the new task validation loss as a measure for learning, i.e., how much the models learn in SFT (lower is better). 
For a comprehensive and fair comparison between algorithms, we plot the Pareto frontier of the learning-forgetting tradeoff, with $x$ and $y$-axis denoting the forgetting and learning metric, respectively. 
For each training algorithm, we consider all the runs in a learning rate grid search, which always contains the learning rate that achieves the best SFT performance, and we evaluate intermediate checkpoints for a single run. Then, we gather all the data points for each training algorithm to present the Pareto frontier to find the limit of this training algorithm in learning and forgetting.
The detailed experiment settings can be found in Appendix~\ref{appendix:experiment_settings_same_optimizer_forgets_less}.

\textbf{Muon achieves better pretraining performance compared to AdamW.} 
We present the results of the models pretrained by Adam and Muon in Table~\ref{tab:pretrain_results}. We can see that Muon outperforms AdamW across the pretraining dataset and different downstream tasks, validating the favorable performance of Muon in pretraining~\citep{jordan2024muon,liu2025muon,wen2025fantastic}.

\begin{table}[H]
    \centering
    \vspace{-0.5cm}
    \caption{The validation loss of different datasets after pretraining on GPT-2.}
    \vspace{0.2cm}
    \label{tab:pretrain_results}
    \begin{tabular}{ccccc}
    \toprule
    Optimizer for Pretraining & OpenWebText & Alpaca & MetaMathQA & Magicoder \\
    \midrule
    AdamW & 3.010 & 3.483 & 3.510 & 3.264 \\
    Muon & 2.974 & 3.235 & 3.200 & 3.000 \\
    \bottomrule
    \end{tabular}
    \vspace{-0.45cm}
\end{table}

\textbf{The same optimizer as pretraining leads to the best learning-forgetting tradeoff.}
With the pretrained checkpoints in hand, we start the SFT stage.
The Pareto frontier plots of full finetuning with AdamW and Muon to quantify the learning-forgetting tradeoff across different downstream tasks are summarized in Figure~\ref{fig:sft_optimizer_compare_same_ckpt}, where the $x$-axis denotes the OpenWebText validation loss and the $y$-axis denotes the task validation loss. One can observe that the same optimizer as pretraining always yields the best tradeoff, no matter whether it is AdamW or Muon, which is consistent with the observation in \citet{liu2025muon} on an instruction following task. It is also worth noting that SGD achieves a reasonable downstream performance, indicating its capability to learn from a small amount of data, but it forgets significantly more when the same downstream ability is learned.

\textbf{Full finetuning can forget less than LoRA with appropriate hyperparameters.}
We also evaluate the learning-forgetting tradeoff of LoRA, which is commonly believed to forget less because of its low-rank nature~\citep{biderman2024lora}. The results are presented in Figure~\ref{fig:pareto_frontier_comparison_with_lora}. We observe that LoRA generally tends to forget less when it achieves the best learning performance (the best $y$-axis value), which is consistent with the common intuition and existing results. 
However, when taking a full view of the tradeoff, we find that for any single LoRA checkpoint, there is always a full finetuning checkpoint that is absolutely better, i.e., with less forgetting while achieving the same or better learning.
Therefore, LoRA in this setting generally gives a worse learning-forgetting tradeoff compared to full finetuning with the same optimizer, and additionally, we can observe that a larger LoRA rank generally corresponds to a better tradeoff. 
While LoRA is an effective training algorithm with incredible parameter efficiency, the results suggest that it does not necessarily help in improving learning-forgetting tradeoff performance.

\begin{figure}[t]
\centering

\begin{tabular}{ccc}
	\includegraphics[width=0.3\linewidth]{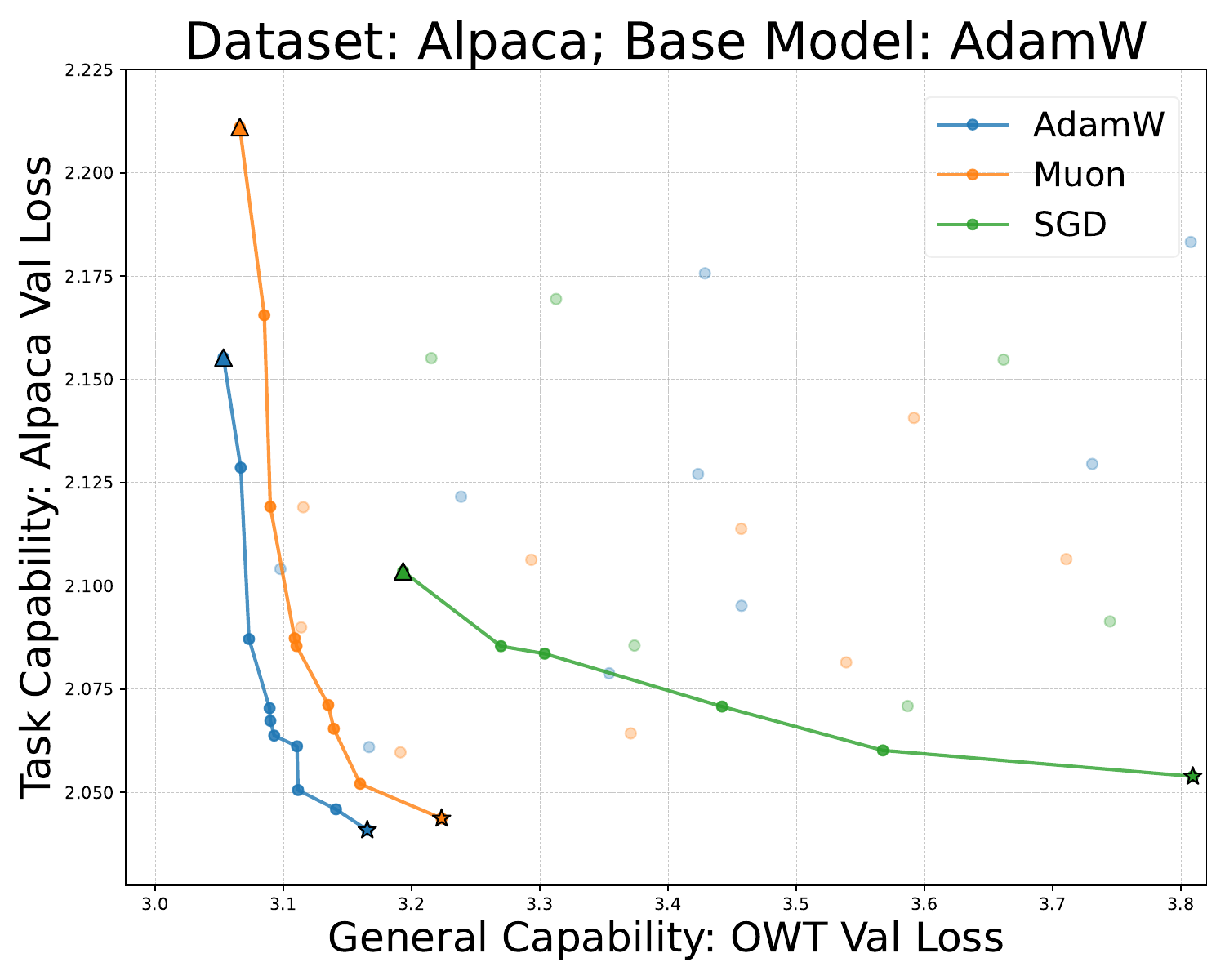}
	\!\!\!
	& \includegraphics[width=0.3\linewidth]{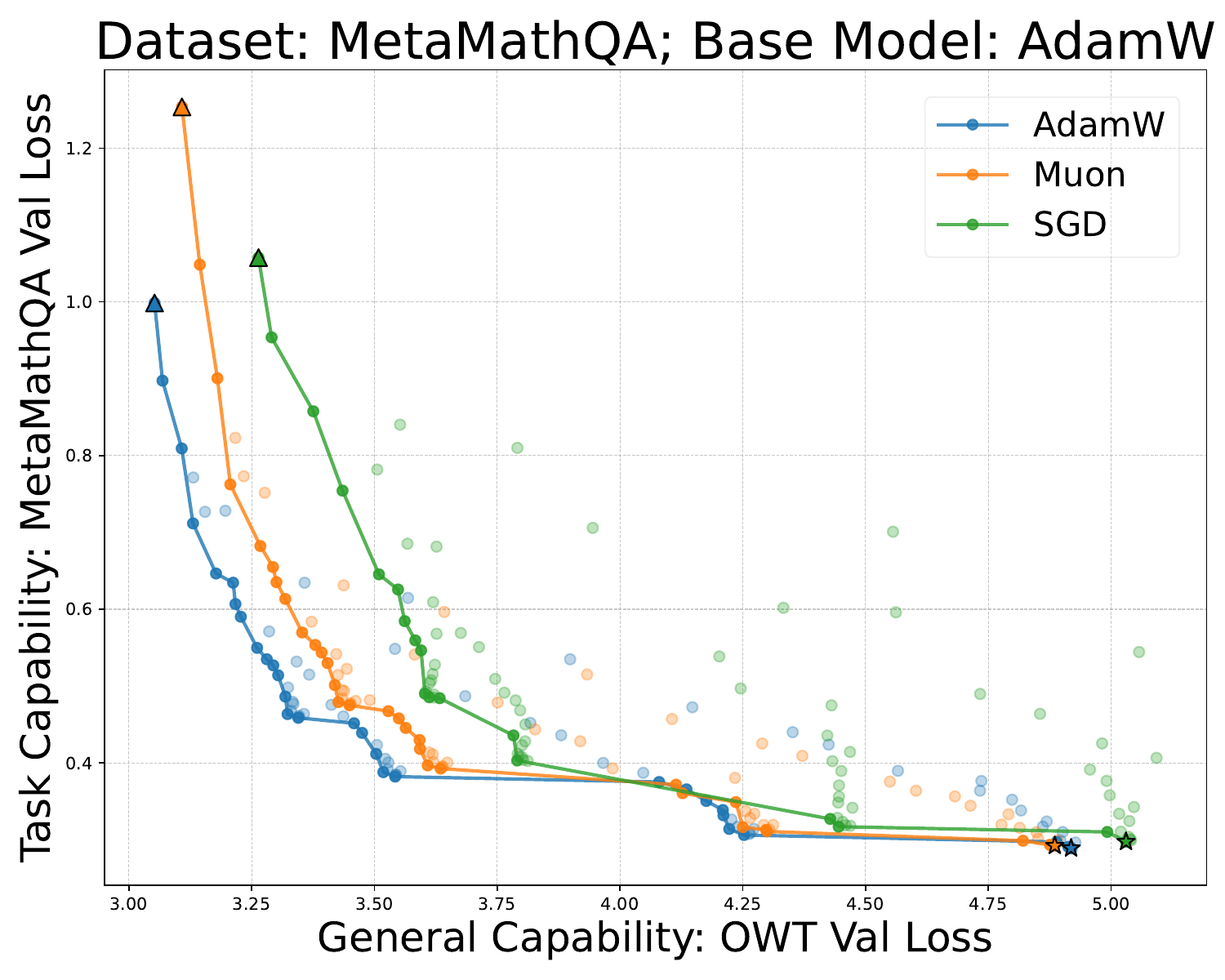} 
    \!\!\!
	& \includegraphics[width=0.3\linewidth]{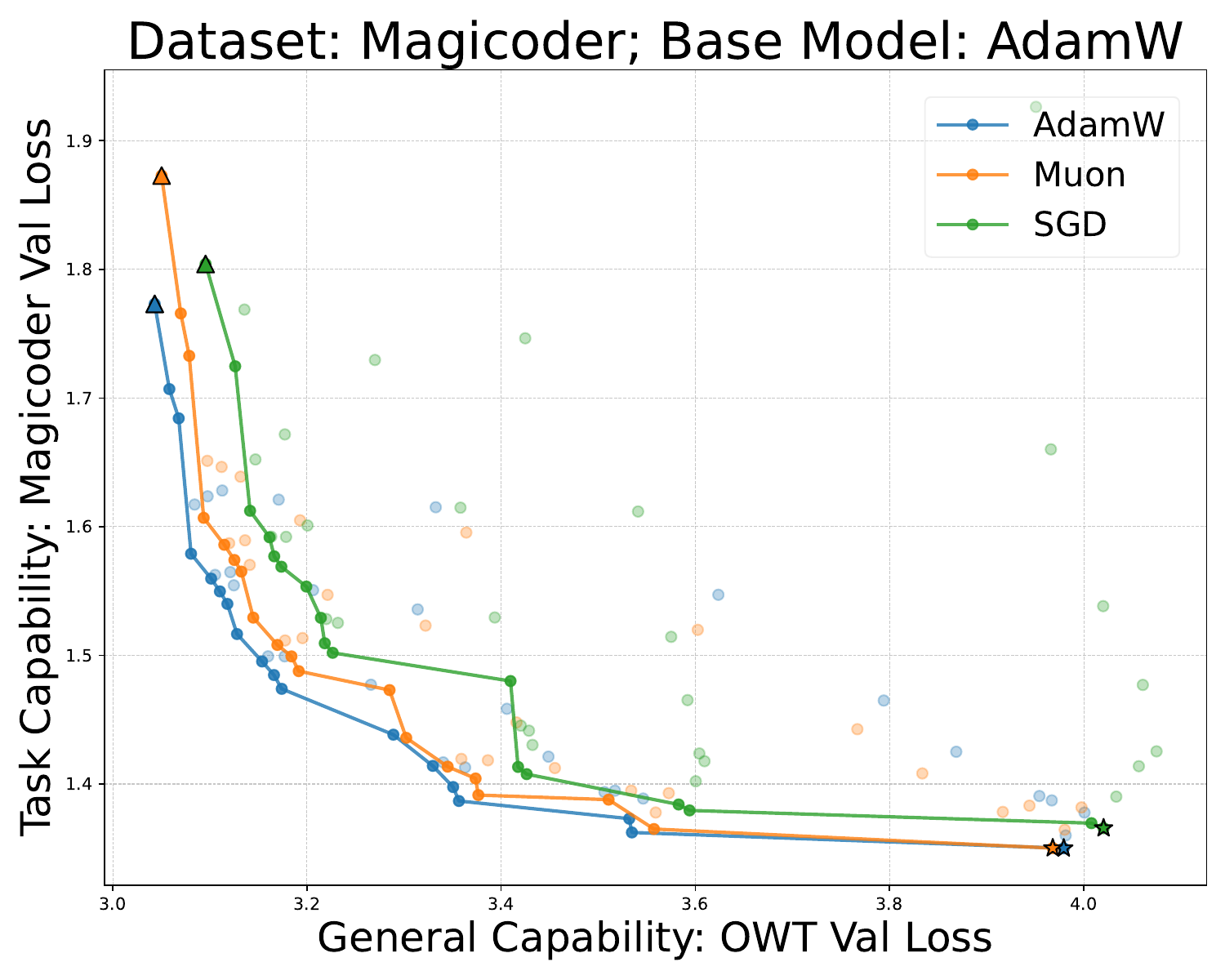} \\
    \includegraphics[width=0.3\linewidth]{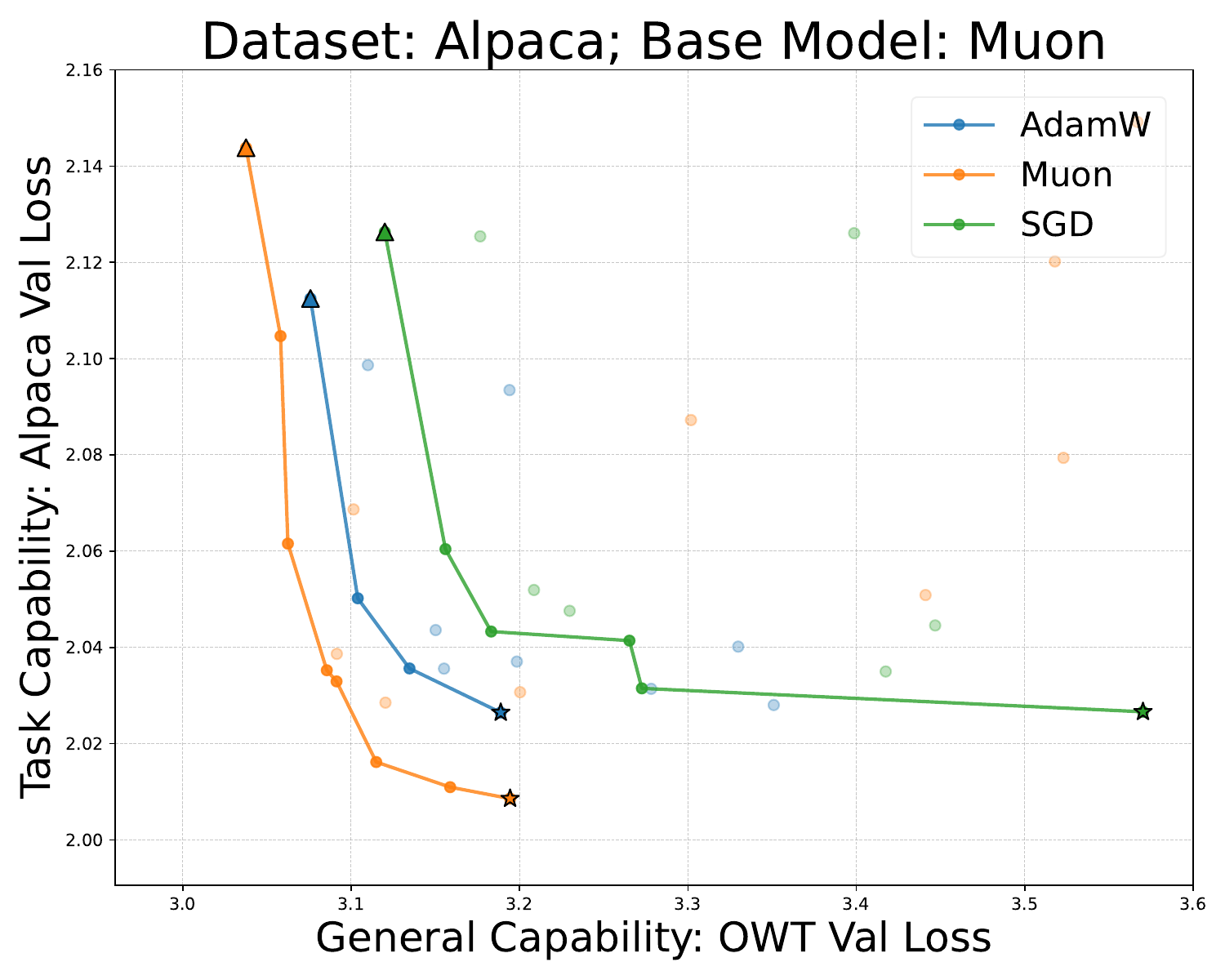}
	\!\!\!
	& \includegraphics[width=0.3\linewidth]{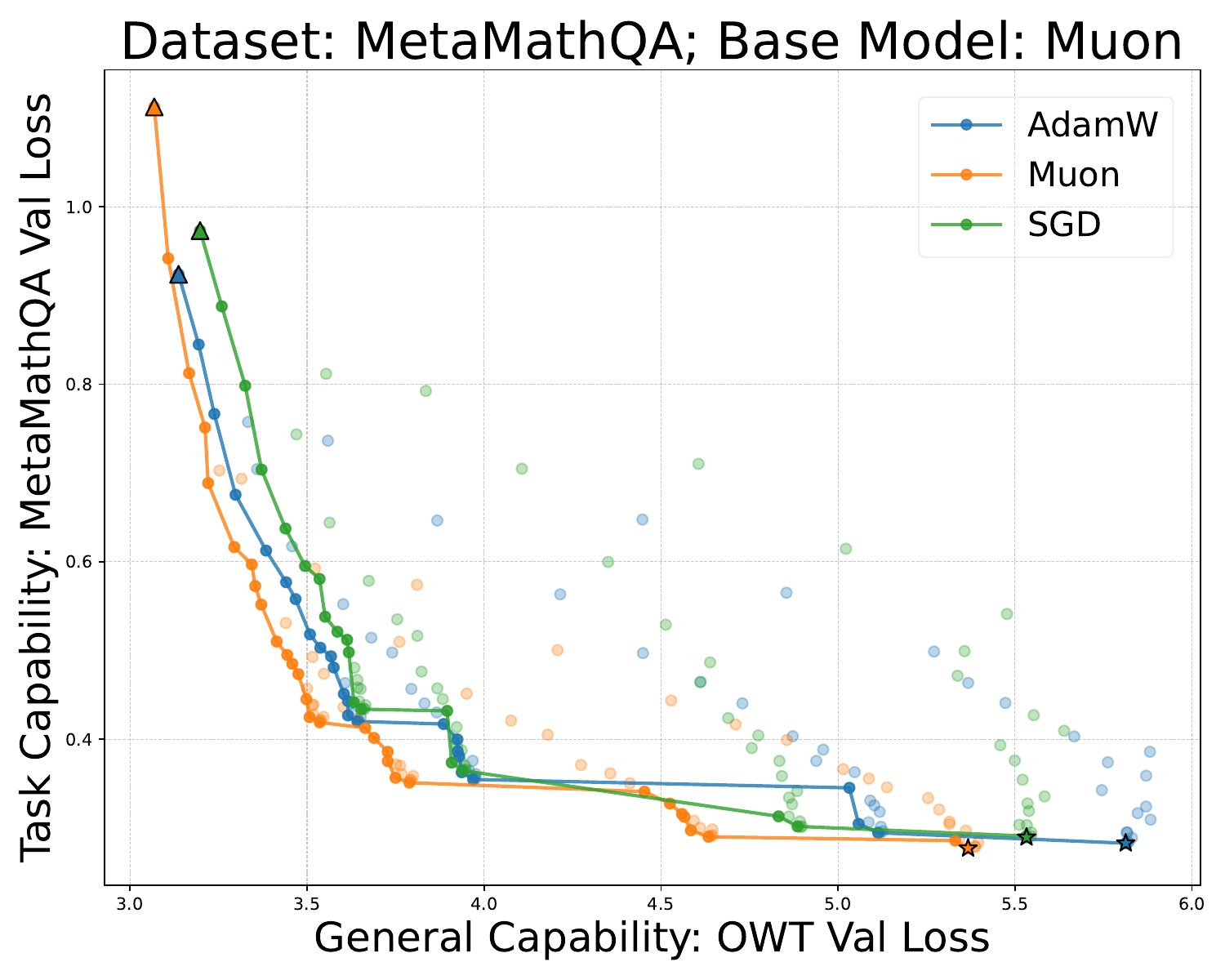} 
    \!\!\!
	& \includegraphics[width=0.3\linewidth]{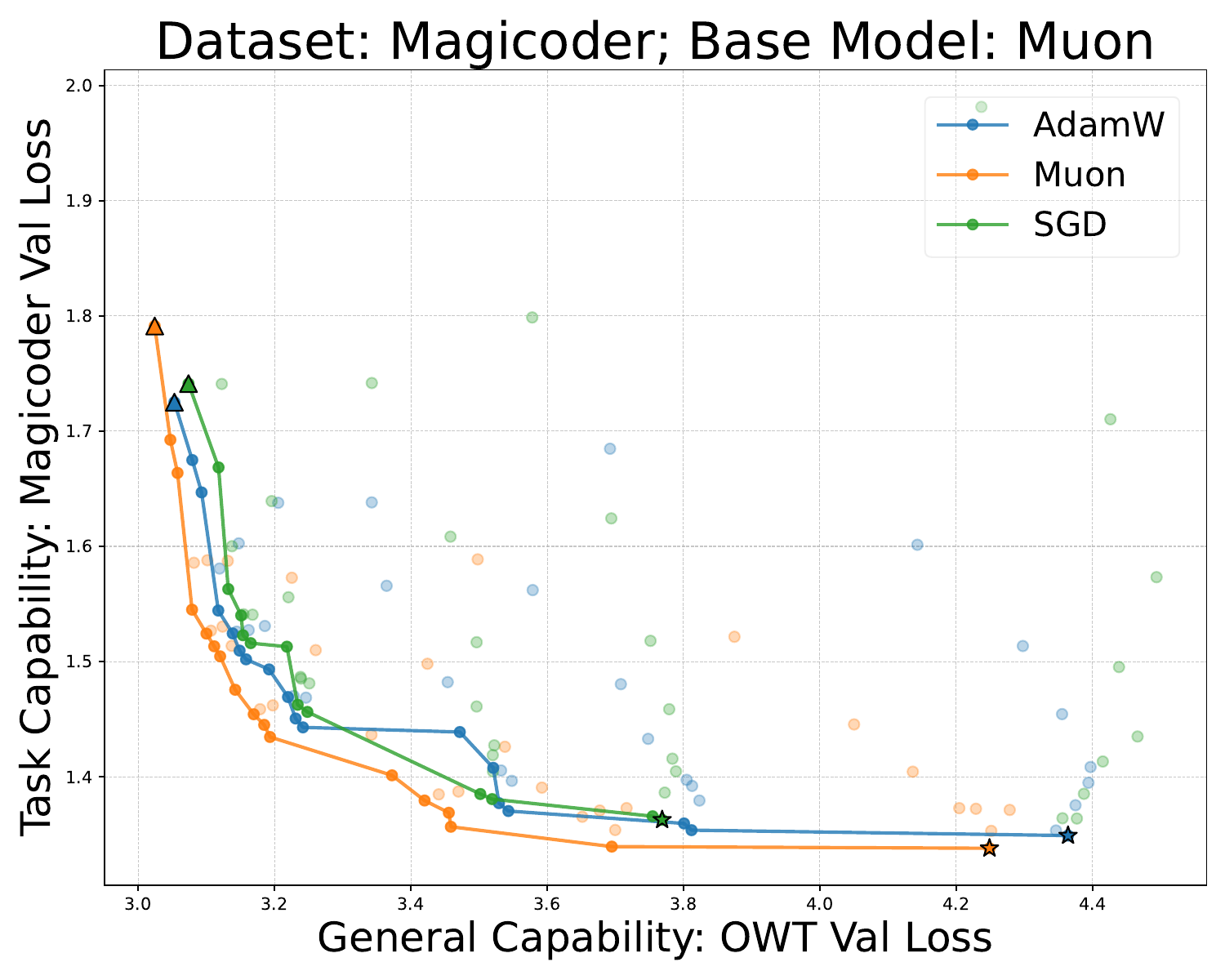}
\end{tabular} 

\vskip-0.2cm
\caption{Pareto
frontier plots of full finetuning with AdamW and Muon optimizers on GPT-2. The rows and columns present the results of different pretraining optimizers and SFT datasets, respectively. The measure for learning and forgetting is the corresponding validation loss (lower is better). We can observe that the same optimizer performs better in the forgetting-learning trade-off in SFT. } \vspace{-0.5cm}
\label{fig:sft_optimizer_compare_same_ckpt} 
\end{figure}

\textbf{Results hold on 7B model.} 
As shown in Figure~\ref{fig:llama_pareto_optimal_comparison}, our conclusions in small models hold still when finetuning Llama-2-7B with MetaMathQA, generally following the settings of \citet{biderman2024lora}. We can observe that full finetuning with Adam (the same optimizer as Llama-2-7B pretraining) yields a better learning-forgetting tradeoff, even without weight decay, compared to other full finetuning optimizers and LoRA. 
{We would be happy to also provide Muon results on this scale in the future due to the lack of Muon-pretrained open-source models on an appropriate scale currently.}

{For brevity, we name this phenomenon that full finetuning with the same optimizer as pretraining leads to the best learning-forgetting tradeoff as \textbf{optimizer-model consistency}.}

\begin{remark}[Discussion on Results in \citet{biderman2024lora}]\label{remark:discussion_lora_results}
    While our results seem to be contradictory to \citet{biderman2024lora}, where the authors suggest that LoRA learns less and forgets less under a similar training setting, we note that the results are actually compatible. 
    The major reason for the varying conclusions is that we consider different learning rates when modeling the learning-forgetting tradeoff, while \citet{biderman2024lora} considers only one optimal learning rate. 
    The importance of learning rate choice for forgetting in SFT is also discussed in a concurrent paper~\citep{rofin2026learning}.
    As shown in Figure~\ref{fig:reproduction_lora_forgets_less_llama}, we can also reproduce the observations of \citet{biderman2024lora} if we only consider one learning rate optimal in learning for each training setting. 
    However, when the learning rate choice is appropriately taken into account, full finetuning can also learn less and forget less simply by lowering the learning rate, leading to that we can always find a checkpoint trained with full finetuning that surpasses LoRA in both learning and forgetting.
    \vspace{-0.2cm}
\end{remark}

\begin{figure}[t]
\centering

\begin{tabular}{ccc}
	\includegraphics[width=0.45\linewidth]{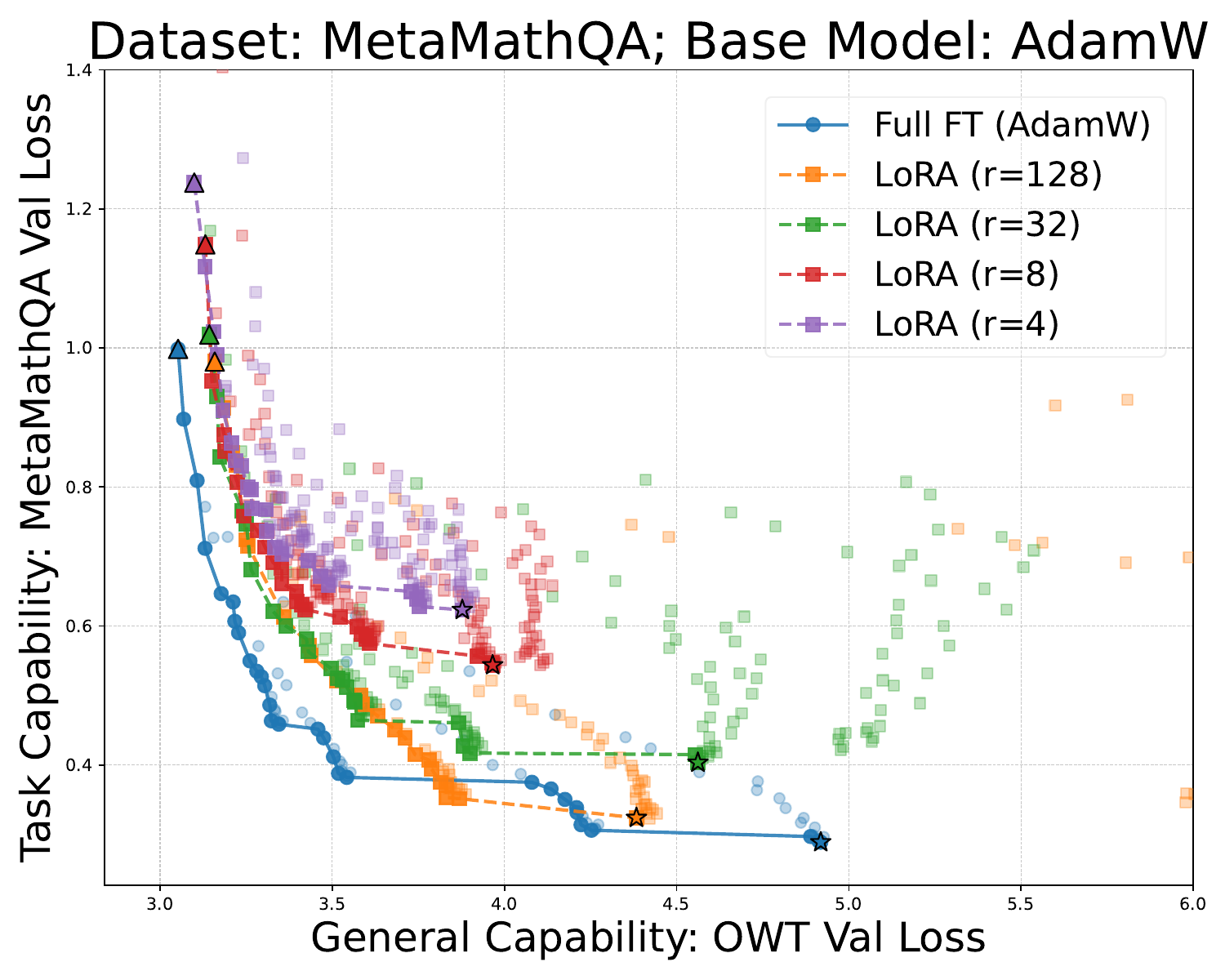}
	\!\!\!
	& \includegraphics[width=0.45\linewidth]{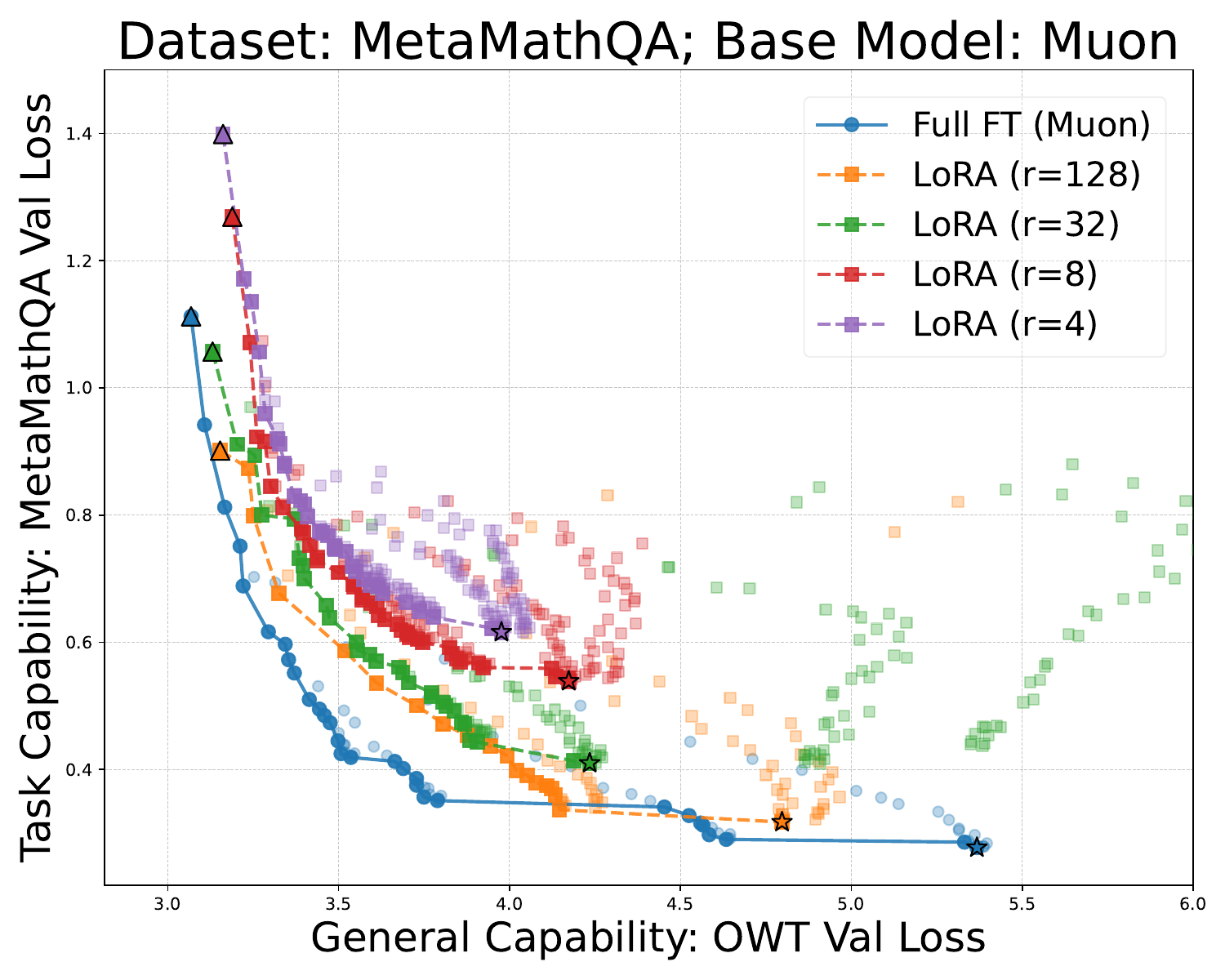} 
\end{tabular} 

\vskip-0.2cm
\caption{Pareto frontiers of full finetuning and LoRA on GPT-2 with the same optimizer as pretraining employed for both. Full finetuning generally achieves a better learning-forgetting tradeoff. } \vspace{-0.2cm}
\label{fig:pareto_frontier_comparison_with_lora} 
\end{figure}

\section{Why Optimizer-Model Consistency Happens}\label{sec:understanding_phenomenon}
The optimizer-model consistency phenomenon may seem surprising at first glance if we consider optimizers or training algorithms only as processes towards some fixed solutions.
However, optimizers actually have a much larger effect on the trained model than this common impression. In LLM training or other deep learning tasks, the losses are generally nonconvex, and different optimizers can lead to very different trajectories and convergence points. 
In LLM training or other deep learning tasks, the losses are generally nonconvex, and different optimizers can consistently lead to significantly different trajectories and convergence checkpoints. 
The difference in checkpoints is more than random noise, but consistent properties in the weights and activations of the model, which we will discuss in Section~\ref{sec:optimizers_shape_models}. 
Then, with these properties, we have the basis to better understand why optimizer-model consistency happens in Section~\ref{sec:theory_sft_forgets_less}.

\subsection{Optimizers Shape Models}\label{sec:optimizers_shape_models}

\begin{wraptable}{r}{0.5\textwidth}
    \centering
    \vspace{-0.3in}
    \caption{Singular spectrum comparison after pretraining. For a singular spectrum $\sigma_1 \ge \cdots \ge \sigma_n$, the stable rank is $\sum_{i=1}^n \sigma_i^2 / \sigma_1^2 $ and singular sparsity is $\sum_{i=1}^n \sigma_i / \sqrt{n\sum_{i=1}^n \sigma_i^2}$. We take the average of the values over all linear layers in GPT-2.}
    \vspace{0.1 in}
    \begin{tabular}{lcc}
    \toprule
    Optimizer & Stable Rank & Singular Sparsity \\ \midrule
    AdamW & $65.75$ & $0.837$ \\ \midrule
    Muon & $114.72$ & $0.875$ \\ 
    \bottomrule
    \end{tabular}
    \vspace{-0.1in}
    \label{tab:stable-rank-comparison}
\end{wraptable}

\textbf{Optimizers shape weights.}
We start with the observations on the difference in model weights trained by different optimizers through pretraining. 
As presented in Table~\ref{tab:stable-rank-comparison}, we can observe that the model trained by Muon shows a denser weight spectrum distribution (higher stable rank and singular sparsity value), which is consistent with the observations in \citet{liu2025muon,wang2025muon}. A straightforward intuition for this spectrum difference is that Muon updates are always approximately full-rank because of its formula, resulting in the final weights, an accumulation of all the updates, being closer to full-rank compared to AdamW.

\textbf{Optimizers regularize activations.}
In addition to explicit weight properties, optimizers also affect the activations of trained models on the pretraining dataset, which is closely related to how the models store pretraining knowledge. As presented in Figure~\ref{fig:adamw_muon_activation_sparsity}, we can observe that the activations, especially input activations, of models trained by AdamW are much sparser compared to those by Muon. 

\begin{figure}[th]
  \centering 
   \begin{tabular}{cc}
	\includegraphics[width=0.45\linewidth]{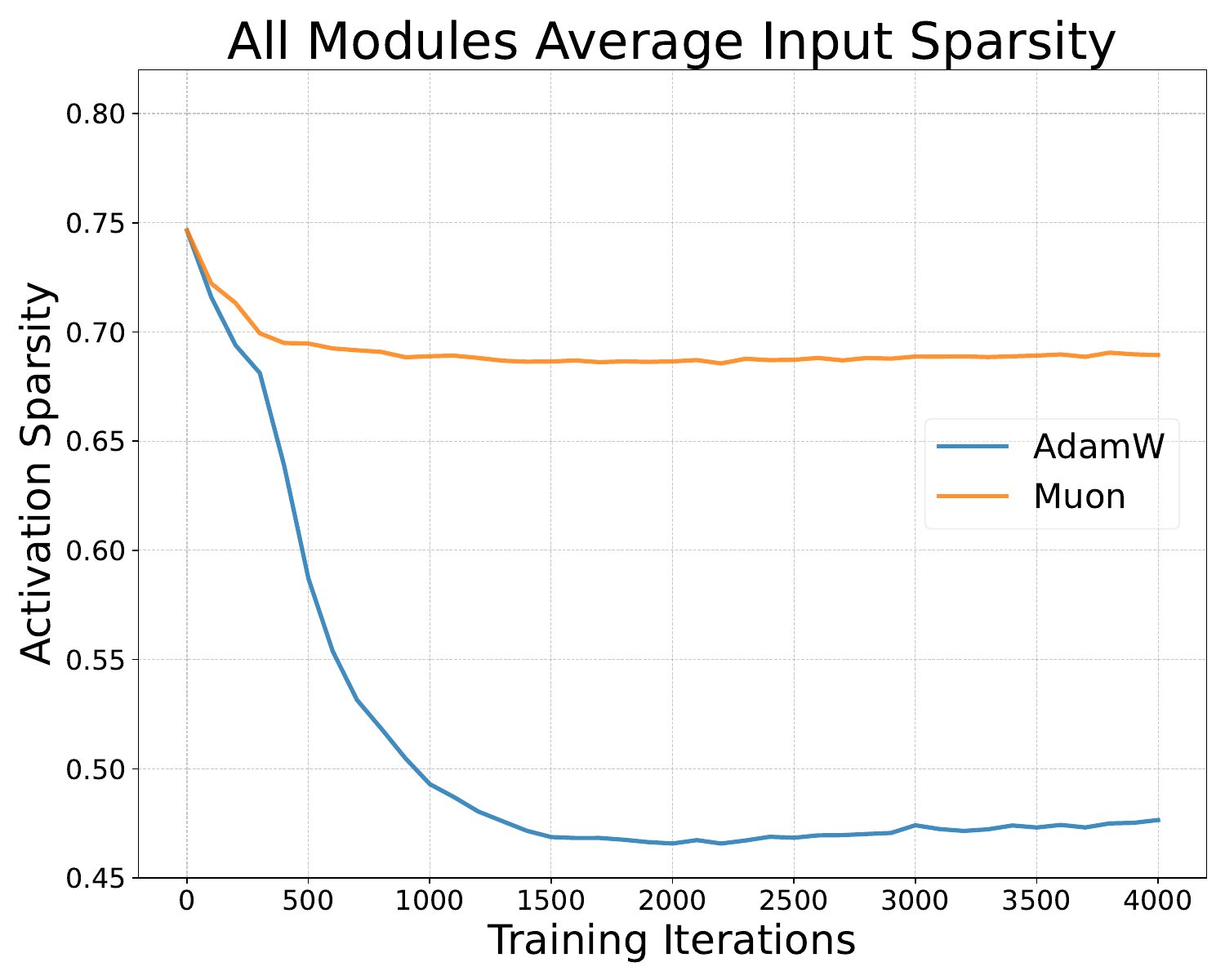}
	\!\!\!
	& \includegraphics[width=0.45\linewidth]{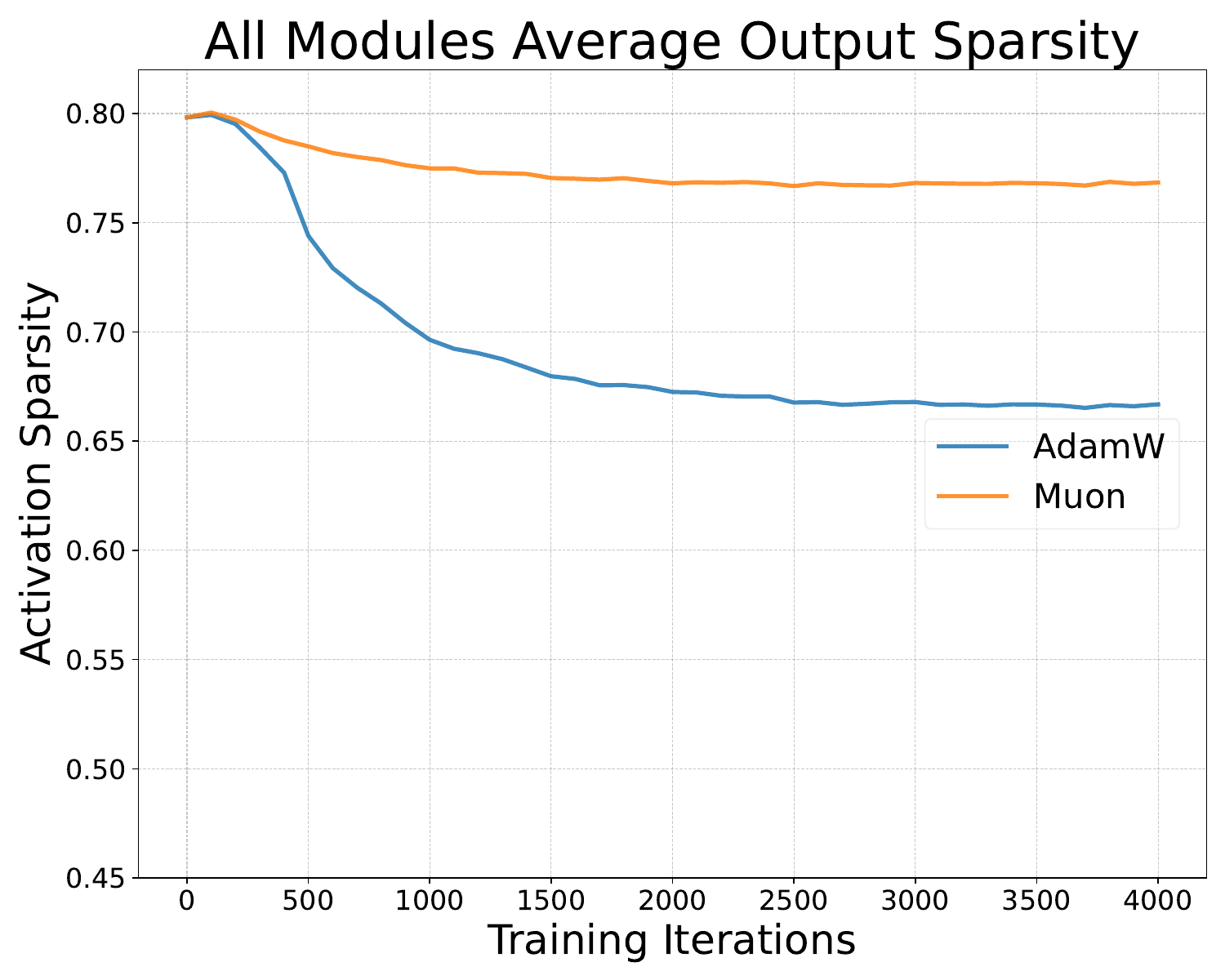} 
\end{tabular} 
  \vskip-0.2cm
  \caption{Activation sparsity of AdamW and Muon. We adopt basically the same pretraining settings as Section~\ref{sec:experiments_same_optimizer_forgets_less}, but for 4k iterations.
The $y$-axis represents the sparsity of activations, which is computed by $\Norm{x}_1 / (\sqrt{d} \Norm{x}_2)$ for an activation vector $x$ (lower means sparser). We take an average of the input and output sparsity over all the linear layers. }
  \vskip-0.2cm
  \label{fig:adamw_muon_activation_sparsity}
\end{figure}

To understand how this difference in activations originates, we start from the following framework for optimizer derivation~\citep{pethick2025training}: with $G_t$ as the gradient of weight $W_t \in \RR^{m\times n}$, we update the weight by
\begin{align}\label{eq:steepest_descent}
    W_{t+1} =& \underset{\Norm{W - W_t} \le R}{\text{argmin}} \dotprod{G_t}{W - W_t}  ,
\end{align}
where we consider a specific norm $\Norm{\cdot}$.
Since the weights of neural nets generally have a matrix structure and the activation computation is by matrix-vector multiplication, we focus on the matrix induced norms under the framework: for $\alpha, \beta \in [1,\infty]$, we denote
$
    \Norm{A}_{\alpha,\beta} \triangleq \max_{x \neq 0}\frac{\Norm{Ax}_{\beta}}{\Norm{x}_{\alpha}} 
$
and the optimizer derived from $\Norm{\cdot}_{\alpha,\beta}$ as $\cA_{\alpha,\beta}$.
Note that Muon is basically $\cA_{2,2}$ from its derivation~\citep{jordan2024muon,bernstein2024old} and SignSGD~\citep{bernstein2018signsgd} is $\cA_{1,\infty}$ by noticing that the vector infinite norm is equivalent to $\Norm{\cdot}_{1,\infty}$. Following~\citet{balles2018dissecting,balles2020geometry,kunstner2023noise}, we consider Adam as a powerful variant of SignSGD, thus it also falls in the family of $\cA_{1,\infty}$. \citet{pethick2025training,xu2026width} discuss additional optimizers derived from this framework with matrix induced norms, demonstrating its generality.

For an optimizer $\cA_{\alpha,\beta}$, the weight $W$ is updated with $\Norm{\cdot}_{\alpha,\beta}$ bounded. Besides a direct constraint on the weights, this can also be considered as some forms of regularization for the activations: regularizing the $\Norm{\cdot}_{\alpha}$ input activation $x$ and $\Norm{\cdot}_{\beta}$ for the output activation $y$, since we have $y = Wx$ and the matrix induced norm definition $\Norm{y}_\beta \le \Norm{W}_{\alpha,\beta} \Norm{x}_\alpha$. Specifically, we have:

\begin{itemize}
    \item Muon is $\cA_{2,2}$, thus regularizes input $x$ by $\Norm{\cdot}_2$ and output $y$ by $\Norm{\cdot}_2$.
    
    \item Adam is $\cA_{1,\infty}$, thus regularizes input $x$ by $\Norm{\cdot}_1$ and output $y$ by $\Norm{\cdot}_\infty$.
\end{itemize}

Different regularizations lead to different properties of activations. For example, $\Norm{\cdot}_1$ regularization corresponds to Lasso, which is widely known to result in a sparse solution. As presented in Figure~\ref{fig:optimizer_activation_sparsity}, empirical observations verify our intuition about activation properties. In addition to SignSGD ($\cA_{1,\infty}$), AdamW($\cA_{1,\infty}$), and Muon ($\cA_{2,2}$), we also consider algorithms $\cA_{1,1}$ and $\cA_{\infty,\infty}$ for a more comprehensive view of the regularization effect. Their details can be found in Appendix~\ref{appendix:experiment_setting}.
We can observe that for an algorithm $\cA_{\alpha,\beta}$, the input activation is consistently sparser when $\alpha$ is taken to be $1$, which holds for $\cA_{1,\infty}$ and $\cA_{1,1}$. The same argument also holds for the output activations when $\beta=1$, which verifies the regularization effect of activations and its relation to the optimizer choice.

\begin{figure}[t]
\centering

\begin{tabular}{ccc}
	\includegraphics[width=0.3\linewidth]{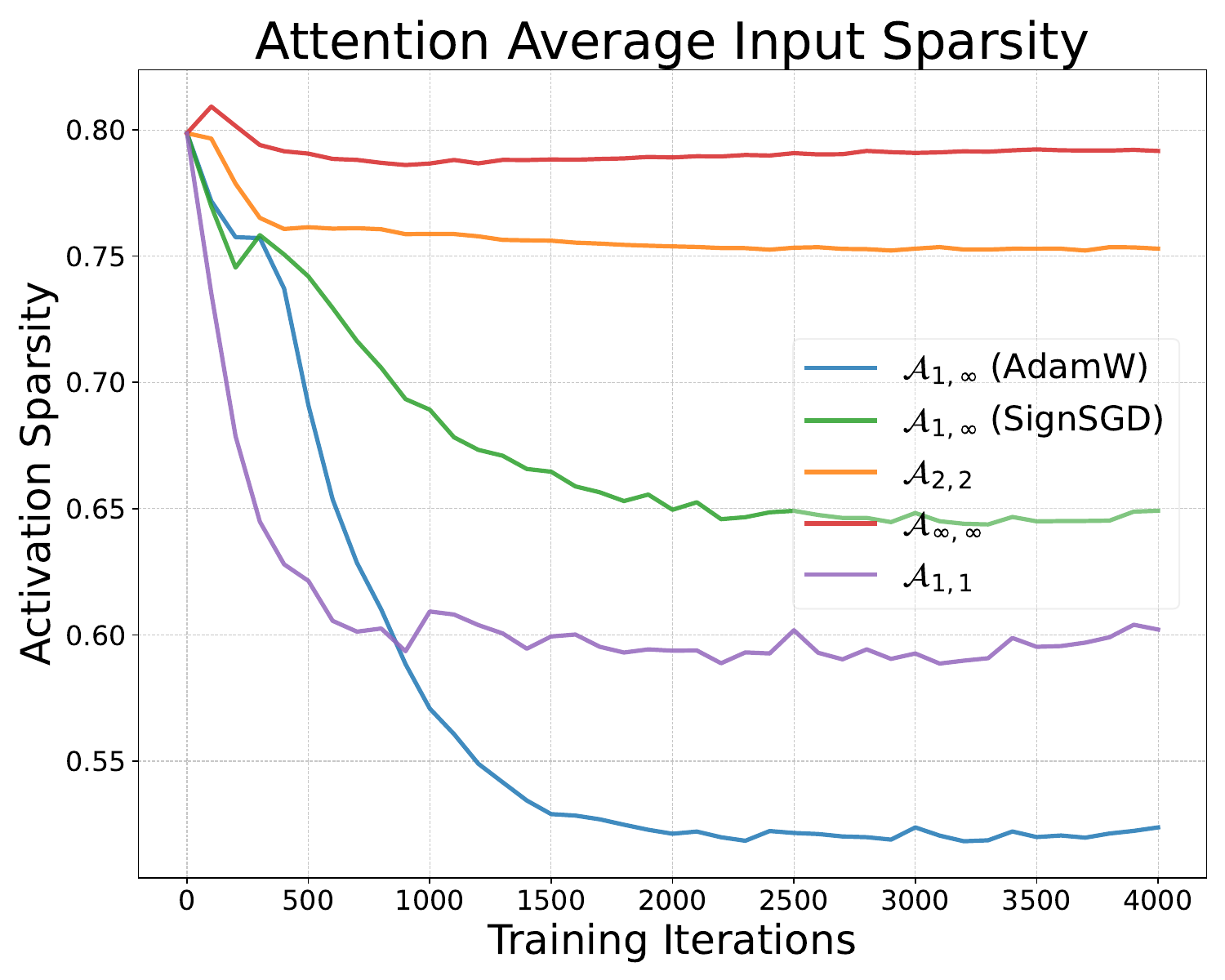}
	\!\!\!
	& \includegraphics[width=0.3\linewidth]{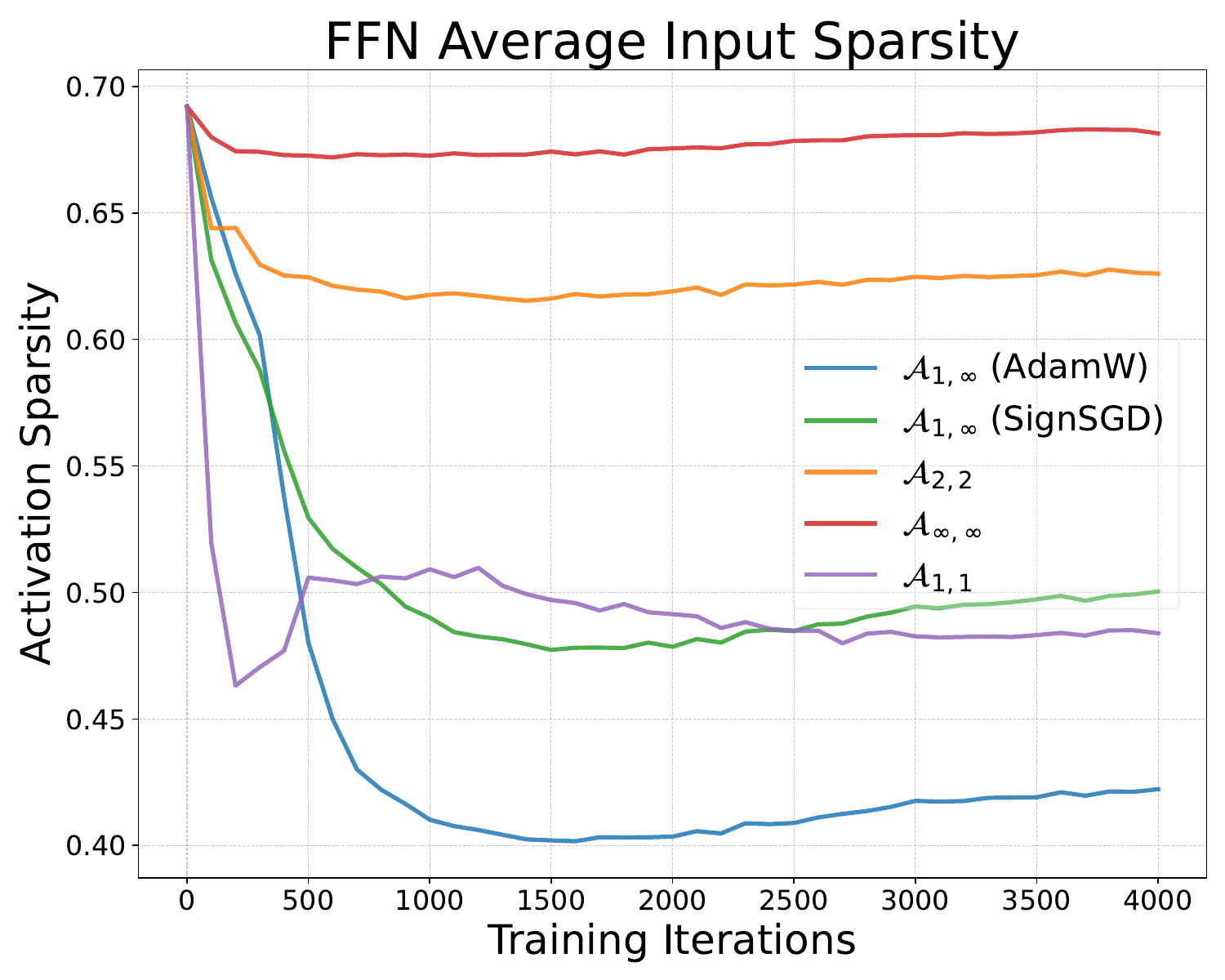} 
    \!\!\!
	& \includegraphics[width=0.3\linewidth]{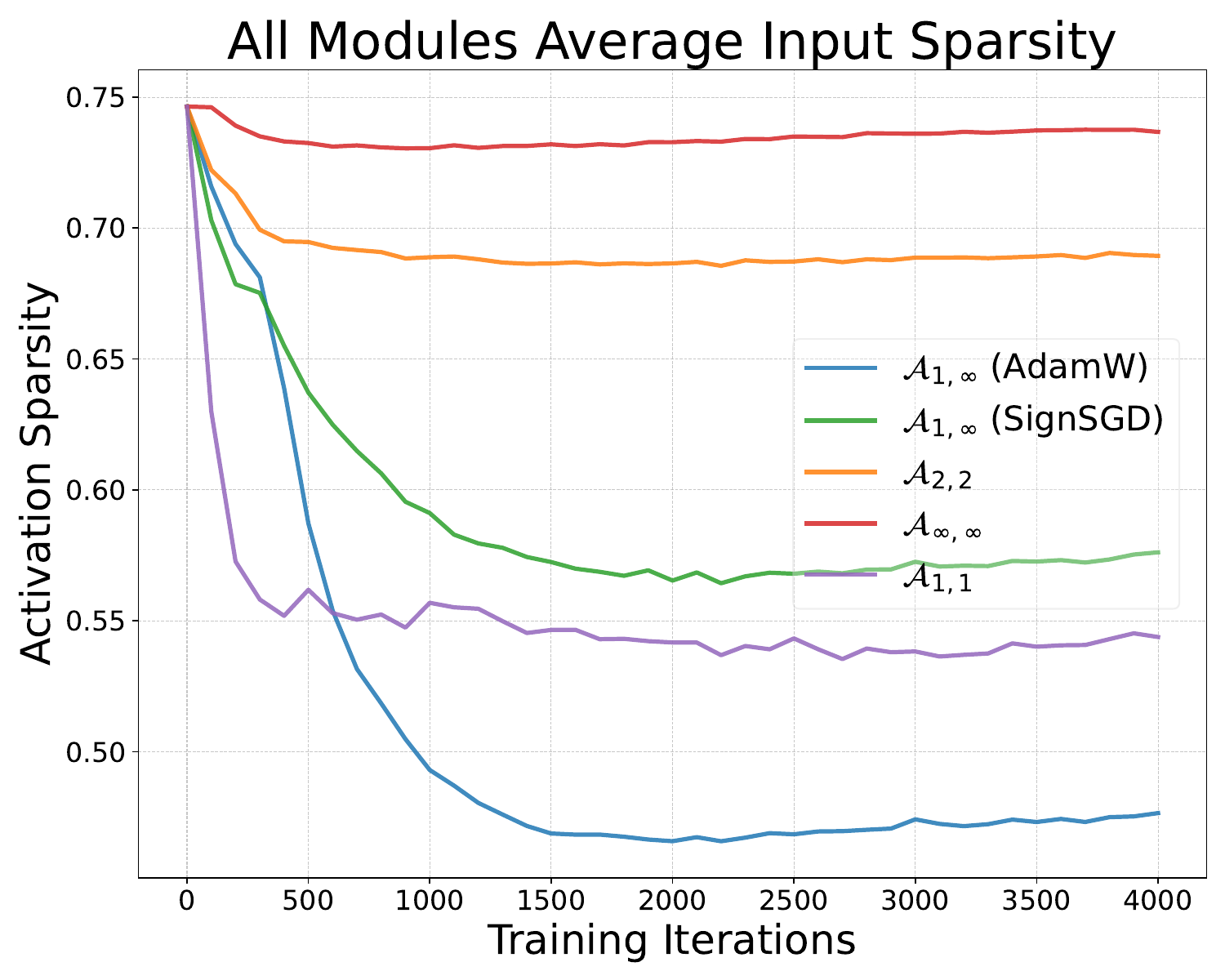} \\
    \includegraphics[width=0.3\linewidth]{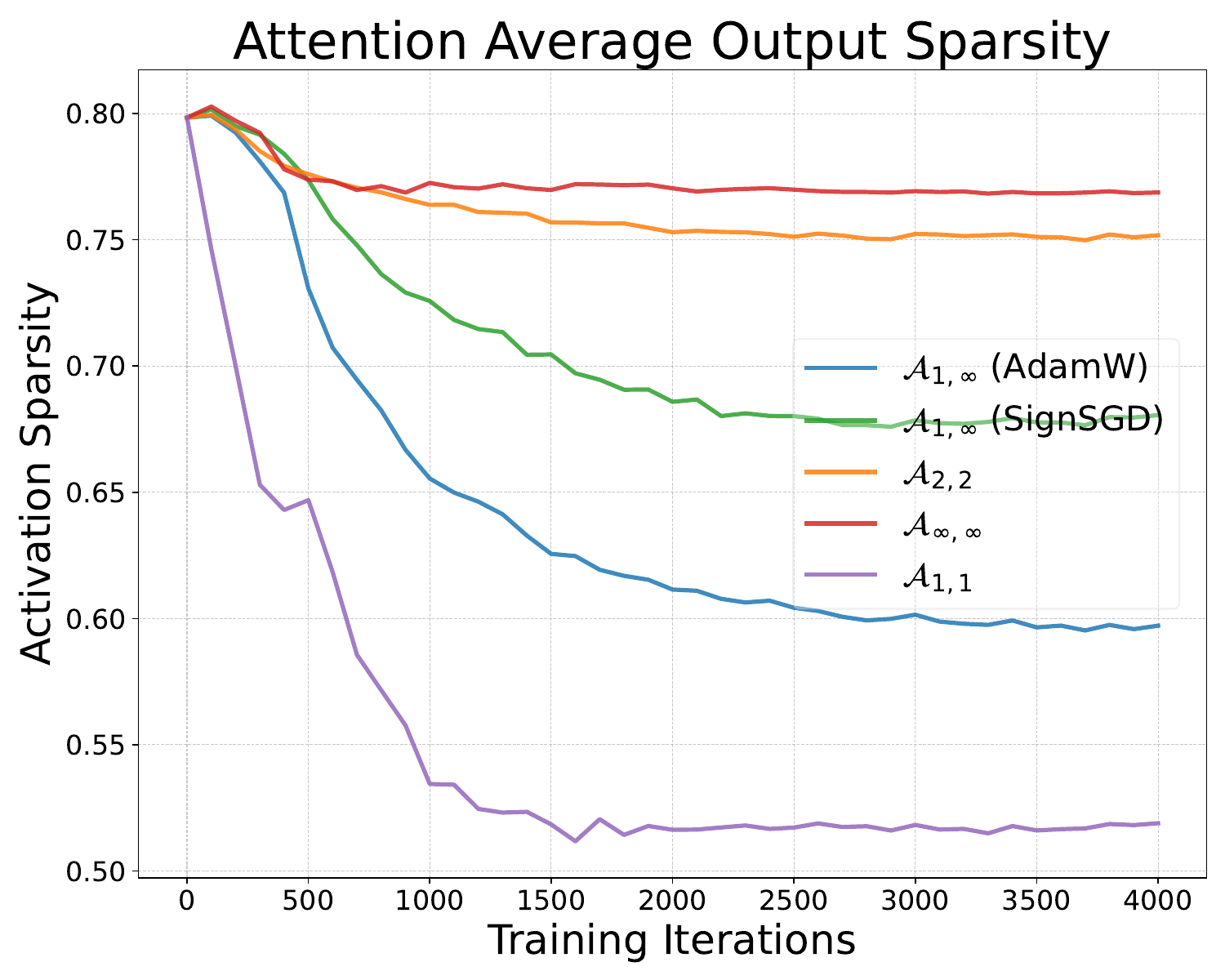}
	\!\!\!
	& \includegraphics[width=0.3\linewidth]{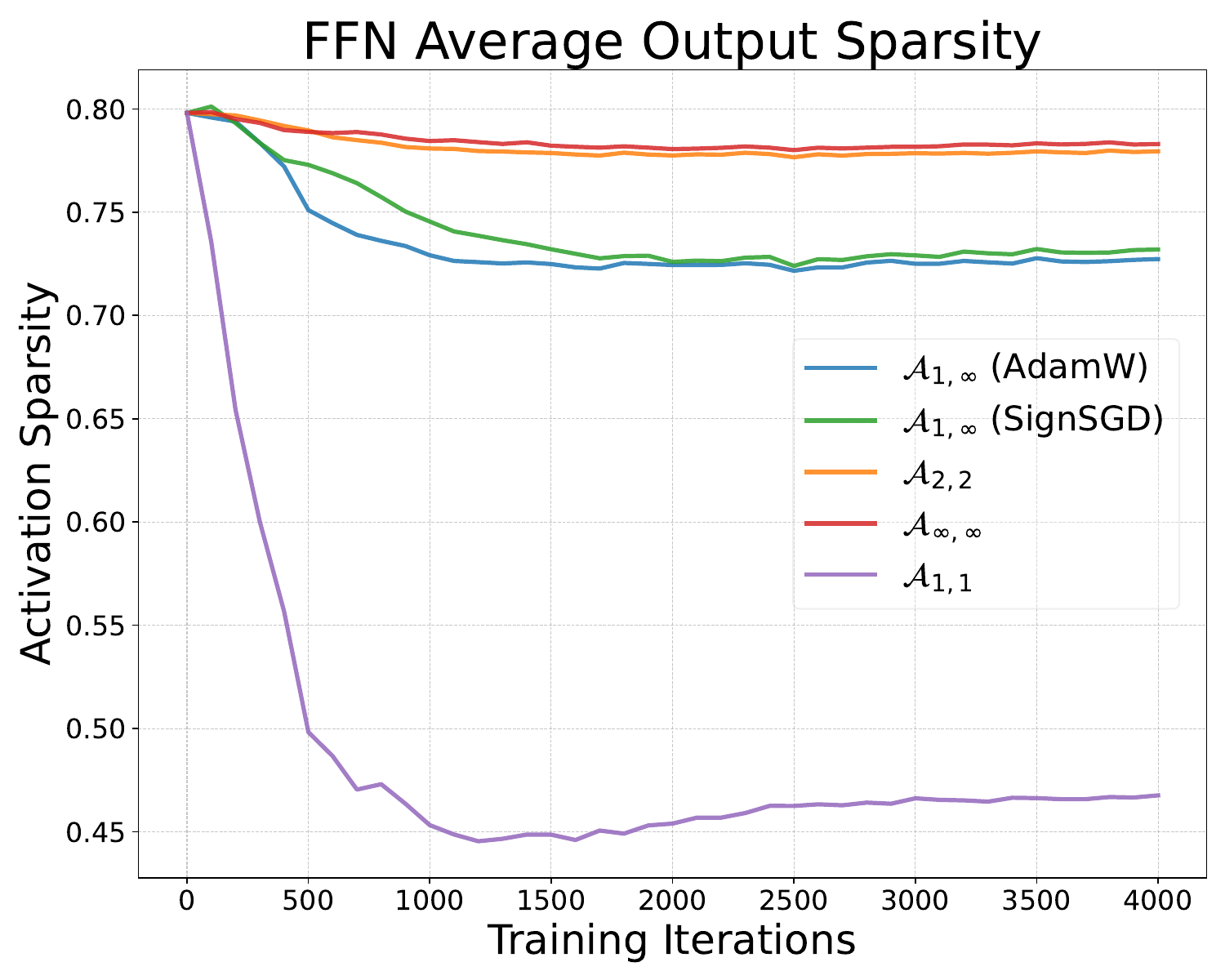} 
    \!\!\!
	& \includegraphics[width=0.3\linewidth]{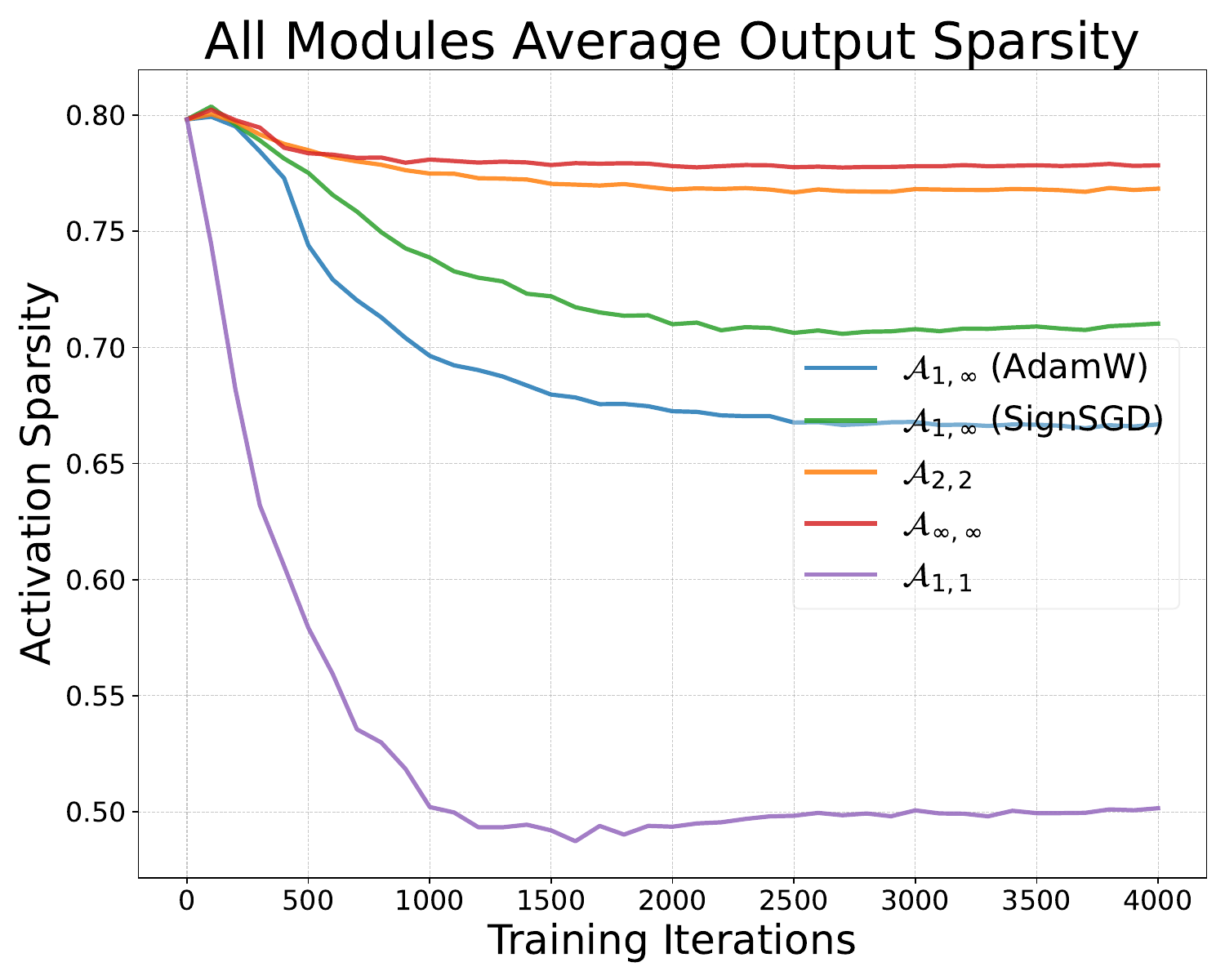}
\end{tabular} 

\vskip-0.2cm
\caption{Activation sparsity of different optimizers $\cA_{\alpha,\beta}$ derived from \eqref{eq:steepest_descent} with the same settings as Figure~\ref{fig:adamw_muon_activation_sparsity}. The columns present the average sparsity of different modules, and the columns present input and output activations of the modules. We only consider the linear layers for doing this average. } \vspace{-0.2cm}
\label{fig:optimizer_activation_sparsity} 
\end{figure}

\subsection{The SFT Optimizer should Fit the Activation Regularization}\label{sec:theory_sft_forgets_less}

Based on the observations that optimizers can shape models, we provide a theoretical explanation of the optimizer-model consistency phenomenon under a simplified setting. The intuition is straightforward: since different optimizers lead to pretrained checkpoints with different properties, forgetting gets less when the SFT update better matches the curvature around the checkpoints, which can be achieved with full finetuning using the same optimizer.
We start with approximations of the forgetting and SFT loss, based on the fact that only a small weight change is applied in the SFT stage.

\textbf{Approximating forgetting.}
We have the following Taylor expansion approximation of the pretraining loss increase after SFT:
\vspace{-0.2cm}
\begin{align*}
\begin{split}
    \Lp(\theta_0 + \Delta \theta) \approx& \Lp(\theta_0) + \dotprod{\nabla_\theta \Lp(\theta_0)}{\Delta \theta} + \frac{1}{2} \Delta \theta^\top \nabla_\theta^2 \Lp(\theta_0) \Delta \theta \\
    \approx& \Lp(\theta_0) + \frac{1}{2} \Delta \theta^\top \nabla_\theta^2 \Lp(\theta_0) \Delta \theta ,
\end{split}
\end{align*}
where we denote $\Lp$ as the pretraining loss, $\theta_0 $ as the model weight after pretraining, and $\Delta \theta$ as the weight change after SFT, whose scale should be small. The second line holds since the model should approximately arrive at a local minima after pretraining, which yields that $\nabla_\theta \Lp(\theta_0) \approx 0$.

Based on the observations that the Hessians of neural networks are typically block-wise diagonal with respect to layers~\citep{martens2015optimizing,singh2021analytic,nayak2025sculpting}, we then consider the forgetting loss in a layer-wise manner without loss of generality, since the total loss can be obtained by summing up all the layers.
For a linear layer $W \in \RR^{m\times n}$ of the model $\theta$, we define
\begin{align*}
    \Lforgeth(\Delta W) \triangleq \frac{1}{2} \vec{\Delta W}^\top \nabla_W^2 \Lp(\theta_0) \vec{\Delta W} ,
\end{align*}
where $\Delta W$ is the update of the layer $W$ after SFT and the $\vec{W}$ operation means vectorizing $W$ by stacking the rows, i.e., $\vec{W} \triangleq [w_1,\dots, w_m] \in \RR^{mn \times 1}$ for $W = [w_1^\top;\dots;w_m^\top] \in \RR^{m\times n}$.
We have the input and output activations of the layer $W$ to be $x,y$ such that $y=Wx$ when the pretraining data is fed to the model.
We further consider the Fisher Information matrix $\EE[gg^\top]$ as an effective approximation of the Hessian $\nabla_W^2 \Lp(\theta_0)$~\citep{martens2015optimizing}, where $g = \vec{\nabla_W \Lp(\theta_0)}$ is the vectorized gradient of layer $W$ at $\theta_0$. Based on the chain rule, it holds that $\nabla_W \Lp(\theta_0) = \delta x^\top$ with $\delta \triangleq \nabla_y \Lp(\theta_0)$. Then, by substituting this gradient form, we have
\begin{align*}
    \nabla_W^2 \Lp(\theta_0) \approx& \EE[gg^\top] = \EE[\vec{\delta x^\top} \vec{\delta x^\top}^\top] = \EE[(\delta\otimes x) (x \otimes \delta)^\top] \\
    =& \EE[(\delta\delta^\top) \otimes (xx^\top) ] \approx \EE[\delta \delta^\top] \otimes \EE[xx^\top],
\end{align*}
where $\otimes$ denotes the Kronecker product, and the expectation is taken over the input $x$. The second approximation applies the assumption in \citet{martens2015optimizing} that we can split the expectation for input activations and output gradients. 
Based on the observations that the Hessians are commonly block-wise diagonal with respect to rows of MLPs~\citep{collobert2004large,zhang2024transformers,zhang2024adam}, we further simplify the Hessian approximation to be mostly relevant to the input activation expectation, i.e., $\nabla_W^2 \Lp(\theta_0) \propto \EE[xx^\top]$.
Then, we can approximate the forgetting loss by
\begin{align*}
\begin{split}
    \Lforgeth(\Delta W) \triangleq& \frac{1}{2} \vec{\Delta W}^\top \nabla_W^2 \Lp(\theta_0) \vec{\Delta W} \\
    \propto & \frac{1}{2} \tr{\Delta W \EE[xx^\top] \Delta W^\top} 
    = \frac{1}{2} \EE\left[ \Norm{\Delta W x}_{2}^2 \right] ,
\end{split}
\end{align*}
where the last equality is based on the property of trace.
Therefore, we take Assumption~\ref{asm:forgetting_loss}.
\begin{assumption}[Approximation of forgetting loss increase]\label{asm:forgetting_loss}
    Consider the SFT update $\Delta W \in \RR^{m\times n}$. The increase in loss caused by SFT can be approximated by
    \begin{align}\label{eq:asm_loss_forget}
        \Lforget(\Delta W) \triangleq \frac{1}{2} \EE\left[ \Norm{\Delta W x}_{2}^2 \right] ,
    \end{align}
    where $x$ denotes the input activation of pretraining data at $\theta_0$ and the expectation is taken over $x$.
\end{assumption}

Moreover, to better quantify the multiplication of weights and activations, we consider their alignment and propose Assumption~\ref{asm:weight_activation_alignment}. Note that this alignment condition with $\alpha=\beta=2$ between weight and the corresponding activations in the pretraining stage has been employed and examined in~\citet{yang2023spectral}. We here consider a more general version with $\alpha,\beta$, and since we always have $\Norm{Wx}_\beta \le \Norm{W}_{\alpha,\beta} \Norm{x}_\alpha$ based on the definition of $\Norm{\cdot}_{\alpha,\beta}$, Assumption~\ref{asm:weight_activation_alignment} generally implies that the best pair $\alpha,\beta$ can achieve this lower bound. 
Throughout the paper, we use $\Theta(\cdot)$ and $\cO(\cdot)$ notation to omit absolute constants not related to dimensionality and the variable scales.

\begin{assumption}[Weight activation alignment]\label{asm:weight_activation_alignment}
    Consider the SFT update $\Delta W \in \RR^{m\times n}$ and the input activation of pretraining data $x \in \RR^{n}$ at $\theta_0$. Fixing $\beta_*\in [1,\infty]$, there exist $\alpha_* \in [1,\infty]$ such that
    \begin{align*}
        \EE\left[ \Norm{\Delta Wx}_{\beta_*}^2 \right] = \Theta\left( \Norm{\Delta W}_{\alpha_*,\beta_*}^2 \EE\left[ \Norm{x}_{\alpha_*}^2 \right] \right) .
    \end{align*}
\end{assumption}

\textbf{Approximating SFT loss.} Since the model moves only a small distance during SFT, we consider the following first-order Taylor expansion as an approximation of the SFT loss.
\begin{assumption}[SFT loss approximation]\label{asm:sft_loss}
    We approximate the SFT loss by
    \begin{align*}
        \cL_{\mathrm{SFT}}(\Delta W) \triangleq H_0 + \dotprod{\Delta W}{G} ,
    \end{align*}
    where $G \in \RR^{m\times n}$ is the gradient of the SFT loss at $\theta_0$ and $H_0$ is the SFT loss at $\theta_0$.
\end{assumption}

\textbf{Optimizer class.} We consider the optimizer class following framework~\eqref{eq:steepest_descent} with $\Norm{\cdot}_{\alpha,\beta}$, including Adam and Muon. We formalize the observations on the activation regularization effects of optimizers observed in Figures~\ref{fig:adamw_muon_activation_sparsity} and \ref{fig:optimizer_activation_sparsity} as in Assumption~\ref{asm:optimizer_choice_activation_property}. While this assumption may be slightly strong for the activations, it provides a straightforward alignment with the intuition that, when a vector is regularized by $\Norm{\cdot}_p$, it will be shaped to minimize $\Norm{\cdot}_p$ while other norms can be larger.

\begin{assumption}[Optimizers regularize activations]\label{asm:optimizer_choice_activation_property}
    We assume the pretraining and SFT optimizers are chosen based on the framework~\eqref{eq:steepest_descent}. For pretraining optimizer $\cA_{\alpha_1,\beta_1}$ and SFT optimizer $\cA_{\alpha_2,\beta_2}$ and the SFT update $\Delta W$ derived by $\cA_{\alpha_2,\beta_2}$ and $y\triangleq \Delta W x$, it holds that for all $\alpha,\beta \in [1,\infty]$, 
    \begin{small}
    \begin{align*}
        \EE\left[\Norm{x}_{\alpha_1}^2 \right] = \begin{cases}
            \Theta \left( \EE\left[ \Norm{x}_{\alpha}^2 \right] \right), & \alpha \ge \alpha_1, \\
            \Theta \left( n^{\frac{1}{\alpha_1} - \frac{1}{\alpha}} \EE\left[ \Norm{x}_{\alpha}^2 \right] \right), & \alpha < \alpha_1 ,
        \end{cases} \; \text{and} \; \;
        \EE\left[ \Norm{y}_{\beta_2}^2 \right] = \begin{cases}
            \Theta \left( \EE\left[ \Norm{y}_{\beta}^2 \right] \right), & \beta \ge \beta_2, \\
            \Theta \left( m^{\frac{1}{\beta_2} - \frac{1}{\beta}} \EE\left[ \Norm{y}_{\beta}^2 \right] \right), & \beta < \beta_2 .
        \end{cases}
    \end{align*}
    \end{small}

\end{assumption}

Under these assumptions, we show that full finetuning with the same family of optimizer outperforms other optimizer choices in Theorem~\ref{thm:same_opt_forgets_less} for this simplified problem scenario.

\begin{theorem}[Same (family of) optimizer forgets less]\label{thm:same_opt_forgets_less}
    Under Assumption~\ref{asm:forgetting_loss}-\ref{asm:optimizer_choice_activation_property}, we can achieve (near) optimal learning-forgetting tradeoff when taking $\alpha_2 = \alpha_1$ and $\beta_2 \in [2,\infty]$ for the SFT optimizer $\cA_{\alpha_2,\beta_2}$, i.e., the least forgetting while achieving the same SFT performance. Formally, for any constant $C$ and $\Delta W (\cA_{\alpha,\beta})$ such that $\Lsft(\Delta W(\cA_{\alpha,\beta})) \le C$, we have
        \begin{align*}
            \Lforget\left( \Delta W(\cA_{\alpha_2,\beta_2}) \right) = \cO\left( \inf_{\alpha\in [1,\infty] , \beta \in[1,2] } \Lforget\left( \Delta W(\cA_{\alpha,\beta}) \right) \right)
        \end{align*}
    with $\alpha_2 = \alpha_1$ and $\beta_2 \ge 2$, where $\Delta W (\cA_{\alpha,\beta})$ denotes the SFT update produced by $\cA_{\alpha,\beta}$.
\end{theorem}

\begin{remark}[A suggestion for the training choice of SFT]
    We show that full finetuning with optimizers in the same family as the pretraining optimizer (or more concretely, with the same $\alpha$ and $\beta \ge 2$) obtains the best learning-forgetting tradeoff. For instance, we may consider Adam, SignSGD~\citep{bernstein2018signsgd}, or Lion~\citep{chen2023symbolic} and their variants for finetuning AdamW checkpoints. 
    Also note that LoRA generally fails to achieve the best tradeoff under the settings, since its adapter structure cannot fit the optimal choice for activations regularized by any norm, as long as we pretrain the model with an optimizer derived from~\eqref{eq:steepest_descent}, which includes Adam, Muon, and their variants.
\end{remark}

\section{A Case Study on Muon Applied to Pretrain and SFT}\label{sec:muon_better_memorization}

Given that we should employ the same (family of) optimizer in SFT as in pretraining to obtain the best performance because of optimizer-model consistency, the learning-forgetting tradeoff when an optimizer is employed throughout the pretrain and SFT stages is a valid and comprehensive metric for the optimizer in training LLMs, since finetuning is usually needed to improve the model.
In this section, we provide a case study to compare Muon with AdamW in this way to examine its potential in LLM training. We employ the same settings as presented in Section~\ref{sec:experiments_same_optimizer_forgets_less} for GPT-2.

\subsection{Muon Performs Potentially Worse for Reasoning}\label{sec:muon_better_memorization_experiment_results}

When Muon and AdamW are employed throughout the pretraining and SFT stages, the learning-forgetting tradeoff is presented in Figure~\ref{fig:ckpt+optimizer_comparison}.

\begin{figure}[ht]
\centering
\vspace{-0.2cm}

\begin{tabular}{ccc}
	\includegraphics[width=0.3\linewidth]{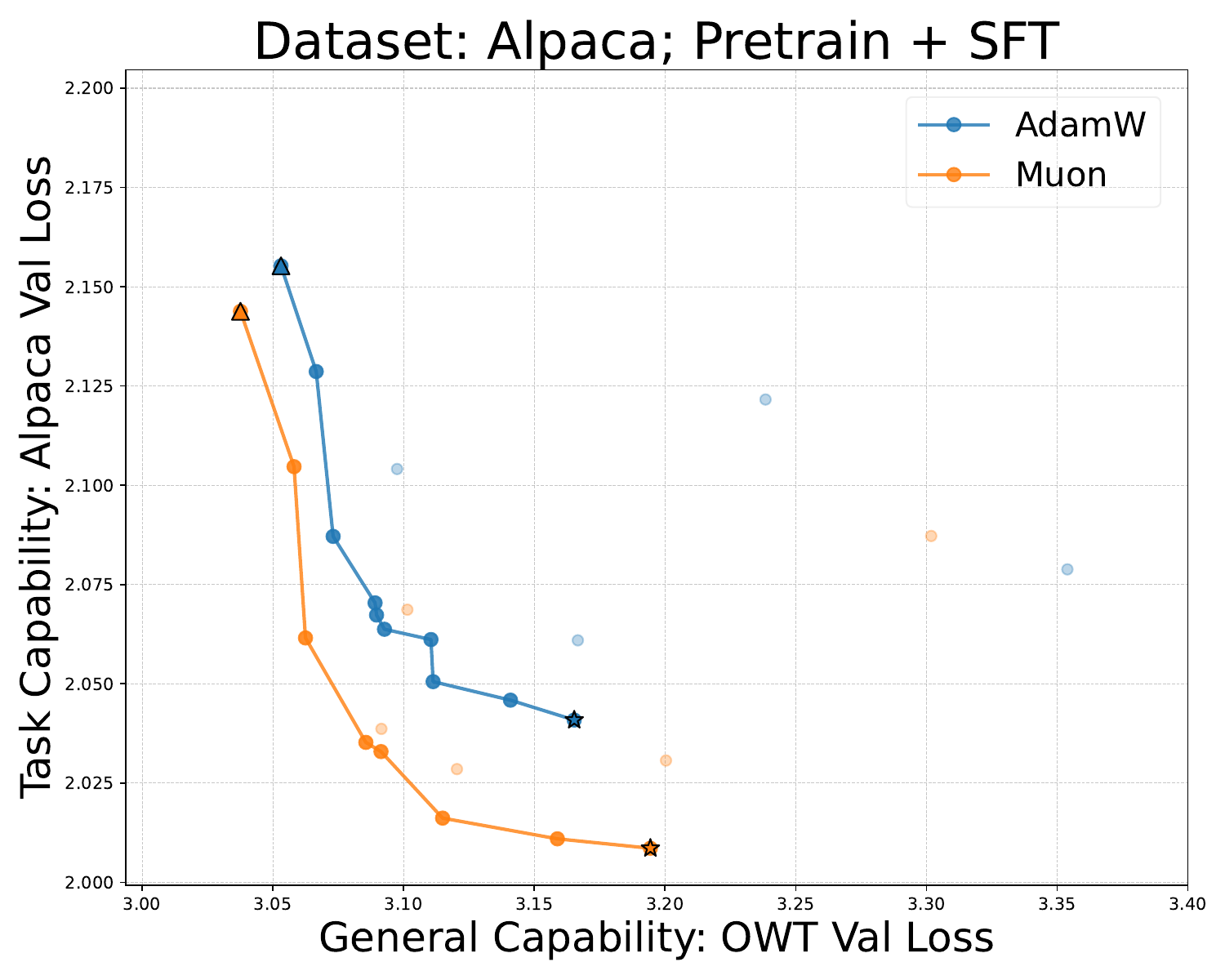}
	\!\!\!
	& \includegraphics[width=0.3\linewidth]{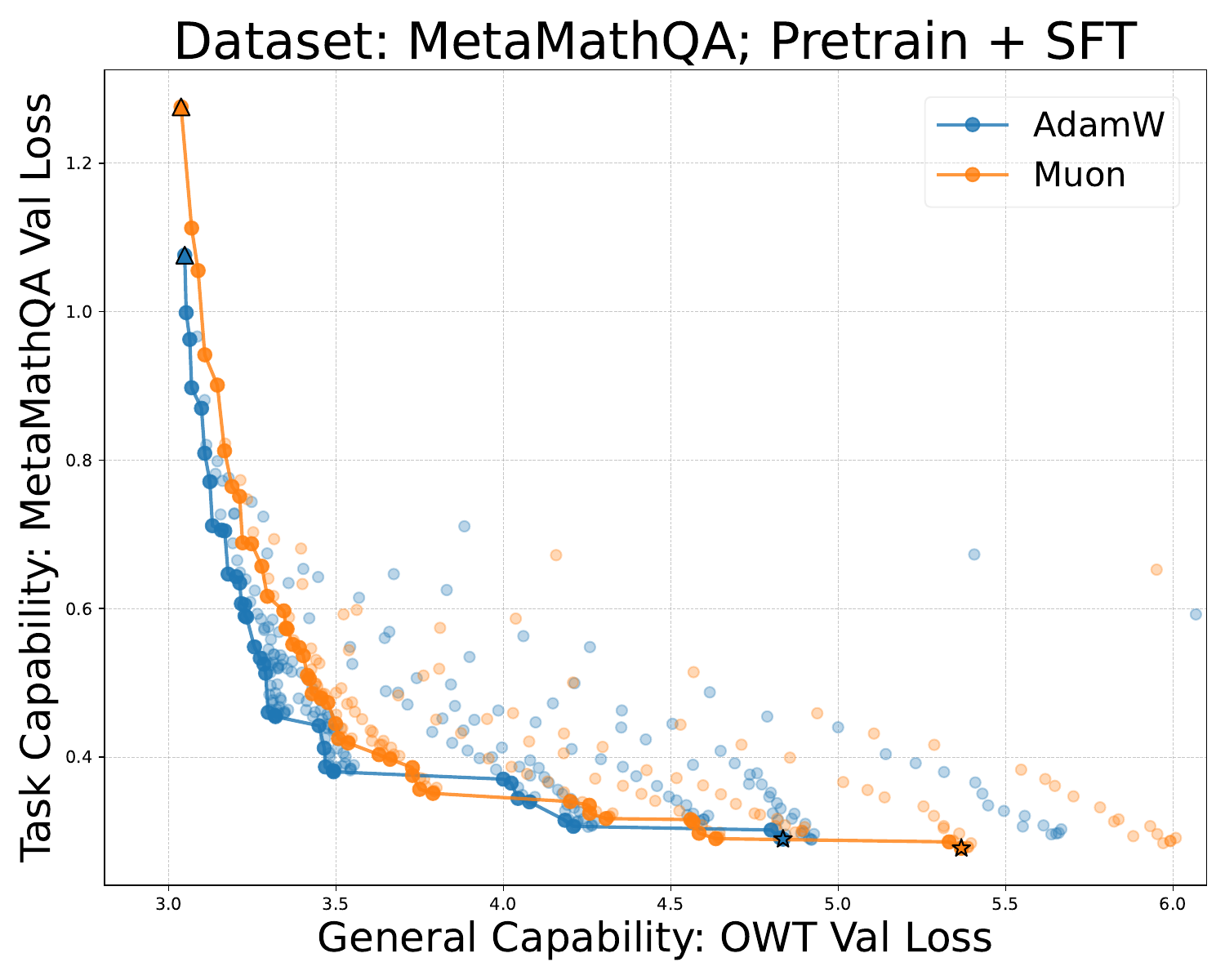} 
    \!\!\!
	& \includegraphics[width=0.3\linewidth]{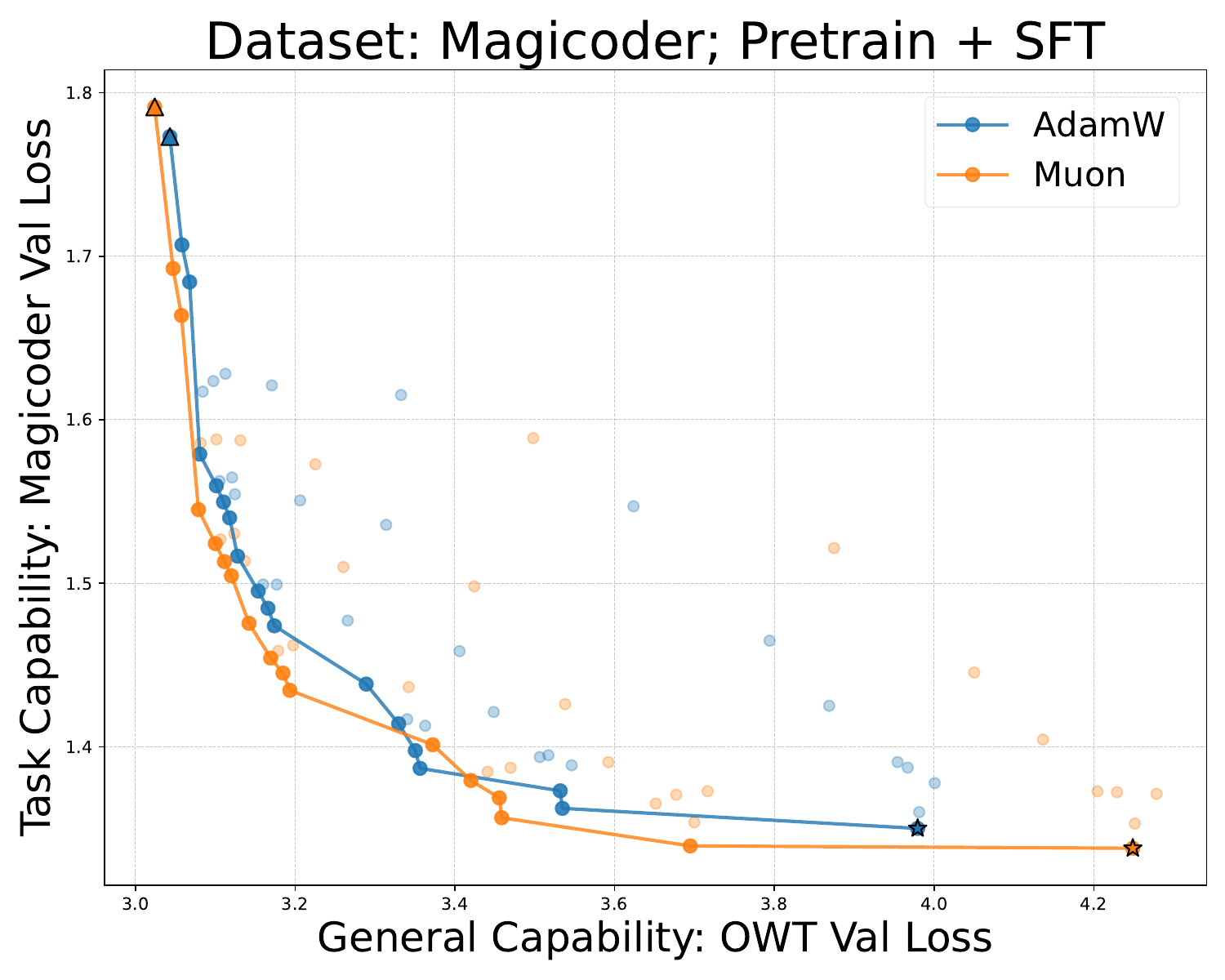} 
\end{tabular} 

\vskip-0.2cm
\caption{The comparison between Muon and AdamW when they are employed throughout the pretraining and SFT stages in Alpaca, MetaMathQA, and Magicoder tasks. } \vspace{-0.2cm}
\label{fig:ckpt+optimizer_comparison} 
\end{figure}

We can clearly observe that Muon achieves a significantly better learning-forgetting tradeoff when finetuned on Alpaca than AdamW. On math and code tasks, however, Muon is generally comparable to AdamW: slightly better in Magicoder, and slightly worse in MetaMathQA.
As demonstrated in Table~\ref{tab:pretrain_results}, Muon performs exceptionally better compared to AdamW before SFT in all three SFT tasks and the pretraining task. 
Therefore, the results in Figure~\ref{fig:ckpt+optimizer_comparison} generally imply a degradation in the performance of Muon after finetuning on code and math tasks, showing a potentially worse capability of learning reasoning ability in SFT with Muon.

\subsection{Muon Memorizes More, but May Learn Less}\label{sec:experiment_corrupted_real}
We hypothesize that \textit{Muon's worse performance on finetuning reasoning tasks is because Muon tends to memorize before patterns in the data have been fully acquired}, which can hurt learning when only a small amount of data is available, as in the SFT stage. This is consistent with Muon's performance variation in SFT, since reasoning tasks generally require learning new abilities, while instruction following tasks are more related to memorization.
Our hypothesis originates from the fact that each update of Muon is full-rank, being able to memorize more knowledge compared to the low-rank updates of AdamW and SGD~\citep{yang2023spectral}, implying that rote memorizing data is easier for Muon. When memorizing is easier, models can tend to do rote memorization instead of learning patterns to lower loss faster, leading to worse generalization and undesirable performance.

To verify our hypothesis, we conduct a synthetic experiment on a corrupted language corpus, in which we split the data into clean and corrupted sequences. The corrupted sequences are generated from the original clean sequences by randomly shuffling the order of tokens, e.g., for a complete sentence ``\textit{attention is all you need}", a random shuffle can be ``\textit{is need you all attention}". In this way, we destroy the internal patterns of language in corrupted sequences, so that only rote memorization is possible, similar to the randomization test~\citep{edgington2007randomization} applied in \citet{zhang2016understanding}. 
By mixing clean and corrupted sequences, we can assess the optimizers' tendency to memorize and learn patterns by checking how well they learn on clean and corrupted data separately.
The detailed experiment settings and the evaluation metric of exact match are in Appendix~\ref{appendix:experiment_setting_corrupted_real}.

\textbf{Models learn patterns first.} The memorization accuracy vs iteration plot is presented in Figure~\ref{fig:memorization_accuracy}. We can observe that with all three optimizers, models memorize the clean data much faster than the corrupted random data, which generally implies the existence of patterns in natural language, making it easier to learn. Also, the models start to learn corrupted data (by rote memorization) only after having some level of capacity for the clean data, which is consistent with the observation that models learn simple patterns before memorization in computer vision tasks~\citep{arpit2017closer}.

\textbf{Muon tends to memorize before fully acquiring the patterns.} 
As we can observe from Table~\ref{tab:memorization_accuracy_corrupted_real} and Figure~\ref{fig:memorization_accuracy}, Muon achieves higher memorization accuracy for the corrupted data, but lower for clean data, compared to AdamW and SGD. Also, the corrupted data memorization accuracy of Muon starts to increase much earlier than the other two optimizers, showing a tendency to start rote memorization earlier.
Also note that the train loss of Muon is still the lowest, showing the fastest speed in lowering loss, which, however, leads to rote memorization rather than learning patterns. 
The results generally support our hypothesis and partially explain the SFT performance of Muon.

\begin{figure}[t]
\centering

\begin{tabular}{ccc}
	\includegraphics[width=0.3\linewidth]{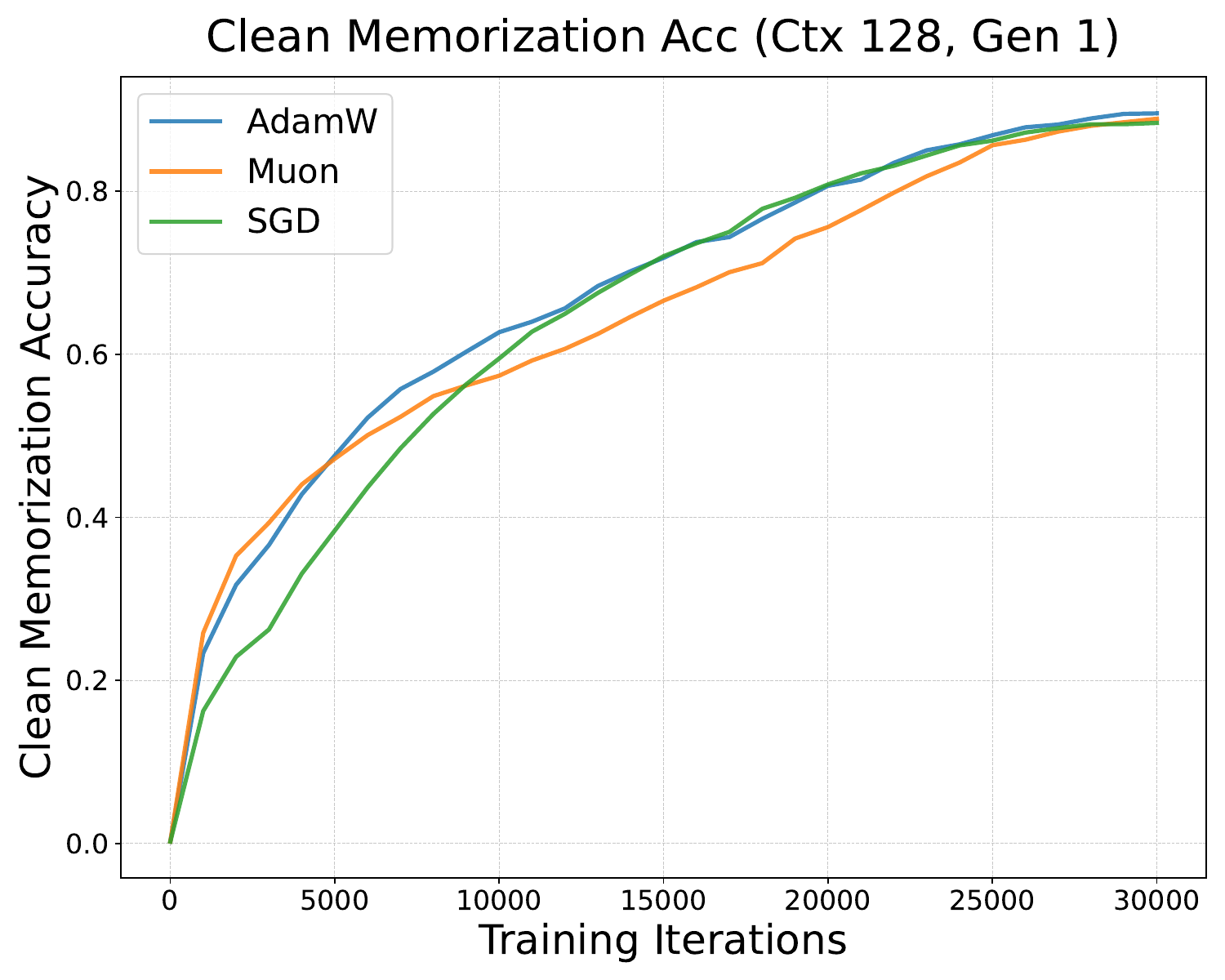}
	\!\!\!
	& \includegraphics[width=0.3\linewidth]{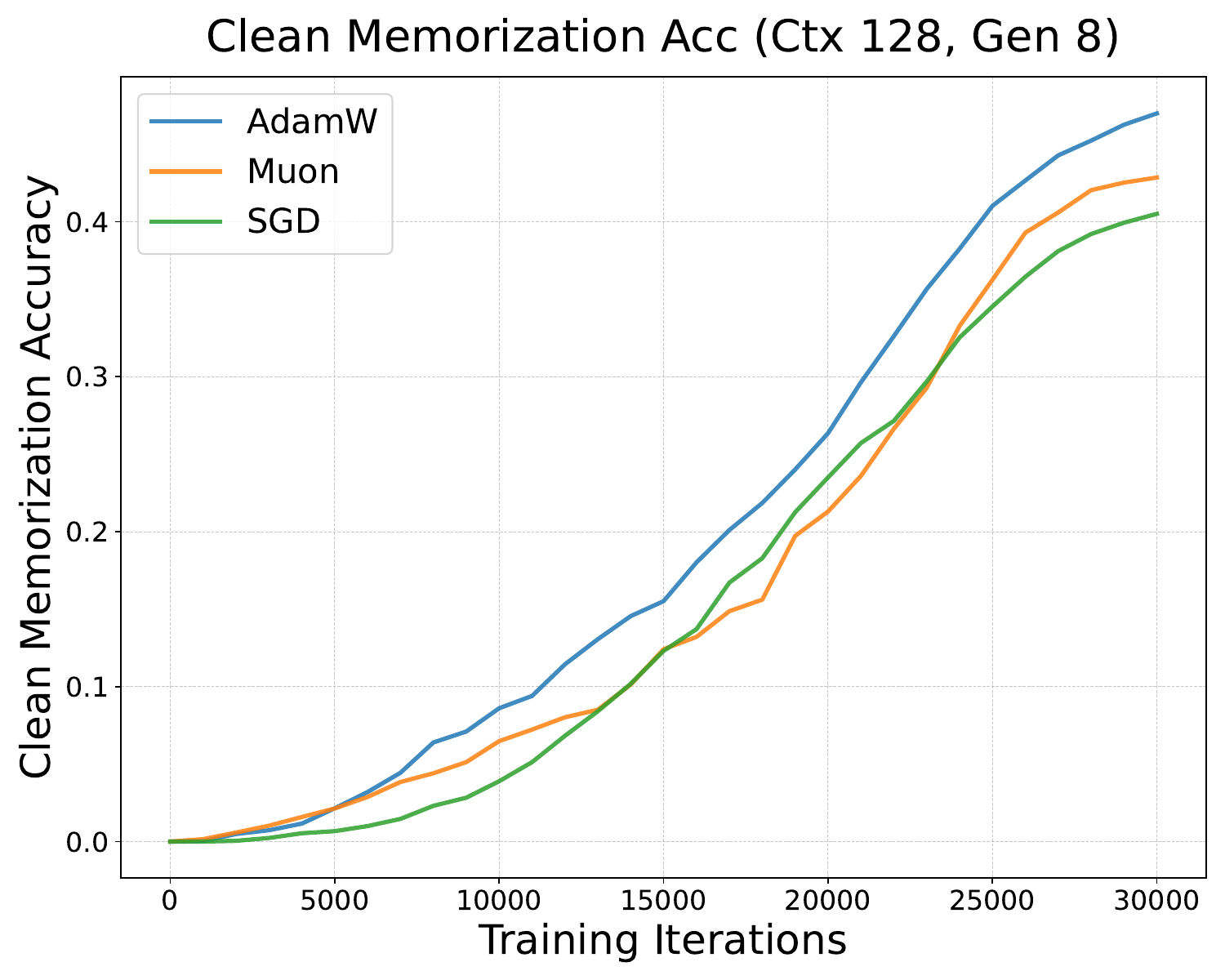} 
    \!\!\!
	& \includegraphics[width=0.3\linewidth]{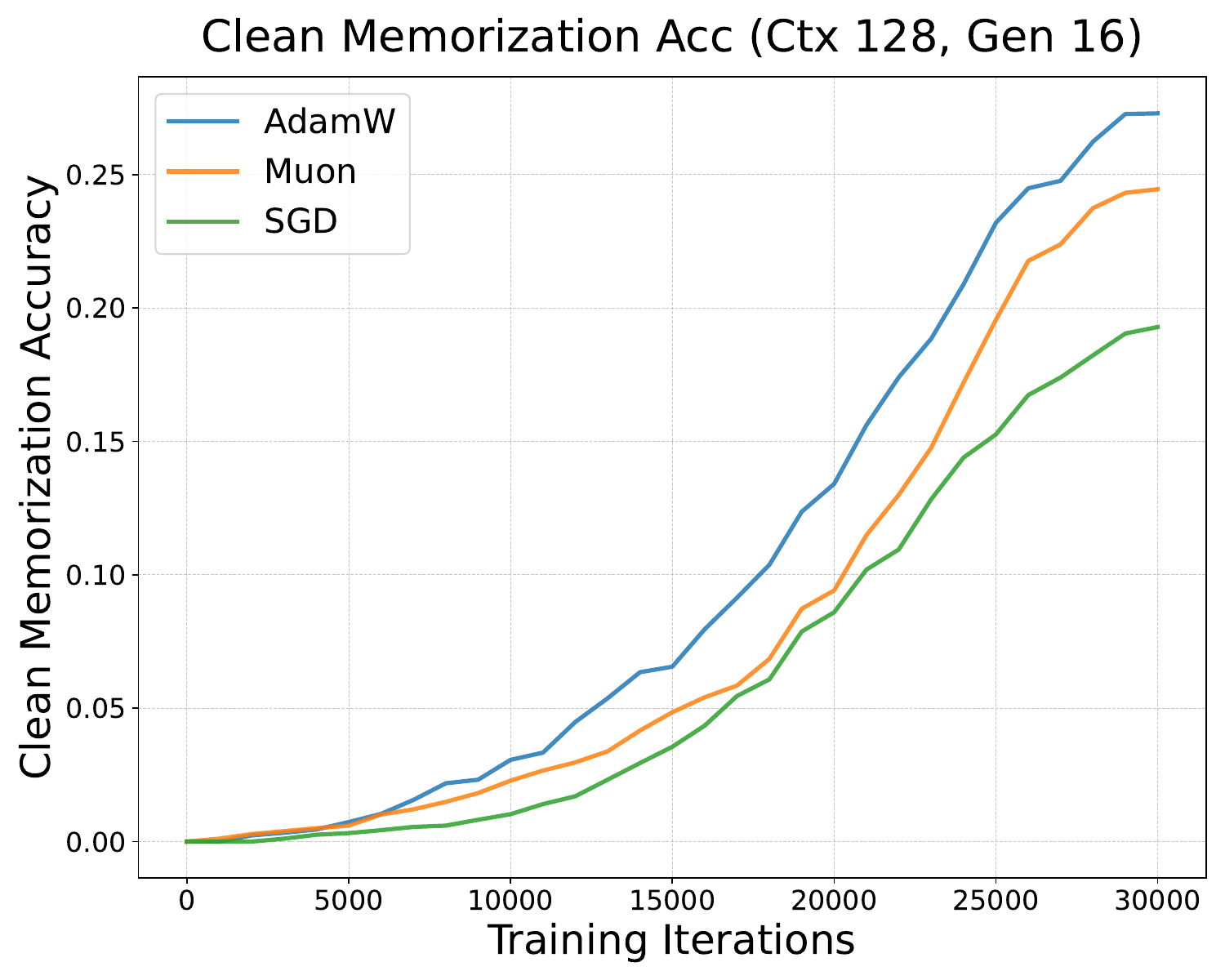} \\
    \includegraphics[width=0.3\linewidth]{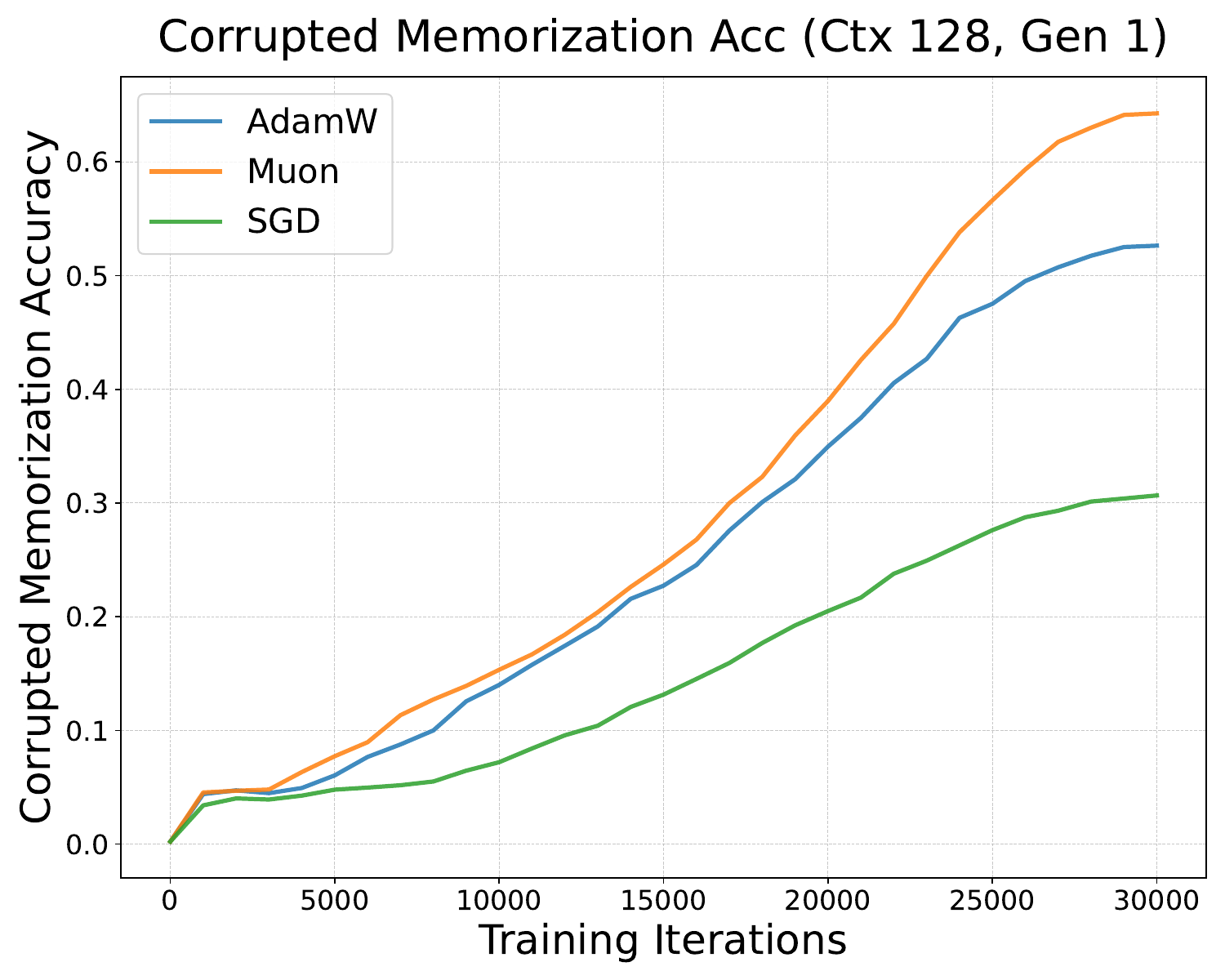}
	\!\!\!
	& \includegraphics[width=0.3\linewidth]{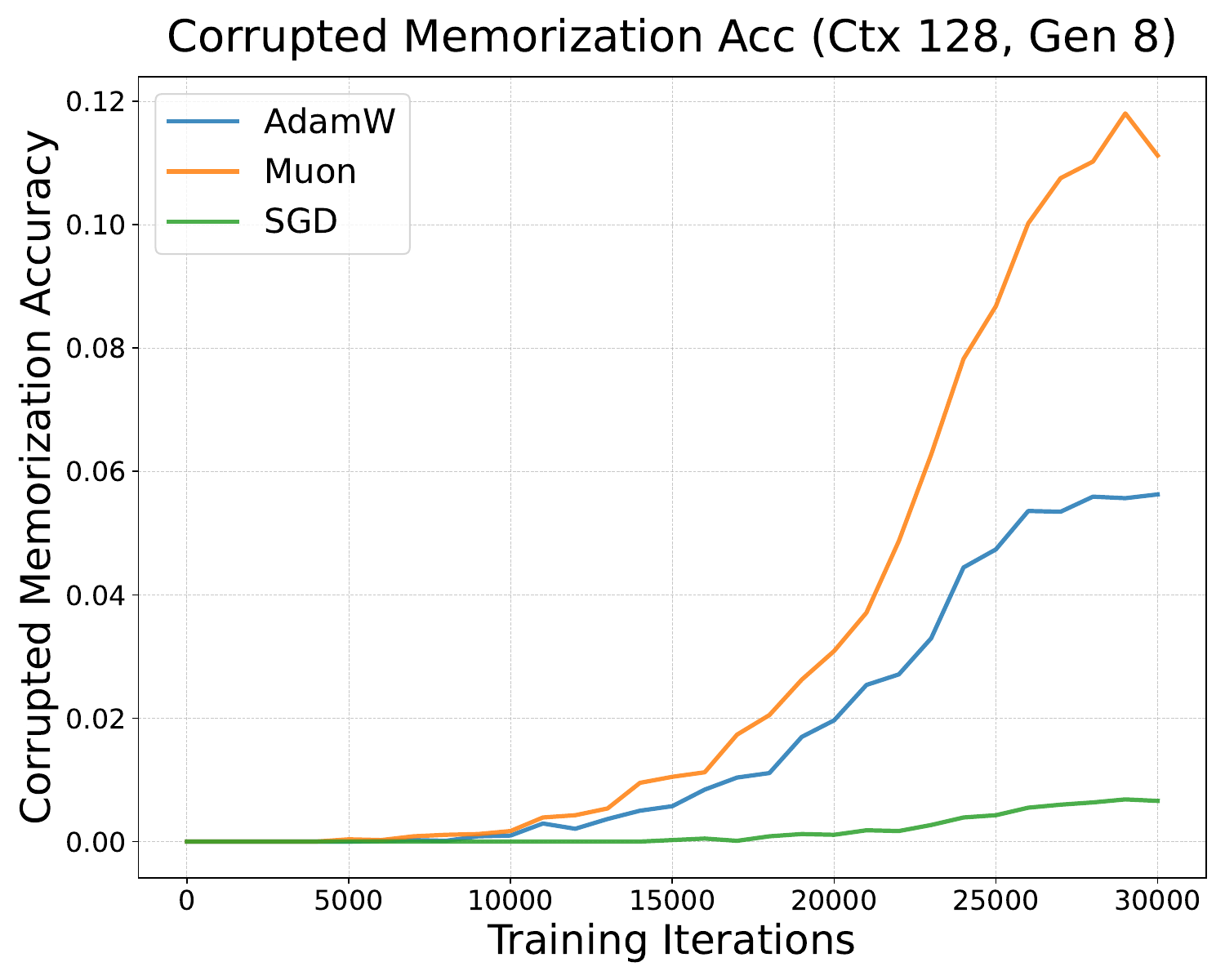} 
    \!\!\!
	& \includegraphics[width=0.3\linewidth]{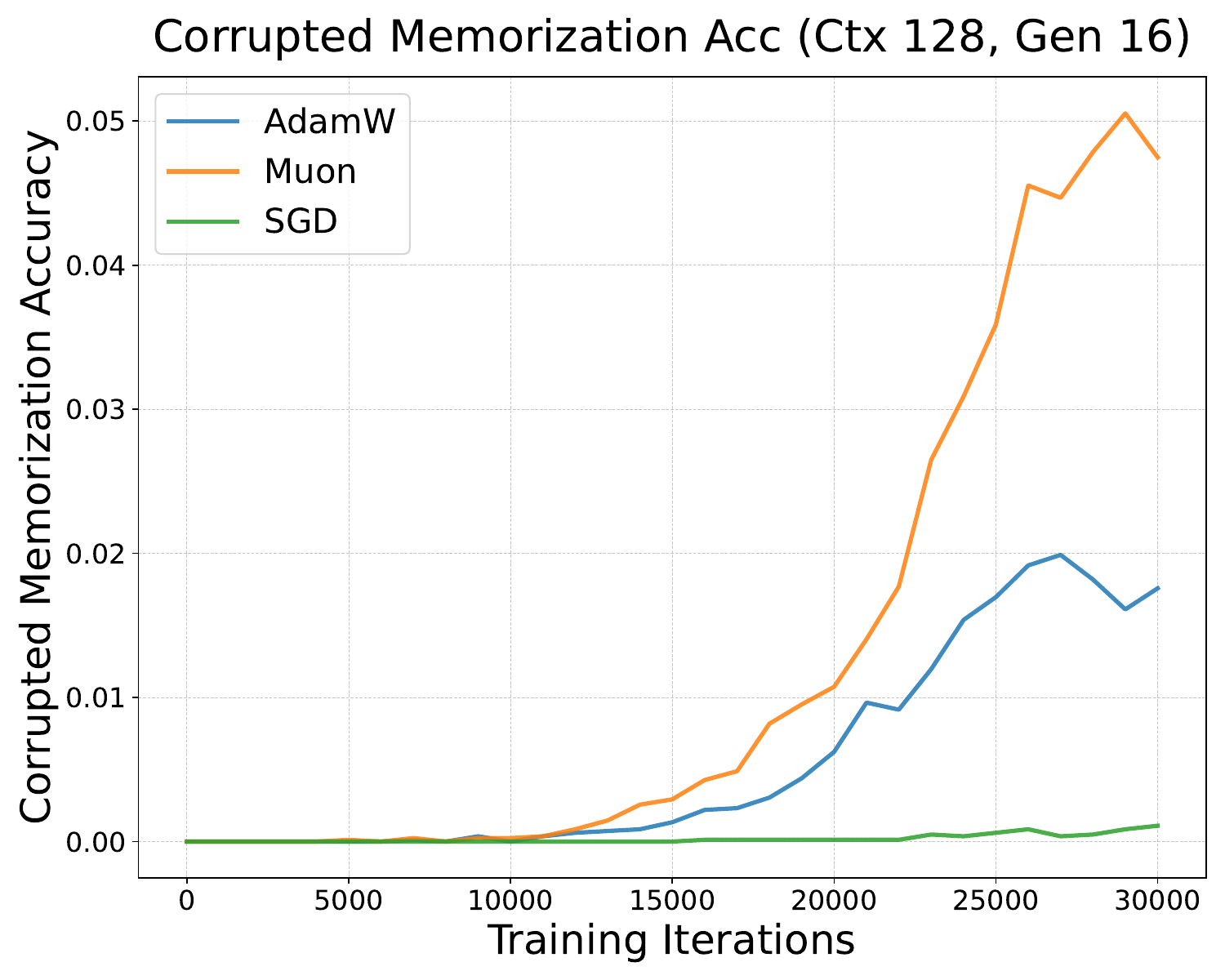}
\end{tabular} 

\vskip-0.2cm
\caption{The memorization accuracy of clean and corrupted data under different generation lengths.} \vspace{-0.5cm}
\label{fig:memorization_accuracy} 
\end{figure}

\begin{table}[ht]
    \centering
    \vspace{-0.4cm}
    \caption{The memorization accuracy and final train loss of models trained by different optimizers. We employ different input context lengths and completion lengths to test the exact match accuracy, e.g., $(128, 8)$ means input context length $128$ and generation length $8$. } \vspace{-0.1cm}
    \label{tab:memorization_accuracy_corrupted_real}
\vspace{0.2cm}
\begin{tabular}{lcccccccc}
\toprule
\multirow{2}{*}{Optimizer} & \multicolumn{3}{c}{Memorization Accuracy Clean} & \multicolumn{3}{c}{Memorization Accuracy Corrupted} 
& \multirow{2}{*}{Train Loss} 
\\
\cmidrule(lr){2-4} \cmidrule(lr){5-7} 
& $(128,1)$ & $(128,8)$ & $(128,16)$ & $(128,1)$ & $(128,8)$ & $(128,16)$ \\
\midrule

Muon & 88.9\% & 44.8\% & 25.4\% & \textbf{65.0\%} & \textbf{12.2\%} & \textbf{5.2\%} & 1.143 \\

AdamW & 90.1\% & 49.8\% & 29.7\% & 55.8\% & 7.0\% & 2.3\% & 1.440 \\

SGD & \textbf{93.4\%} & \textbf{60.3\%} & \textbf{40.8\%} & 32.9\% & 1.3\% & 0.0\% & 2.310 \\

\bottomrule
\end{tabular}

\vspace{-0.1cm}

\end{table}

\begin{remark}[A better finetuning approach for Muon would be helpful]
    For Muon pretrained models, we should apply full finetuning with Muon in order to obtain the best learning-forgetting tradeoff based on optimizer-model consistency, while Muon in SFT may potentially hurt the overall performance, which creates a dilemma. Therefore, a better finetuning approach for Muon that can give a better tradeoff for both sides would be helpful for improving the method in LLM training.
\end{remark}

\section{Related Work}\label{sec:related_work}

\textbf{Optimizers for Deep Learning.} 
The current workhorse of optimizers for training deep learning models, or more specifically, transformers~\citep{vaswani2017attention}, is undoubtedly adaptive gradient methods, including AdamW~\citep{kingma2014adam,loshchilov2017decoupled} and many variants~\citep{duchi2011adaptive,streeter2010less,tieleman2012lecture,zeiler2012adadelta,shazeer2018adafactor,reddi2019convergence,you2019large,zhang2024adam,defazio2024road,yuan2024mars,defazio2025gradients,bernstein2018signsgd,crawshaw2022robustness,chen2023symbolic,yu2026stosignsgd}, to mention a few.
Recently, optimizers aware of the matrix structure of the weights have emerged as a strong opponent to the adaptive gradient methods, showing great potential in training transformers. KFAC~\citep{martens2015optimizing} and Shampoo~\citep{gupta2018shampoo} are two pioneering works utilizing the neural network structure in optimizer design, which is further explored in more recent works~\citep{ren2021tensor,zhao2024galore,jordan2024muon,vyas2024soap,liu2025muon,an2025asgo,xie2025structured,pan2025unbiased,xie2026controlled,wen2025hyperball}, to mention a few, in either saving memory footprint or improving training performance.
The idea has become particularly popular after Muon~\citep{jordan2024muon,liu2025muon} has been applied to training real frontier models~\citep{team2026kimi,zeng2025glm,deepseekai2026deepseekv4}. 
To compare and understand the performance of optimizers, research on empirically benchmarking and tracking the performance~\citep{wen2025fantastic,semenov2025benchmarking,wang2025muon} or theoretically analyzing the convergence properties~\citep{duchi2011adaptive,liu2024adagrad,jiang2024provable,su2025isotropic} has emerged, but mainly focused on pretraining or one training setting rather than throughout the pretrain-SFT paradigm.

\textbf{Learning-Forgetting Tradeoff.}
Models tend to lose the information learned from a previous task when training on a new task. This phenomenon, termed catastrophic forgetting~\citep{mccloskey1989catastrophic,mcclelland1995there,kirkpatrick2017overcoming}, is a central problem in continual learning and multitask learning. Researchers have proposed numerous methods to find a better tradeoff between learning the new task and forgetting the old task, including introducing regularization~\citep{kirkpatrick2017overcoming,zenke2017continual,li2017learning} and replaying samples in previous tasks~\citep{lopez2017gradient,rebuffi2017icarl,rolnick2019experience}.
Studies have been conducted to examine LoRA~\citep{hu2022lora} for finding a better learning-forgetting tradeoff. \citet{biderman2024lora} claims that LoRA forgets less than full finetuning, although it also learns less, focusing on continual pretraining and SFT. However, on sequential continual learning settings, there is simultaneously evidence showing that LoRA can resist forgetting~\citep{wistuba2023continual,qiao2024learn} and LoRA may not forget less~\citep{shuttleworth2024lora} in varying settings and viewpoints. \citet{springer2025overtrained,rofin2026learning} also discuss the forgetting in SFT and draw a link to the pretraining length and SFT learning rate choice.

\textbf{Memorization and Generalization.}
Memorization and generalization of the model, and their relations, are one of the fundamental problems of machine learning. 
Traditional statistical learning theory provides viewpoints from VC dimension~\citep{vapnik1998adaptive}, Rademacher complexity~\citep{bartlett2002rademacher}, and uniform stability~\citep{bousquet2002stability}. 
\citet{zhang2016understanding} provides evidence that deep learning models can easily memorize data while generalizing well with the randomization test methodology applied to image classification tasks, showing that such learning tasks may go beyond the scope of the traditional theory. Building upon the methodology, \citet{arpit2017closer} further shows that simple patterns are learned first by the neural networks compared to noise, implying that model learning and generalization is something more than rote memorization. Recent work also extends the study to modern large foundation models~\citep{carlini2021extracting,carlini2022quantifying,tirumala2022memorization,jagielski2022measuring,schwarzschild2024rethinking,panwar2025better}, revealing an interesting but complex mechanism of the memorization and generalization of the models.

\section{Conclusions}\label{sec:discussion_conclusion}
\vspace{-0.05cm}
In this paper, we present and analyze the optimizer-model consistency phenomenon. Based on the understanding of it, we suggest that full finetuning with the same (family of) optimizer as pretraining should be employed in SFT to achieve the best learning-forgetting tradeoff and also conduct a case study on the performance of Muon throughout the pretrain-SFT paradigm.
The major limitation of the work is that the experiments are mainly conducted on small models due to the requirement of pretraining and limited computational resources, and the theory is obtained under approximations.
We will be interested in scaling up our experiments and exploring more optimizers for future work.

\bibliographystyle{plainnat}
\bibliography{sample}

\appendix

\section{Experiment Settings}\label{appendix:experiment_setting}

\subsection{Experiment Settings in Section~\ref{sec:experiments_same_optimizer_forgets_less}}\label{appendix:experiment_settings_same_optimizer_forgets_less}
The GPT-2 pretraining experiments are done with 4 NVIDIA RTX A6000 GPUs with 48 GB of memory, GPT-2 SFT experiments with 1 A6000 GPU, and Llama-2-7B finetuning experiments with 4 A6000 GPUs. All experiments are conducted on bf16 precision.

For GPT-2 experiments, we adopt a training codebase modified from the nanoGPT codebase~\citep{karpathy2022nanogpt}.
For Llama-2-7B experiments, we employ the LlamaFactory~\citep{zheng2024llamafactory} codebase and Muon fsdp2 implementation from the Dion codebase~\citep{ahn2025dion,ahn2025dion2}.

\textbf{Training Algorithms.} In this paper, we typically consider three optimizers: AdamW, Muon, and SGD. For the SFT stage, we consider both full finetuning with the three optimizers and LoRA~\citep{hu2022lora} with AdamW. Specifically, we use the RMSNorm-aligned version of Muon as suggested in~\citep{liu2025muon}. We also apply a similar update RMSNorm alignment to SGD in this paper, i.e., the update of SGD for a layer $W \in \RR^{m\times n}$ is $W \gets W - \eta \cdot 0.2 \cdot \sqrt{mn} g / \Norm{g}$, where $g$ is the momentum.

\textbf{Model \& Pretraining.}
For the model, we adopt the same architecture and parameter number as the GPT-2-small model~\citep{ouyang2022instructgpt} (approximately 123M), changing only the activation function from GELU~\citep{hendrycks2016gaussian} to SwiGLU~\citep{shazeer2020glu}. We train the models from scratch with 8B tokens on the OpenWebText dataset~\citep{gokaslan2019openWeb}. For pretraining, we adopt AdamW and Muon with learning rate 3e-3. 

\textbf{SFT.} We consider three different types of tasks for SFT: MetaMathQA~\citep{yu2023metamath} for math, Alpaca~\citep{alpaca} for instruction following, Magicoder-Evol-Instruct-110K~\citep{wei2023magicoder} for coding. 
We employ full finetuning and LoRA with $r=4,8,32,128$ and $\alpha=2r$ following the suggesion in \citet{biderman2024lora}. 
For each training algorithm, we apply a grid search for learning rates, with other hyperparameters fixed. The grid always contains the learning rate that achieves the best validation loss in the SFT task, and is listed in Appendix~\ref{appendix:detailed_experiment_settings}. 

\textbf{Evaluation.} We care about two major metrics for evaluation: the OpenWebText validation loss as a measure for forgetting, i.e., how much the models preserve the pretraining knowledge, and the loss on the new task validation loss as a measure for learning, i.e., how much the models learn the new ability (lower is better). 
For a comprehensive and fair comparison between the training algorithms, we plot the Pareto frontier of the learning-forgetting tradeoff, with $x$-axis denoting the forgetting metric and $y$-axis denoting the learning metric. 
For each training algorithm, we consider all the runs in the learning rate grid search, and for a single run, we evaluate intermediate checkpoints. Then, we gather all the data points for each training algorithm to present the Pareto frontier to find the limit of this training algorithm in learning the most and forgetting the least.
When plotting the Pareto frontier figures, we use a robust frontier selection strategy for clarity of the figures. For figures with loss as measurements, for each point $(x_i,y_i)$ in the Pareto frontier of a specific algorithm, we plot it as the frontier only when 
\{$x_i \le x_j - c$ or $y_i \le y_j - c$\}
holds for all other $(x_j,y_j)$ in this specific algorithm, where $c=0.001$, meaning that we only include the points with a significant improvement compared to others. We apply this plotting strategy to figures with accuracy as measurements similarly with $c=0.001$.

\textbf{Llama-2-7B Settings.} 
We also verify the results in finetuning Llama-2-7B with MetaMathQA, following the MATH IFT scheme in \citet{biderman2024lora}.
We follow the same logic as the GPT-2 experiment to do learning rate grid search and plot the Pareto frontier of different training algorithms, changing only the measures of forgetting and learning from loss to (benchmark-averaged) accuracy (higher is better). 

\subsubsection{Detailed Experiment Settings}\label{appendix:detailed_experiment_settings}

\textbf{GPT-2 Pretraining.} GPT-2 models pretrained on OpenWebText dataset.
\begin{itemize}
    \item \texttt{sequence length}: 1024
    \item \texttt{effective batch size}: 240 (4 GPUs $\times$ 5 gradient accumulation $\times$ 12 batch size)
    \item \texttt{token}: 8B
    \item \texttt{scheduler}: cosine~\citep{loshchilov2016sgdr}, with 10\% warmup
    \item \texttt{gradient clipping}: norm (threshold=1)
    \item \texttt{weight decay}: 0.1
    \item \texttt{learning rate}: 3e-3 for both AdamW and Muon
    
\end{itemize}

\textbf{GPT-2 Alpaca.} GPT-2 models (pretrained by AdamW and Muon) finetuned on Alpaca.
\begin{itemize}
    \item \texttt{sequence length}: 1024
    \item \texttt{effective batch size}: 128 (1 GPU $\times$ 4 gradient accumulation $\times$ 32 batch size)
    \item \texttt{iterations}: 800 (approximately 2 epochs)
    \item \texttt{evaluation interval}: every 200 iterations
    \item \texttt{scheduler}: cosine~\citep{loshchilov2016sgdr}, with 10\% warmup
    \item \texttt{gradient clipping}: norm (threshold=1)
    \item \texttt{full finetuning weight decay}: 0
    \item \texttt{full finetuning dropout}: 0.1
    \item \texttt{full finetuning learning rate grid}: \{3e-5, 5e-5, 1e-4, 3e-4, 5e-4\} for both AdamW and Muon
\end{itemize}

\textbf{GPT-2 MetaMathQA.} GPT-2 models (pretrained by AdamW and Muon) finetuned on MetaMathQA.
\begin{itemize}
    \item \texttt{sequence length}: 1024
    \item \texttt{effective batch size}: 128 (1 GPU $\times$ 4 gradient accumulation $\times$ 32 batch size)
    \item \texttt{iterations}: 6100 (approximately 2 epochs)
    \item \texttt{evaluation interval}: every 300 iterations
    \item \texttt{scheduler}: cosine~\citep{loshchilov2016sgdr}, with 10\% warmup
    \item \texttt{gradient clipping}: norm (threshold=1)
    \item \texttt{full finetuning weight decay}: 0
    \item \texttt{full finetuning dropout}: 0.1
    \item \texttt{full finetuning learning rate grid}: \{5e-5, 1e-4, 3e-4, 5e-4\} for both AdamW and Muon
    \item \texttt{LoRA dropout}: 0.1
    \item \texttt{LoRA alpha}: $2r$ for each $r=4,8,32,128$
    \item \texttt{LoRA learning rate grid}: \{5e-4,1e-3,3e-3,5e-3,1e-2\} for both AdamW and Muon
\end{itemize}

\textbf{GPT-2 Magicoder.} GPT-2 models (pretrained by AdamW and Muon) finetuned on Magicoder.
\begin{itemize}
    \item \texttt{sequence length}: 1024
    \item \texttt{effective batch size}: 128 (1 GPU $\times$ 4 gradient accumulation $\times$ 32 batch size)
    \item \texttt{iterations}: 1800 (approximately 2 epochs)
    \item \texttt{evaluation interval}: every 200 iterations
    \item \texttt{scheduler}: cosine~\citep{loshchilov2016sgdr}, with 10\% warmup
    \item \texttt{gradient clipping}: norm (threshold=1)
    \item \texttt{full finetuning weight decay}: 0
    \item \texttt{full finetuning dropout}: 0.1
    \item \texttt{full finetuning learning rate grid}: \{5e-5, 1e-4, 3e-4, 5e-4, 1e-3\} for both AdamW and Muon
\end{itemize}

\textbf{Llama-2 MetaMathQA.} Llama-2-7B-hf\footnote{\url{https://huggingface.co/meta-llama/Llama-2-7b-hf}} finetuned on MetaMathQA.
\begin{itemize}
    \item \texttt{sequence length}: 1024
    \item \texttt{effective batch size}: 128 (4 GPU $\times$ 8 gradient accumulation $\times$ 4 batch size)
    \item \texttt{iterations}: 3055 (approximately 1 epoch)
    \item \texttt{evaluation interval}: every 500 iterations
    \item \texttt{scheduler}: cosine~\citep{loshchilov2016sgdr}, with 10\% warmup
    \item \texttt{gradient clipping}: norm (threshold=1)
    \item \texttt{full finetuning weight decay}: 0
    \item \texttt{full finetuning dropout}: 0
    \item \texttt{full finetuning learning rate grid}: \{1e-5,2e-5,5e-5,1e-4\} for both AdamW and Muon
    \item \texttt{LoRA dropout}: 0
    \item \texttt{LoRA alpha}: $2r$ for each $r=16,64,256$
    \item \texttt{LoRA learning rate grid}: \{5e-5,1e-4,2e-4,5e-4\}
\end{itemize}

\textbf{Additional hyperparameter settings of optimizers.} The following hyperparameters are fixed in the experiments.
\begin{itemize}
    \item AdamW:
    \begin{itemize}
        \item $(\beta_1,\beta_2)$: (0.9,0.95)
        \item $\epsilon$: 1e-8
    \end{itemize}
    \item Muon:
    \begin{itemize}
        \item momentum: 0.9
        \item AdamW $(\beta_1,\beta_2)$: (0.9,0.95)
        \item AdamW $\epsilon$: 1e-8
    \end{itemize}
    \item SGD and other optimizers (Algorithm~\ref{alg:cA_11},\ref{alg:cA_infinf}):
    \begin{itemize}
        \item momentum: 0.9
    \end{itemize}
\end{itemize}

\subsection{Experiment Settings in Section~\ref{sec:experiment_corrupted_real}}\label{appendix:experiment_setting_corrupted_real}

\textbf{Dataset \& Training.} 
We use corrupted datasets constructed from the OpenWebText dataset~\citep{gokaslan2019openWeb}, tokenized by the standard GPT-2 tokenizer. To obtain the target dataset, we first retrieve a subset of $40$M tokens from the original dataset. Then we split the tokens into blocks of length $1024$. For each block, we do a random permutation, i.e., randomly reorder the tokens, with probability $\alpha=0.5$. We call a block corrupted if it is randomly permuted. In this way, we keep the vocabulary frequency the same, but destroy the patterns in the corrupted blocks, ensuring that models can complete these blocks only by memorization. 
We use the same model configuration and optimizers as Section~\ref{sec:experiments_same_optimizer_forgets_less} for GPT-2, and set the sequence length to $512$, with an effective batch size of $256$. We conduct a learning rate grid search in \{3e-4, 5e-4, 1e-3, 3e-3\} and select the run with the best accuracy on clean data.

\textbf{Evaluation.}
To evaluate the memorization accuracy of models, we test the exact match accuracy of the model generation compared to the original text. Specifically, we randomly sample a subset from the original training dataset to form the evaluation set. By fixing an input length and generation length pair $(a,b)$, we test whether the model can generate exactly the same sequence as the original context of length $b$ with the input of length $a$ using greedy decoding and compute the accuracy.

\subsection{Additional Optimizers Implementation}
In Figure~\ref{fig:optimizer_activation_sparsity}, we involve two additional optimizers $\cA_{1,1}$ and $\cA_{\infty, \infty}$, which are derived from the framework~\eqref{eq:steepest_descent} with norm $\Norm{\cdot}_{1,1}$ and $\Norm{\cdot}_{\infty,\infty}$. For an arbitrary matrix $Z \in \RR^{m\times n}$, it holds that
\begin{align*}
    \Norm{Z}_{1,1} = \max_{i \in [m]} \sum_{j=1}^n \Abs{Z_{i,j}} \quad \text{and} \quad \Norm{Z}_{\infty,\infty} = \max_{j \in [n]} \sum_{i=1}^m \Abs{Z_{i,j}} .
\end{align*}
We then solve~\eqref{eq:steepest_descent} for $\Norm{\cdot}_{1,1}$. Notice that
\begin{align*}
    \dotprod{Z}{G} =& \sum_{i=1}^m \sum_{j=1}^n Z_{i,j} G_{i,j} \le \sum_{i=1}^m \sum_{j=1}^n \Abs{Z_{i,j}} \cdot \max_{j\in[n]} \Abs{G_{i,j}} \le 
    \left( \max_{i \in [m]} \sum_{j=1}^n \Abs{Z_{i,j}} \right) \left( \sum_{i=1}^m \max_{j\in[n]} \Abs{G_{i,j}} \right) \\
    =& \Norm{Z}_{1,1} \cdot \left( \sum_{i=1}^m \max_{j\in[n]} \Abs{G_{i,j}} \right) ,
\end{align*}
where $Z \triangleq W_{t+1} - W_t$. We note that the inequality holds as an equality when for each $i \in [m],j \in [n]$
\begin{align*}
    \text{$\cA_{1,1}$ Update:} \quad
    Z_{i,j} = \begin{cases}
        \frac{R}{k_i}, & \text{if } \Abs{G_{i,j}} = \max_{l \in [n]} \Abs{G_{i,l}}, \\
        0, & \text{else,}
    \end{cases}
\end{align*}
where $k_i$ denotes the number of entries $G_{i,j}$ such that $ \Abs{G_{i,j}} = \max_{l \in [n]} \Abs{G_{i,l}}$ for a fixed $i \in [m]$ and $R$ is the update scale. Similarly, we can also derive the algorithm under $\Norm{\cdot}_{\infty,\infty}$.
\begin{align*}
    \text{$\cA_{\infty,\infty}$ Update:} \quad
    Z_{i,j} = \begin{cases}
        \frac{R}{k_j}, & \text{if } \Abs{G_{i,j}} = \max_{l \in [m]} \Abs{G_{l,j}}, \\
        0, & \text{else,}
    \end{cases}
\end{align*}
where $k_j$ denotes the number of entries $G_{i,j}$ such that $\Abs{G_{i,j}} = \max_{l \in [m]} \Abs{G_{l,j}}$ for a fixed $j \in [n]$ and $R$ is the update scale. 
Intuitively, the update formula means that we only update the entries with the largest gradient scale in every row for $\cA_{1,1}$ and every column for $\cA_{\infty,\infty}$ in each iteration. 
Then, by incorporating momentum and the update RMS norm alignment~\citep{liu2025muon}, we formally write the two algorithms in Algorithm~\ref{alg:cA_11} and \ref{alg:cA_infinf}. We also include our version of SGD in Algorithm~\ref{alg:sgd}.

\begin{algorithm}[ht]
   \caption{Row Max Descent ($\cA_{1,1}$)}
   \label{alg:cA_11}
\begin{algorithmic}[1]
   \STATE {\bfseries Input:} $W_0\in \RR^{m\times n}$, the number of iterations $T$, schedule $\{\eta_t\}_{t=0}^{T-1}$, momentum $\beta$
   \STATE Initialize $M_{-1}=0 \in \RR^{m\times n}$
   \FOR{$t=0$ {\bfseries to} $T-1$}
   \STATE Obtain gradient $G_t$ at $W_t$
   \STATE $M_t = \beta M_{t-1} + (1 - \beta) G_t$
   \STATE Compute $U_t$ by $[U_t]_{i,j} = \begin{cases}
        \frac{R}{k_i}, & \text{if } \Abs{[M_t]_{i,j}} = \max_{l \in [n]} \Abs{[M_t]_{i,l}}, \\
        0, & \text{else.} \end{cases}$, where $k_i$ denotes the number of $M_t$ entries such that $\Abs{[M_t]_{i,j}} = \max_{l \in [m]} \Abs{[M_t]_{i,l}}$ for a fixed $j$
   \STATE
   $W_{t+1} = W_t - \eta_t \cdot 0.2 \cdot \frac{U_t}{\Norm{U_t}_{\mathrm{RMS}}}$
   \ENDFOR
\end{algorithmic}
\end{algorithm}

\begin{algorithm}[ht]
   \caption{Column Max Descent ($\cA_{\infty,\infty}$)}
   \label{alg:cA_infinf}
\begin{algorithmic}[1]
   \STATE {\bfseries Input:} $W_0\in \RR^{m\times n}$, the number of iterations $T$, schedule $\{\eta_t\}_{t=0}^{T-1}$, momentum $\beta$
   \STATE Initialize $M_{-1}=0 \in \RR^{m\times n}$
   \FOR{$t=0$ {\bfseries to} $T-1$}
   \STATE Obtain gradient $G_t$ at $W_t$
   \STATE $M_t = \beta M_{t-1} + (1 - \beta) G_t$
   \STATE Compute $U_t$ by $[U_t]_{i,j} = \begin{cases}
        \frac{R}{k_j}, & \text{if } \Abs{[M_t]_{i,j}} = \max_{l \in [m]} \Abs{[M_t]_{l,j}}, \\
        0, & \text{else.} \end{cases}$, where $k_j$ denotes the number of $M_t$ entries such that $\Abs{[M_t]_{i,j}} = \max_{l \in [m]} \Abs{[M_t]_{l,j}}$ for a fixed $j$
   \STATE
   $W_{t+1} = W_t - \eta_t \cdot 0.2 \cdot \frac{U_t}{\Norm{U_t}_{\mathrm{RMS}}}$
   \ENDFOR
\end{algorithmic}
\end{algorithm}

\begin{algorithm}[ht]
   \caption{SGD (with update RMS norm alignment)}
   \label{alg:sgd}
\begin{algorithmic}[1]
   \STATE {\bfseries Input:} $W_0\in \RR^{m\times n}$, the number of iterations $T$, schedule $\{\eta_t\}_{t=0}^{T-1}$, momentum $\beta$
   \STATE Initialize $M_{-1}=0 \in \RR^{m\times n}$
   \FOR{$t=0$ {\bfseries to} $T-1$}
   \STATE Obtain gradient $G_t$ at $W_t$
   \STATE $M_t = \beta M_{t-1} + (1 - \beta) G_t$
   \STATE $U_t = M_t$
   \STATE
   $W_{t+1} = W_t - \eta_t \cdot 0.2 \cdot \frac{U_t}{\Norm{U_t}_{\mathrm{RMS}}}$
   \ENDFOR
\end{algorithmic}
\end{algorithm}

\section{Additional Experiment Results}\label{appendix:additional_experiment_results}

\subsection{More Results for LoRA Comparing to Full Finetuning}\label{appendix:reproduction_lora_forgets_less}

In this section, we provide some additional experiment results specifically on LoRA under the same settings as Section~\ref{sec:experiments_same_optimizer_forgets_less}.

\paragraph{Reproduction of the Observations in \cite{biderman2024lora}.}
We can reproduce the observations in \citet{biderman2024lora} with our Llama results, as presented in Figure~\ref{fig:reproduction_lora_forgets_less_llama}, where we can observe in the left figure that LoRA learns less (lower math score) and forgets less (higher general score) if we all choose the learning rates leading to the best math performance. Here, LoRA rank 64 shows a greater forgetting compared to rank 256, mainly because rank 256 diverges for lr=5e-4, and a smaller learning rate leads to less forgetting. In the right figure of Figure~\ref{fig:reproduction_lora_forgets_less_llama}, we choose a smaller learning rate instead, and can find that the conclusions are different. With a smaller learning rate, full finetuning tends to learn more and forget less compared to LoRA.

\begin{figure}[th]
\centering

\begin{tabular}{ccc}
	\includegraphics[width=0.45\linewidth]{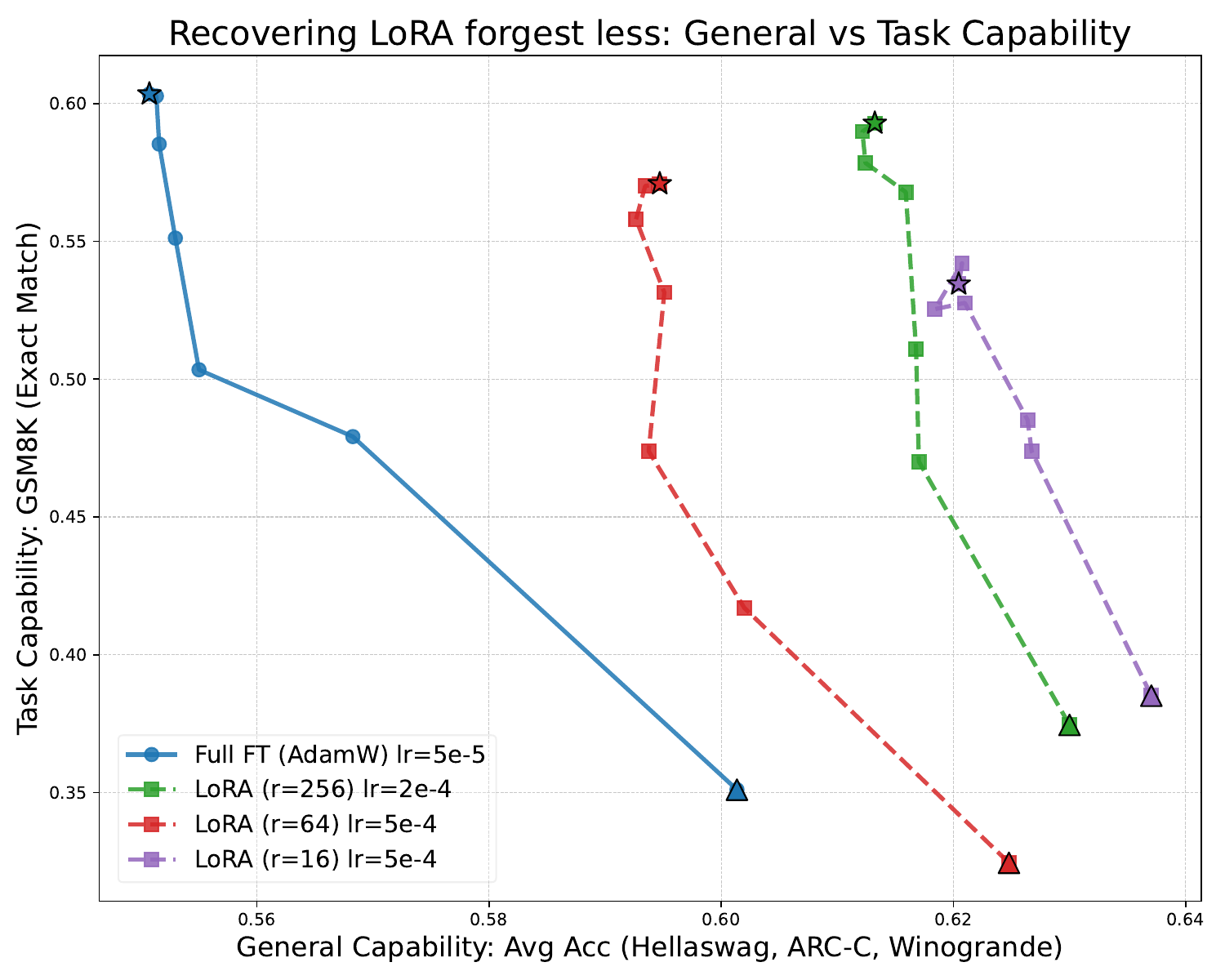}
	\!\!\!
	& \includegraphics[width=0.45\linewidth]{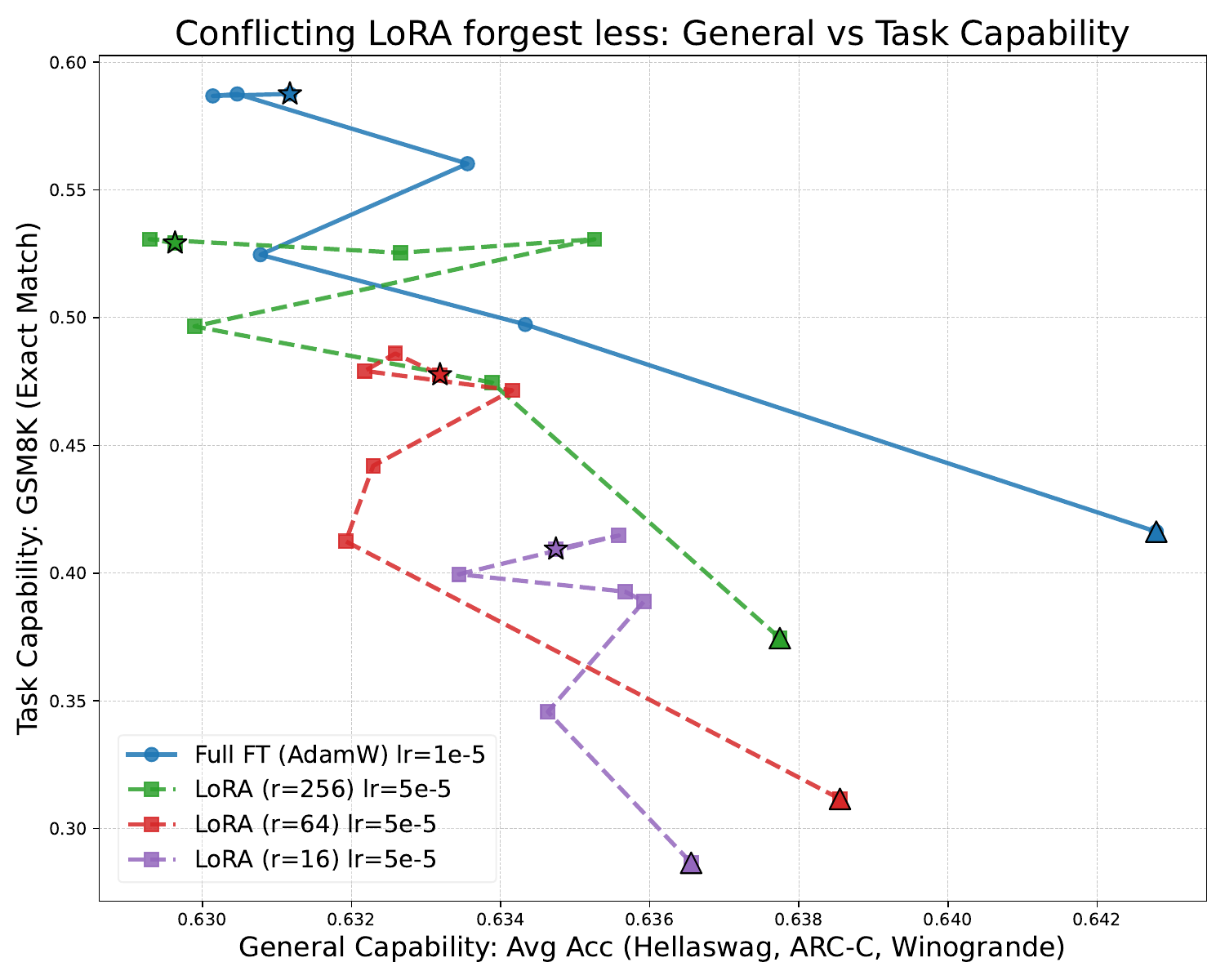} 
\end{tabular} 

\vskip-0.2cm
\caption{Reproduction of the results in \citet{biderman2024lora}. In the left figure, we choose only the learning rate leading to the best math performance for each training algorithm instead of presenting a Pareto frontier. In the right figure, we choose some smaller learning rates for comparison. }
\label{fig:reproduction_lora_forgets_less_llama} \vskip-0.15cm
\end{figure}

\paragraph{Learning Rate is Important in Learning-Forgetting Tradeoff.}
In Figure~\ref{fig:learning_rate_affects_forgetting_llama}, we show how learning rate affects the learning-forgetting tradeoff. Generally, a larger learning rate achieves a relatively higher learning score while leading to much more significant forgetting, which also aligns with the observations in \citet{rofin2026learning}, a concurrent work to us. Therefore, it is important to choose an appropriate learning rate in SFT based on the requirements of the final model.

\begin{figure}[th]
\centering
\begin{tabular}{cccc}
	\includegraphics[width=0.22\linewidth]{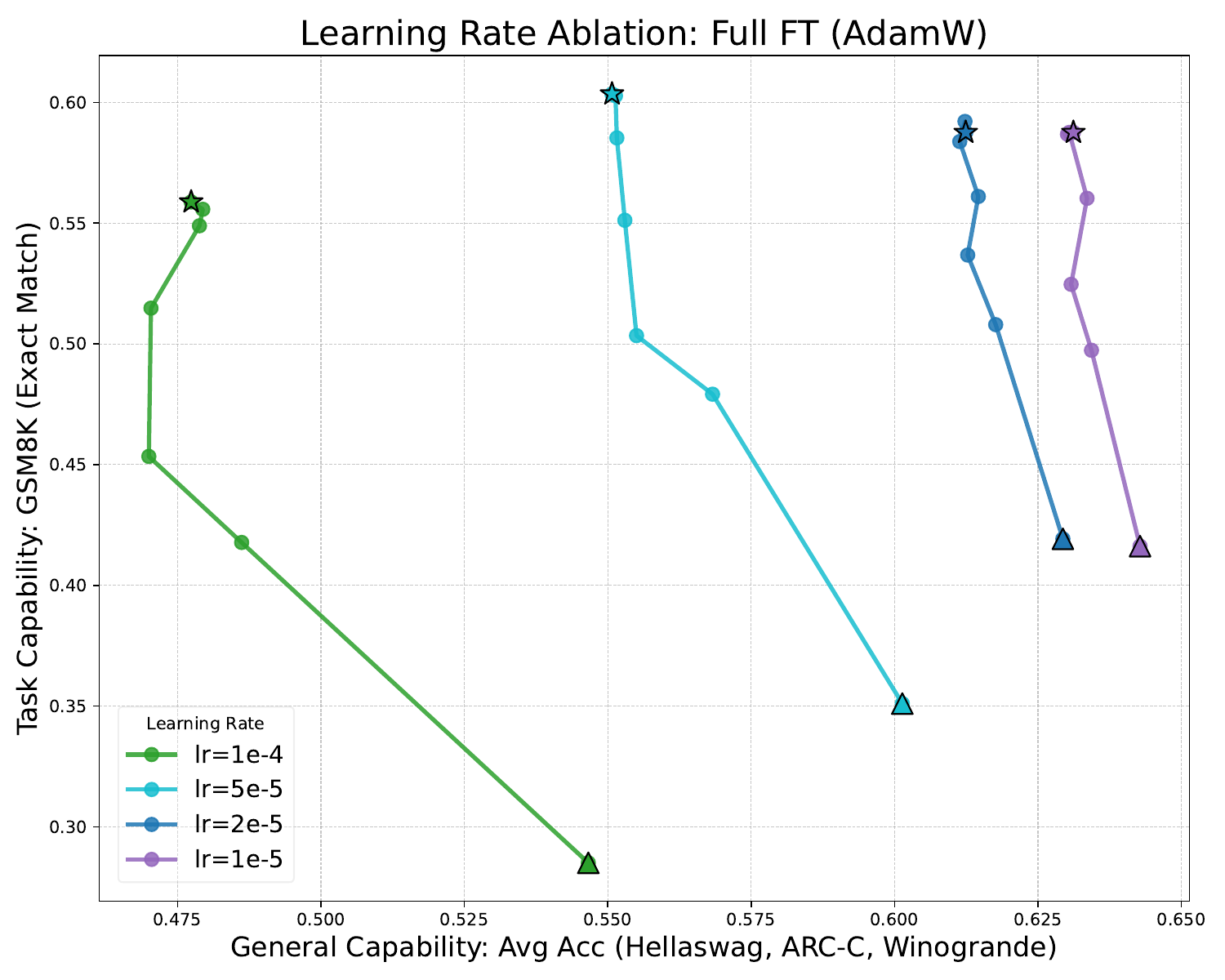}
	\!\!\!
	& \includegraphics[width=0.22\linewidth]{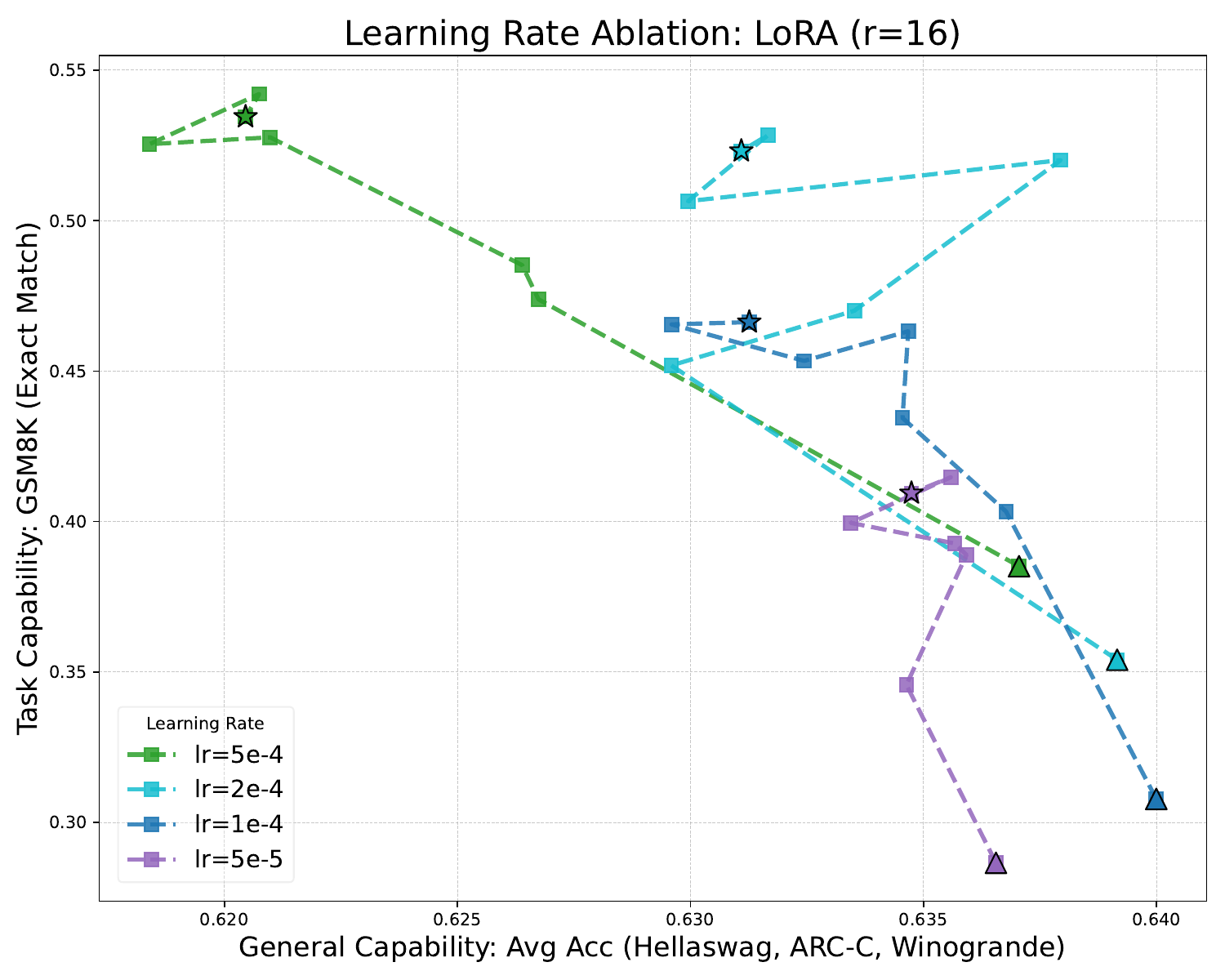} 
    \!\!\!
	& \includegraphics[width=0.22\linewidth]{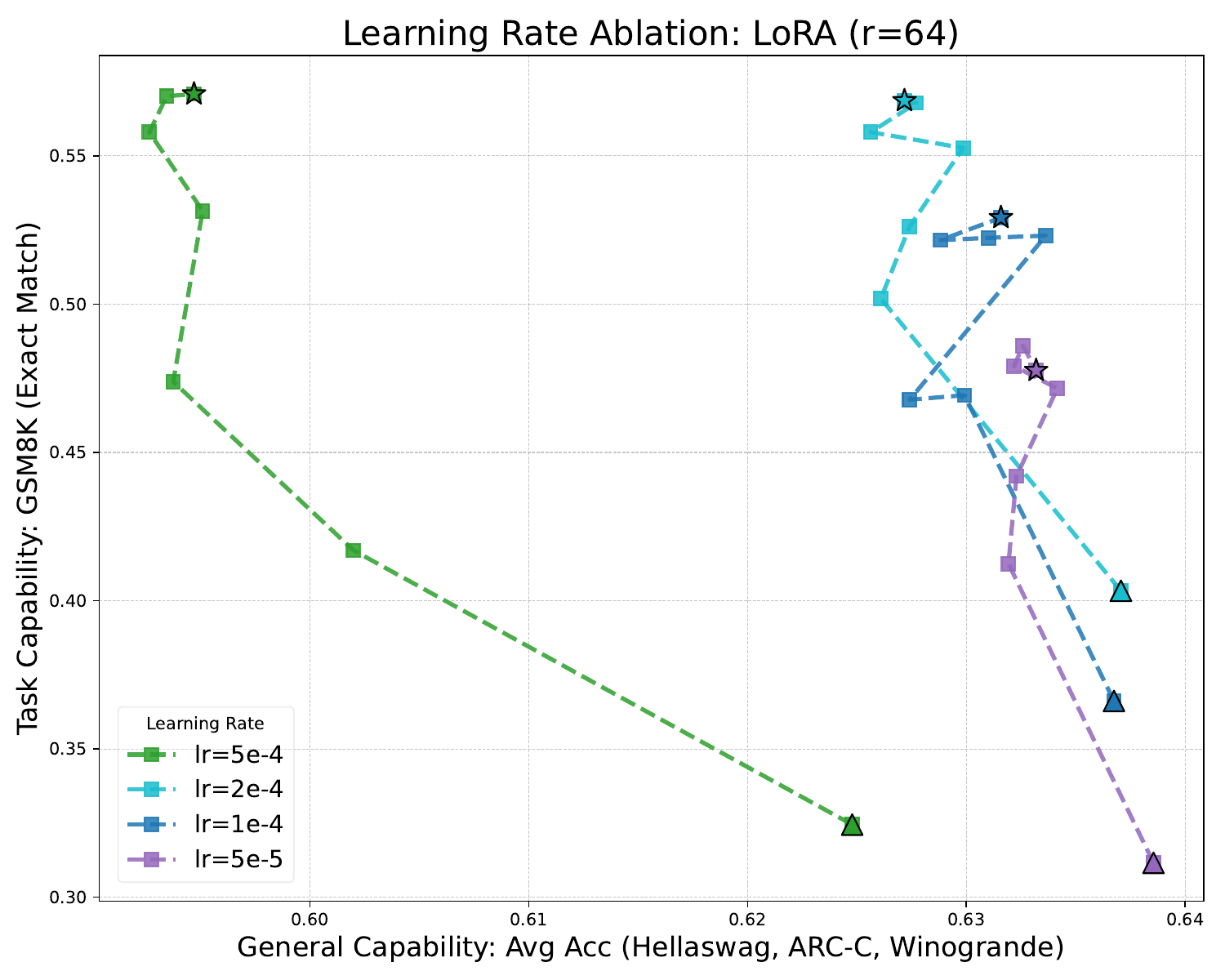}
    \!\!\!
	& \includegraphics[width=0.22\linewidth]{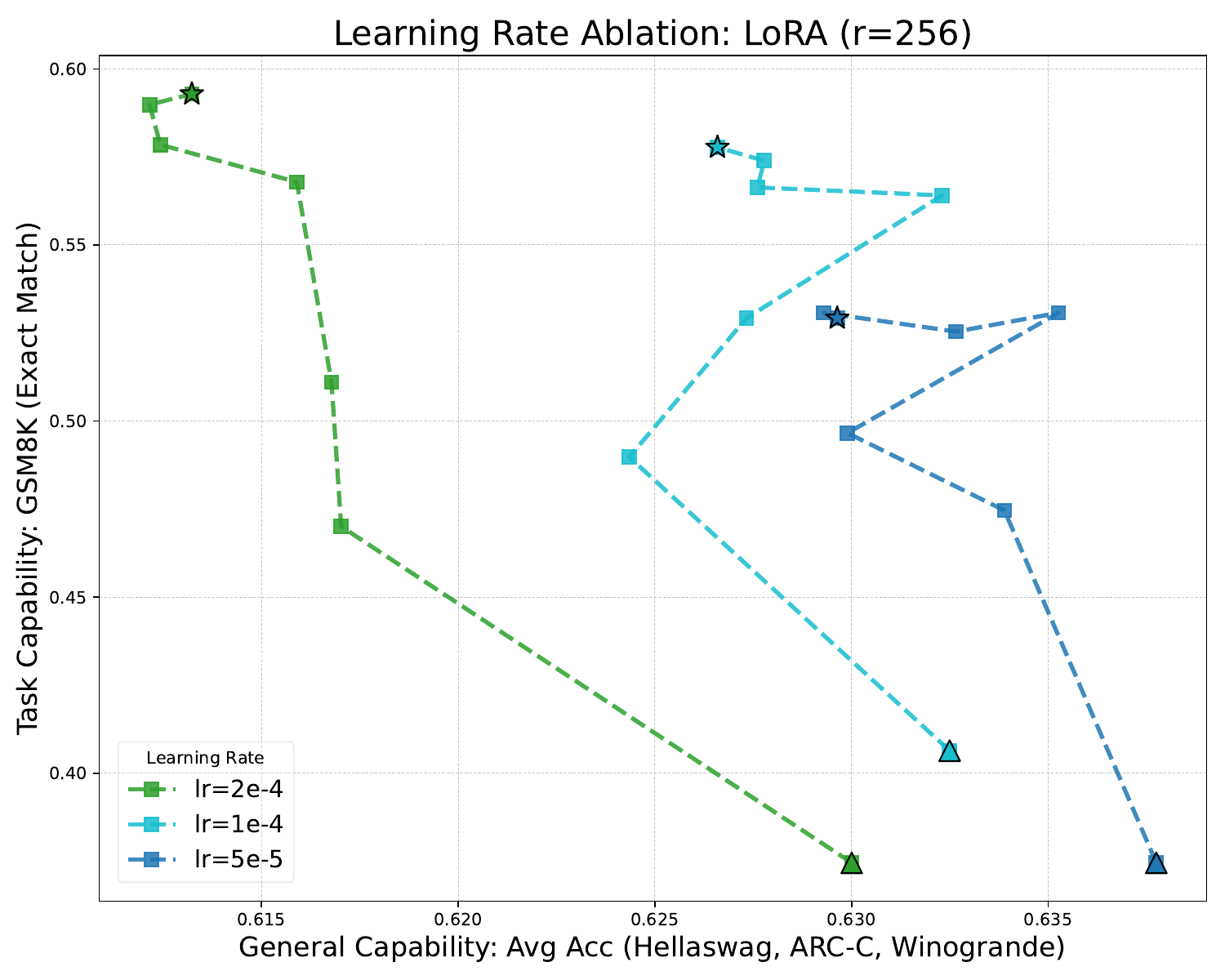} 
\end{tabular} 

\vskip-0.2cm
\caption{Learning rate affects forgetting. Note that learning rate lr=5e-4 failed for LoRA with $r=256$.}
\label{fig:learning_rate_affects_forgetting_llama} \vskip-0.15cm
\end{figure}

\subsection{More Detailed Activation Plots} 
In Figures~\ref{fig:optimizer_activation_sparsity_module_wise} and \ref{fig:optimizer_activation_sparsity_activation_split}, we present the average activation sparsity of detailed modules and specific activation splits. We can observe that the detailed sparsity plots also generally follow our observations that optimizers regularize the activations based on the corresponding matrix-induced norm $\alpha,\beta$. Interestingly, in terms of sparsity for $\cA_{1,\infty}$ (including AdamW and SignSGD), the input activations of the next layer can be quite different from the output activations of the last layer, which should connect. This can be related to the effect of the activation function (inside FFN), attention logits computation (inside attention modules), and layernorms and residual connections (between modules and layers). It would be an interesting future work to further explore the mechanism behind this phenomenon and how the activations change during training.

\begin{figure}[t]
\centering
\begin{tabular}{cccc}
	\includegraphics[width=0.225\linewidth]{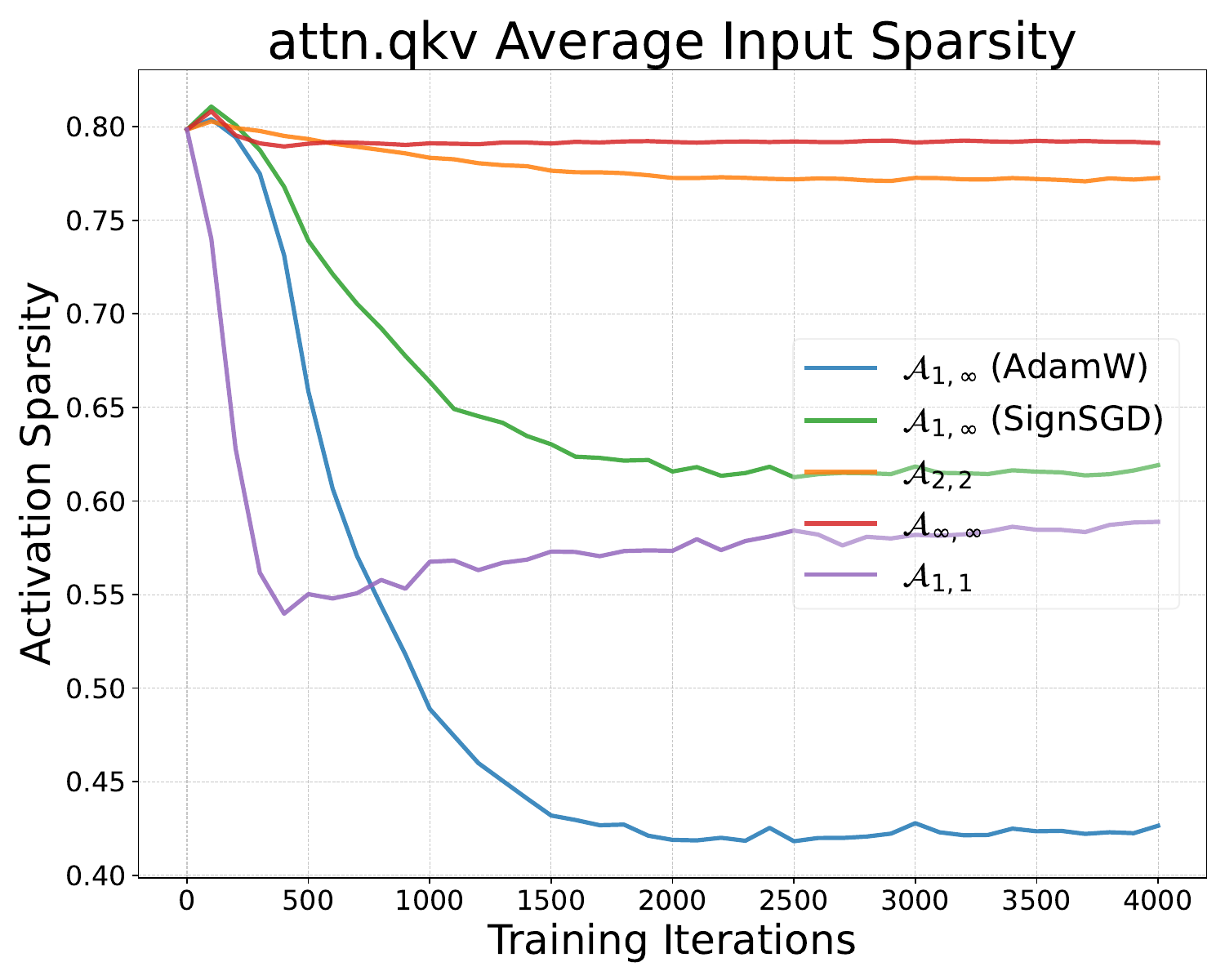}
	\!\!\!
	& \includegraphics[width=0.225\linewidth]{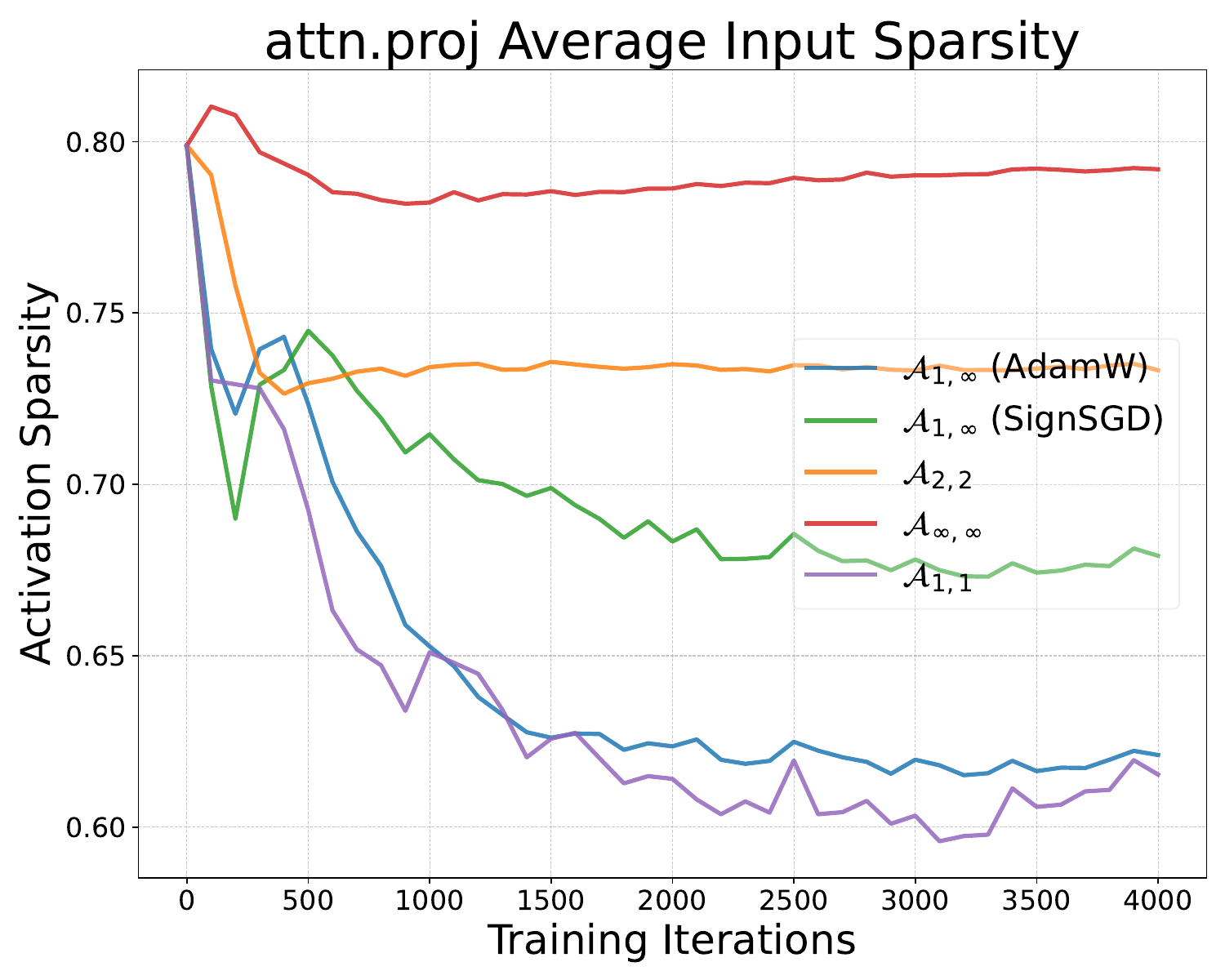} 
    \!\!\!
	& \includegraphics[width=0.225\linewidth]{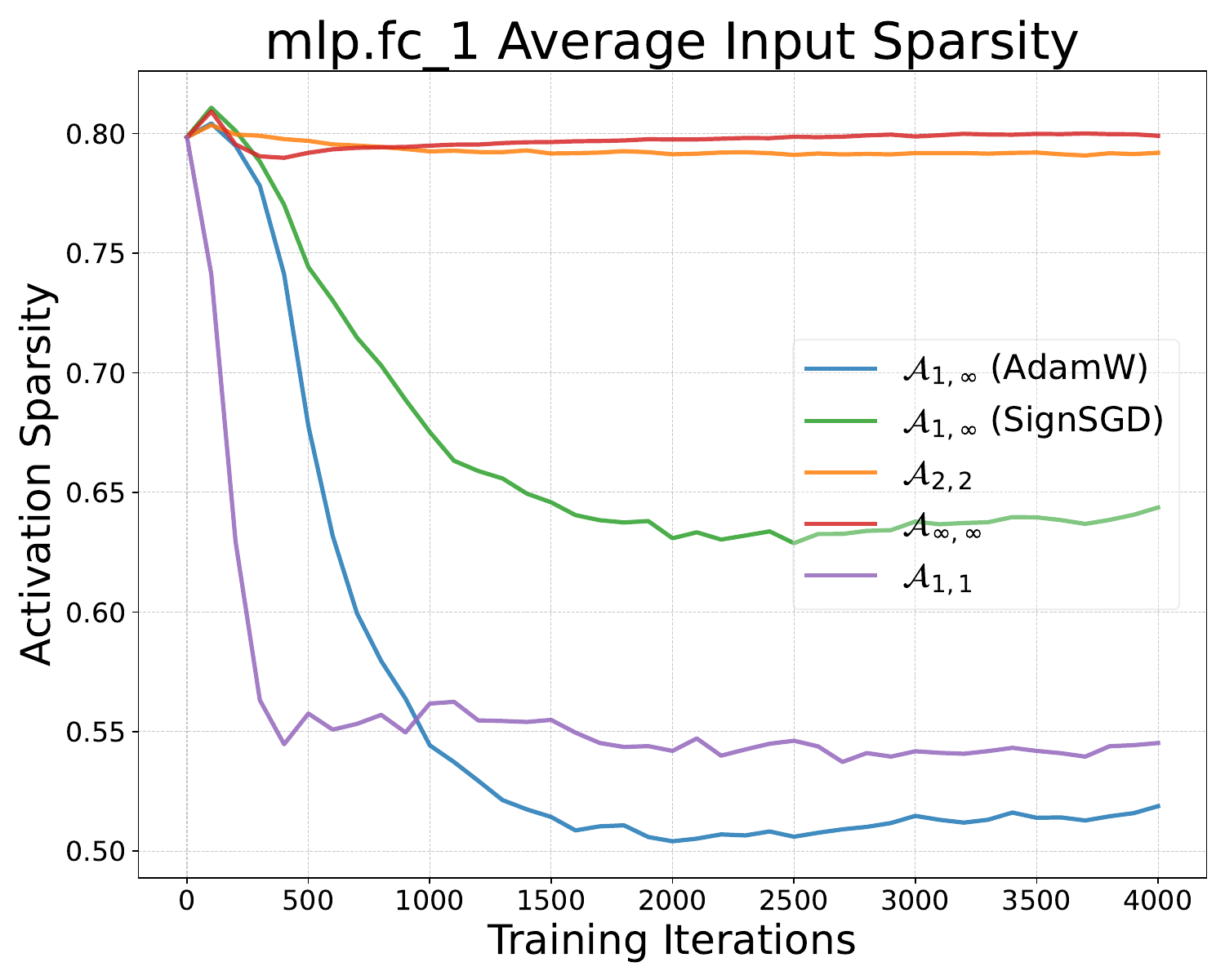}
    \!\!\!
	& \includegraphics[width=0.225\linewidth]{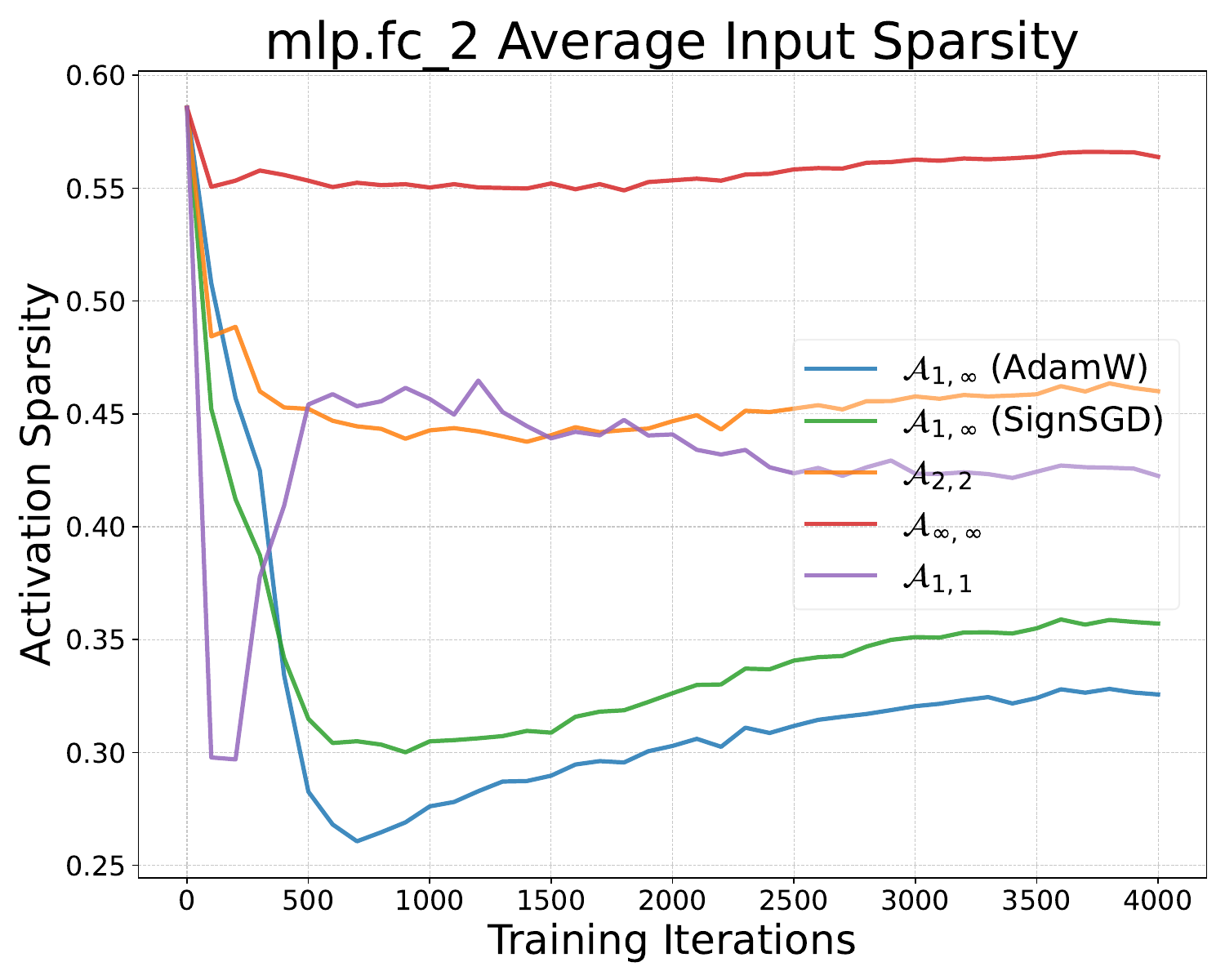} \\
    \includegraphics[width=0.225\linewidth]{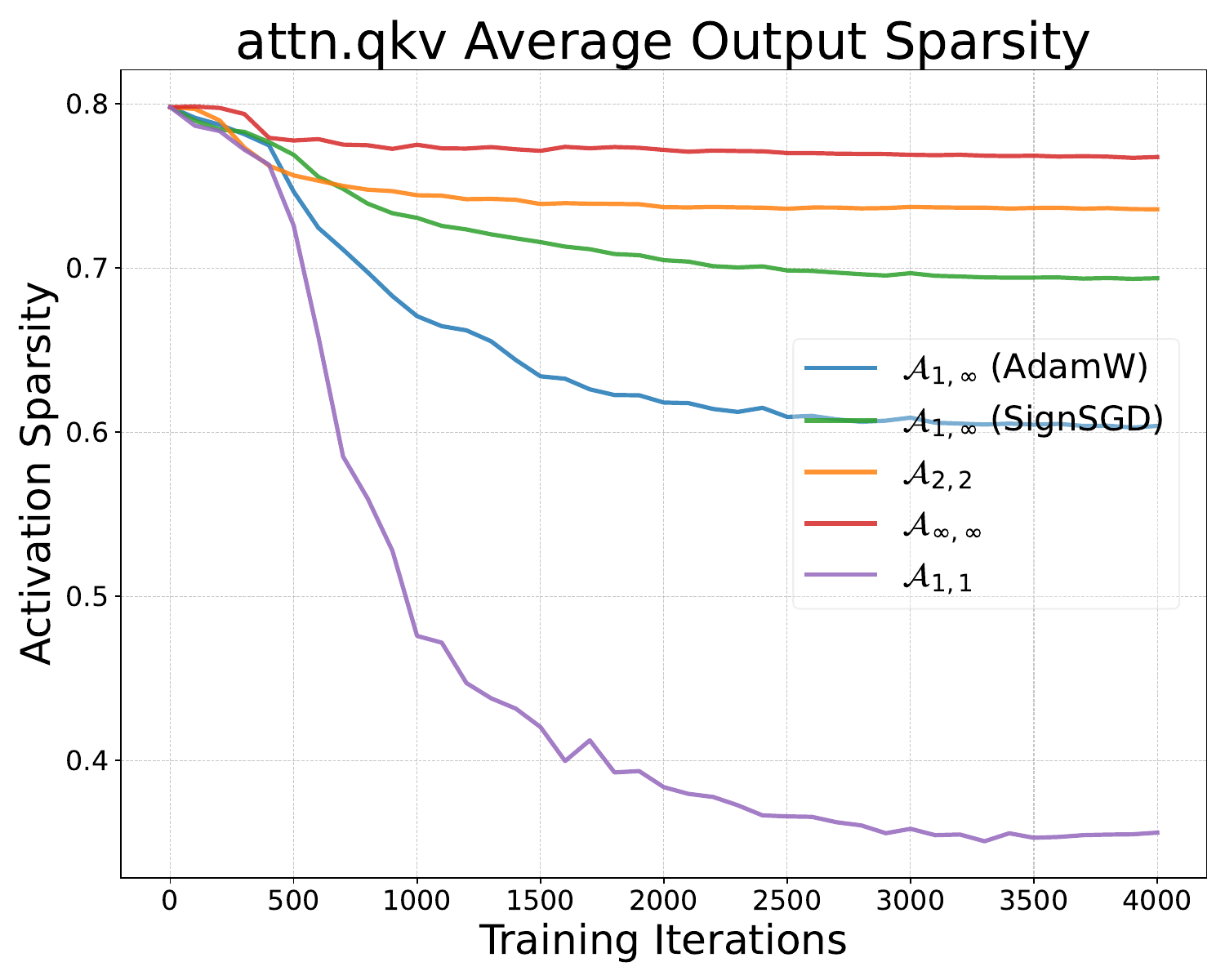}
	\!\!\!
	& \includegraphics[width=0.225\linewidth]{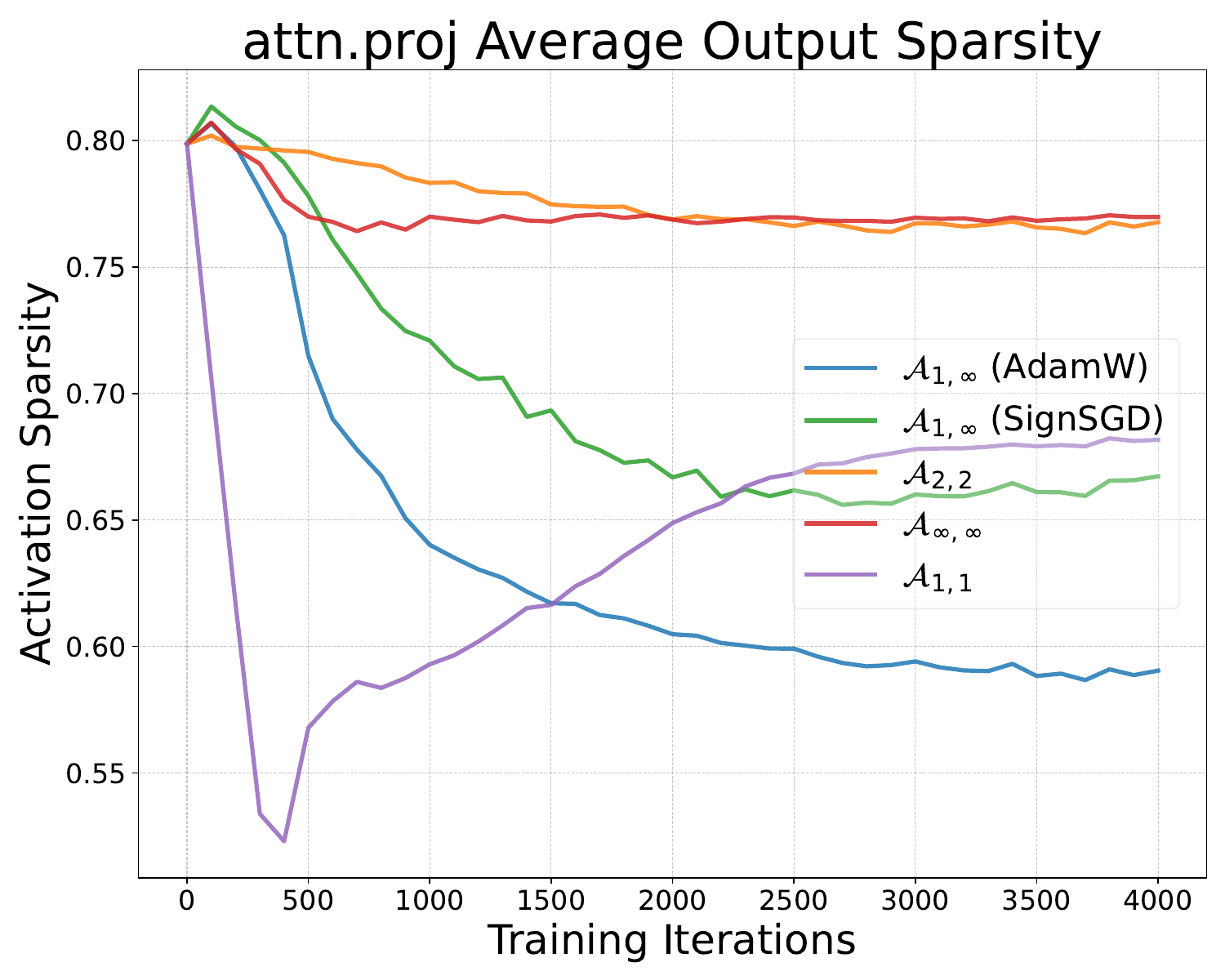} 
    \!\!\!
	& \includegraphics[width=0.225\linewidth]{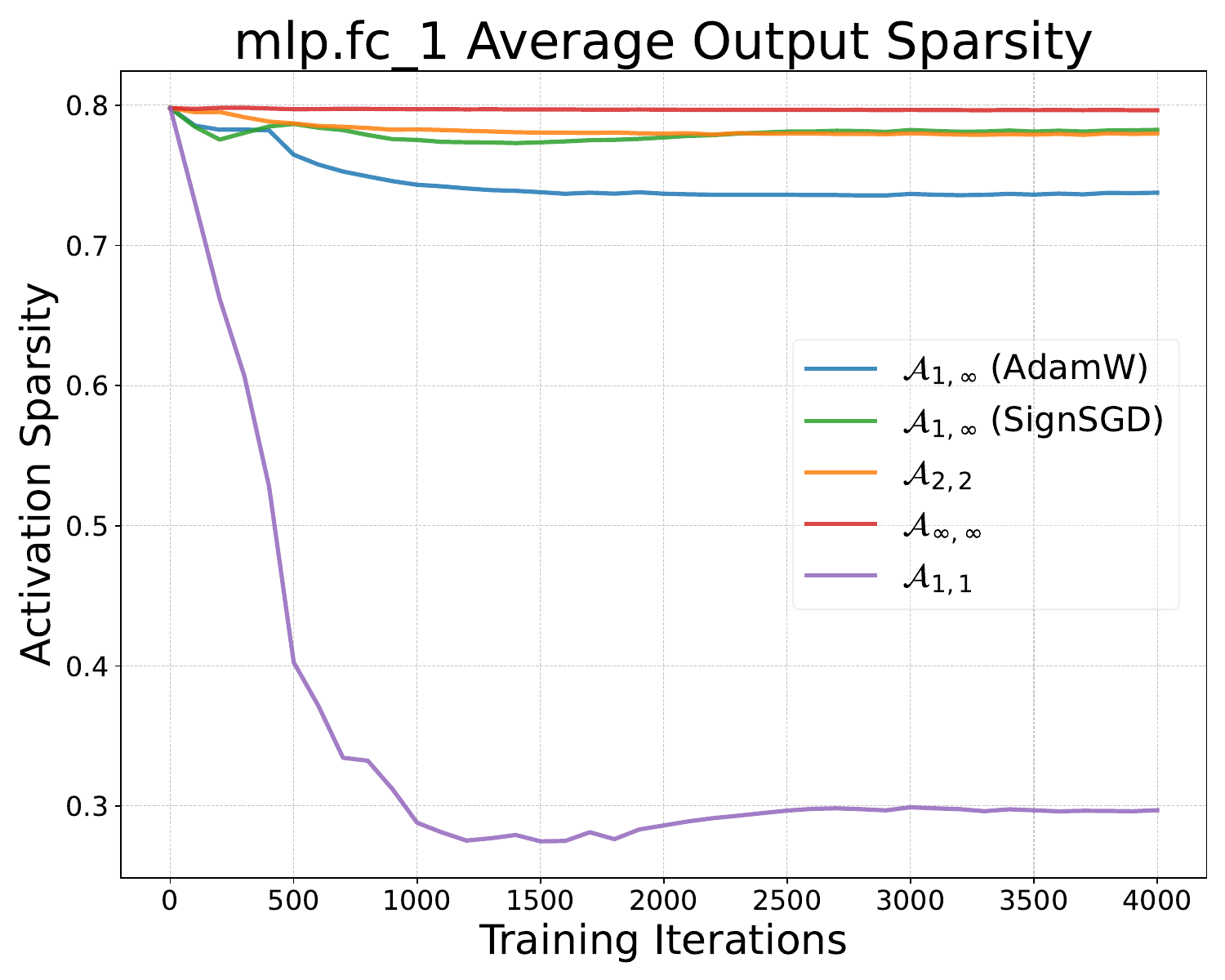}
    \!\!\!
	& \includegraphics[width=0.225\linewidth]{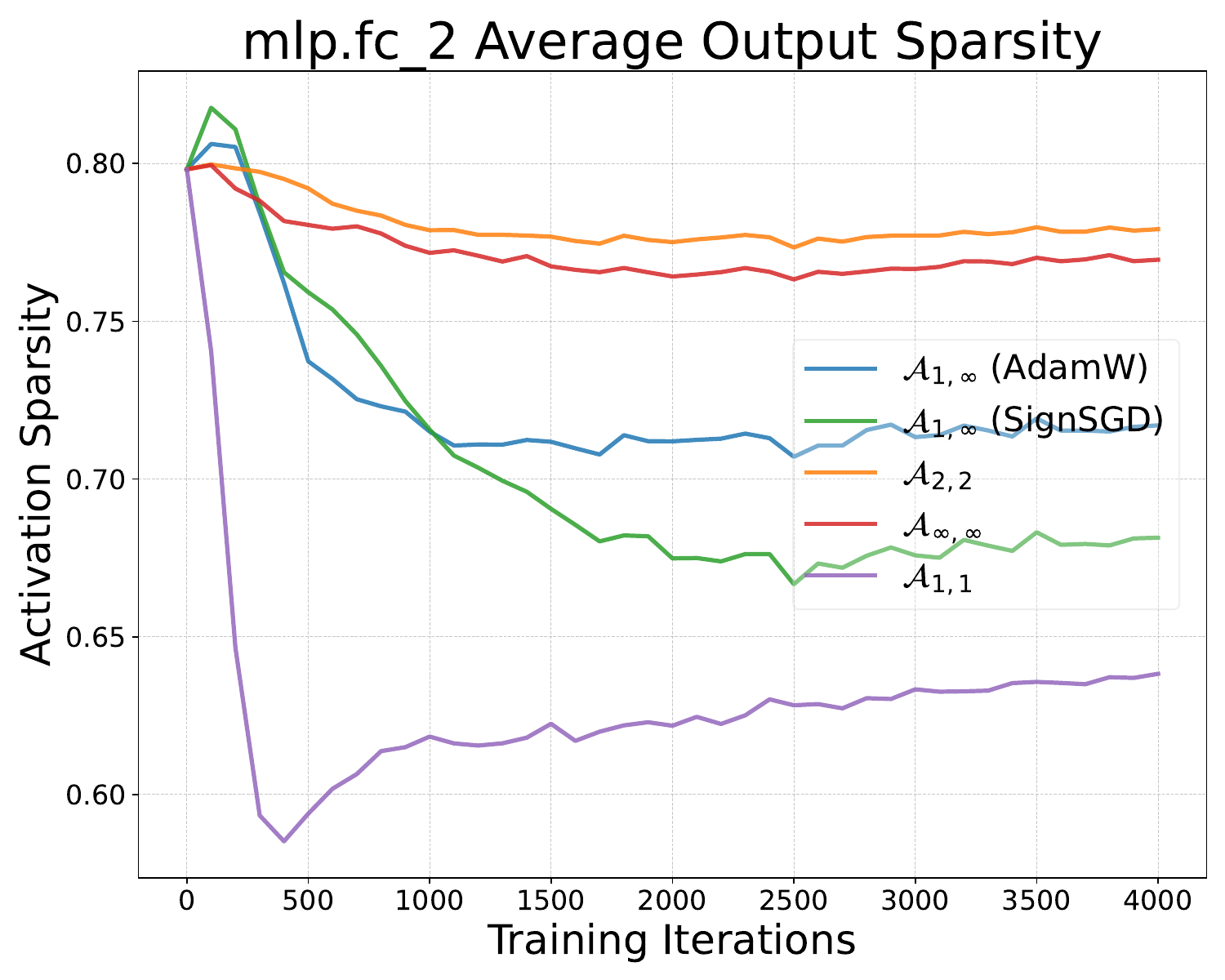}
\end{tabular} 

\vskip-0.2cm
\caption{The average activation sparsity of each linear layer under the same settings of Figure~\ref{fig:optimizer_activation_sparsity}. 
We consider 2 linear layers for attention modules: the QKV projection (attn.qkv) and attention output projection (attn.proj); 2 linear layers for FFN modules: the up projection (mlp.fc\_1) and down projection (mlp.fc\_2).}
\label{fig:optimizer_activation_sparsity_module_wise} \vskip-0.15cm
\end{figure}

\begin{figure}[ht]
\centering
\begin{tabular}{cc}
	\includegraphics[width=0.4\linewidth]{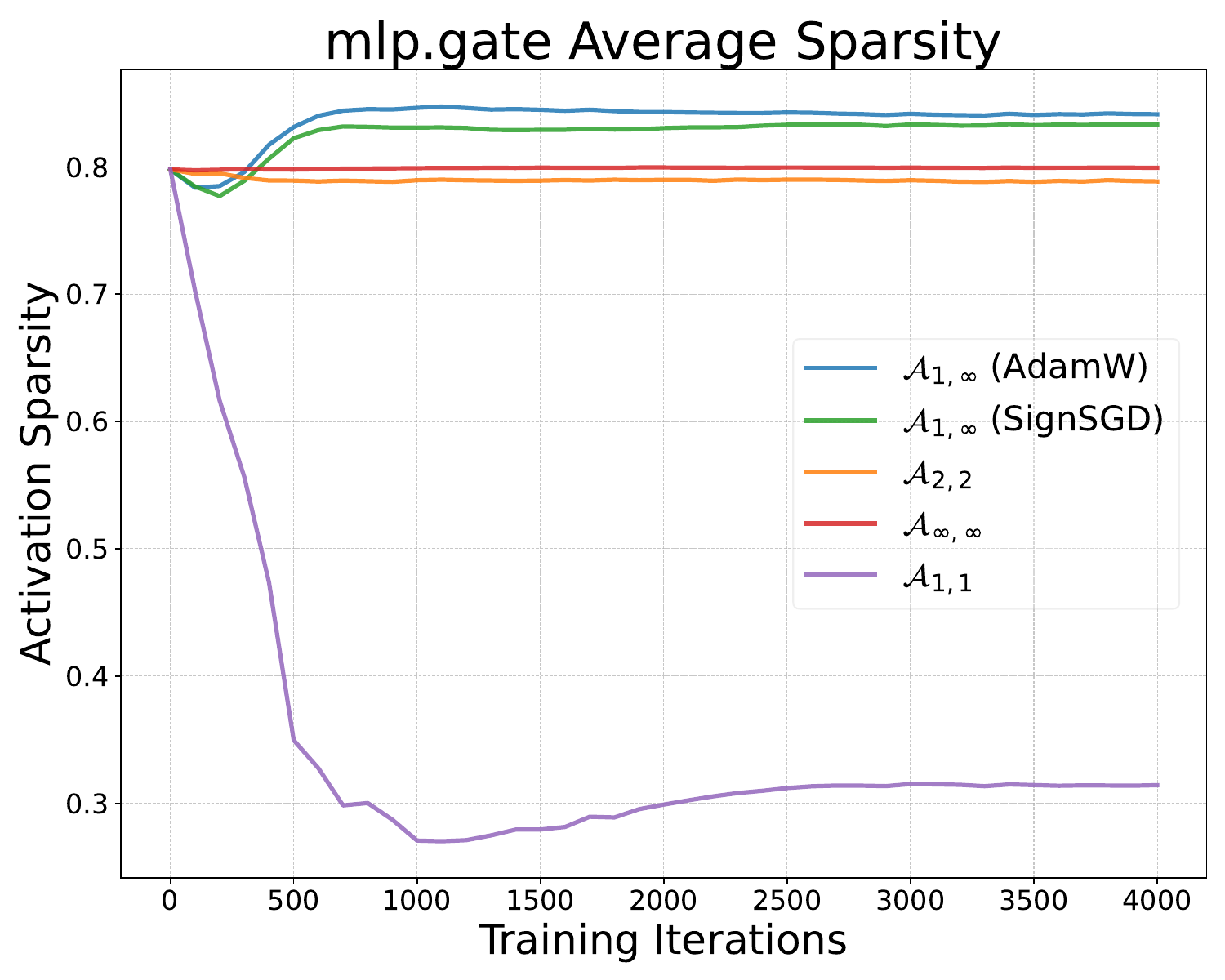}
	\!\!\!
	& \includegraphics[width=0.4\linewidth]{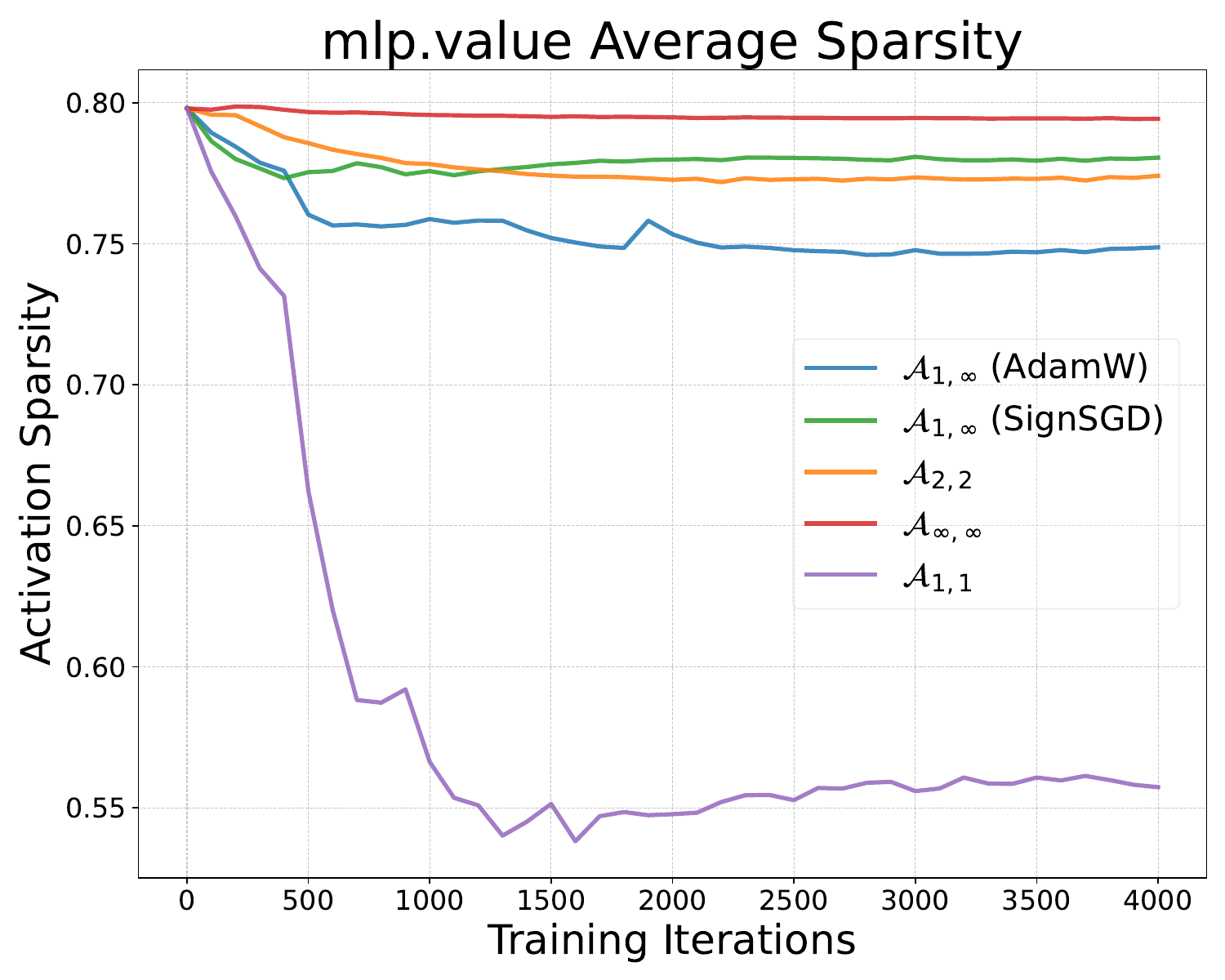} 
\end{tabular} 

\begin{tabular}{ccc}
	\includegraphics[width=0.3\linewidth]{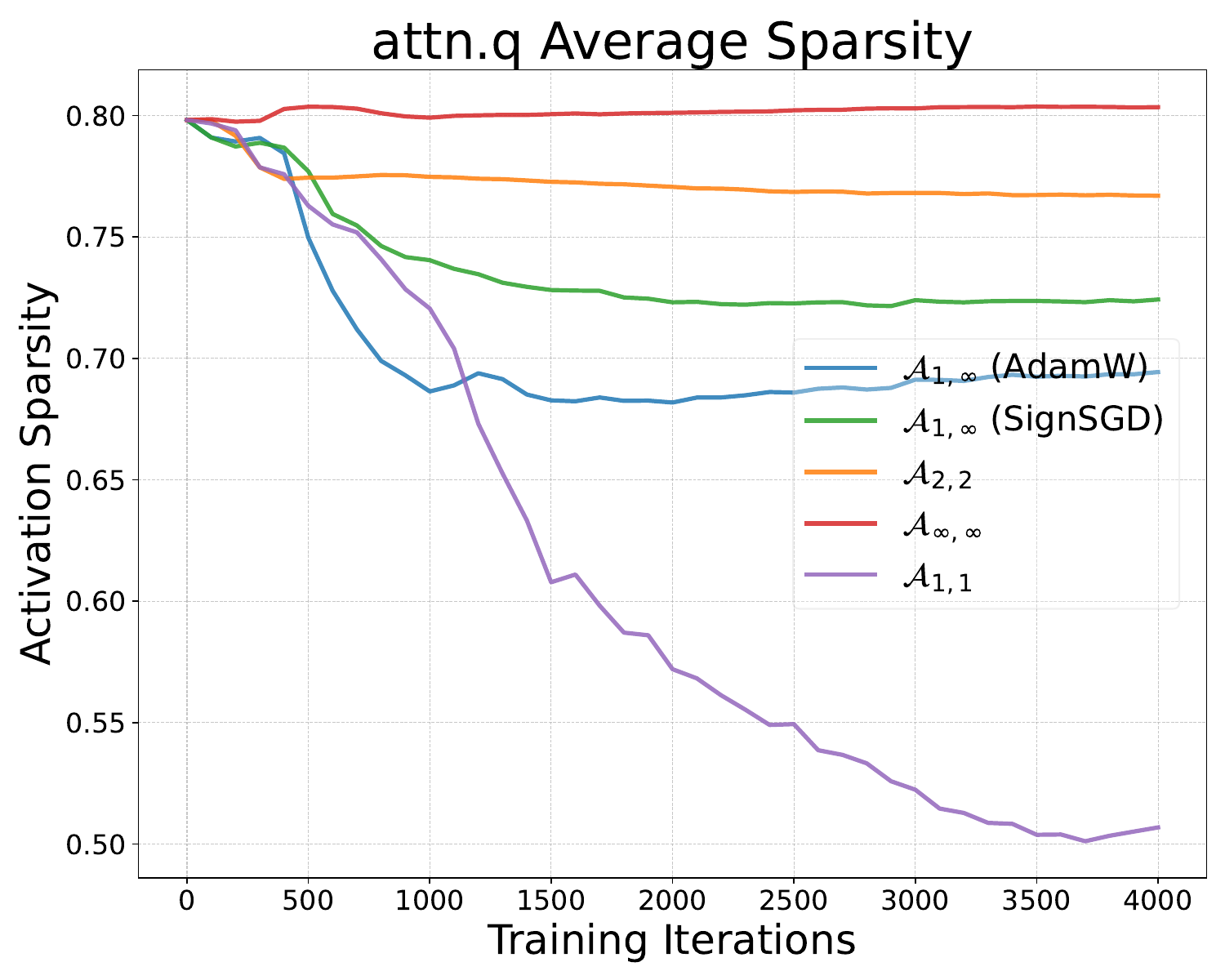}
	\!\!\!
	& \includegraphics[width=0.3\linewidth]{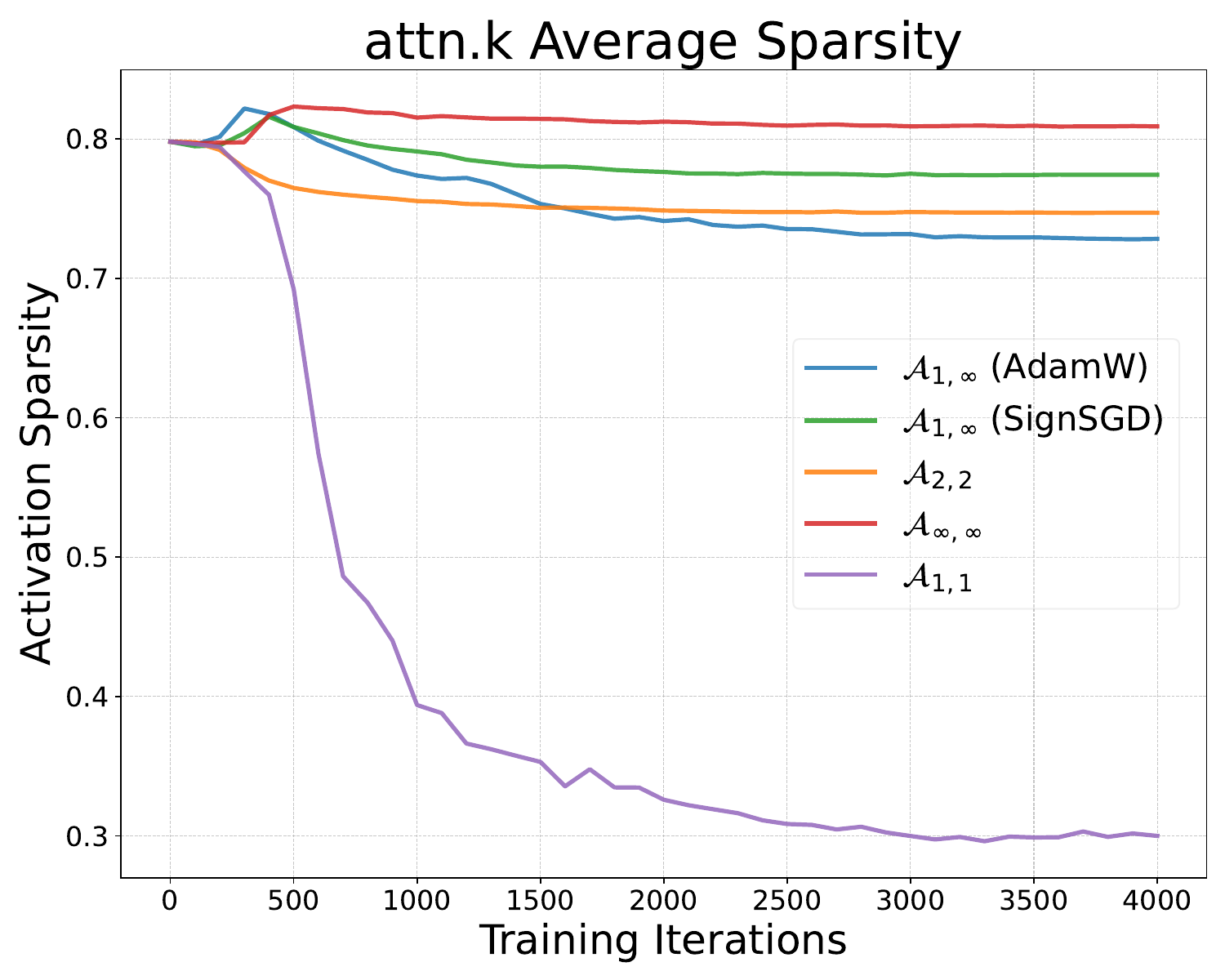} 
    \!\!\!
	& \includegraphics[width=0.3\linewidth]{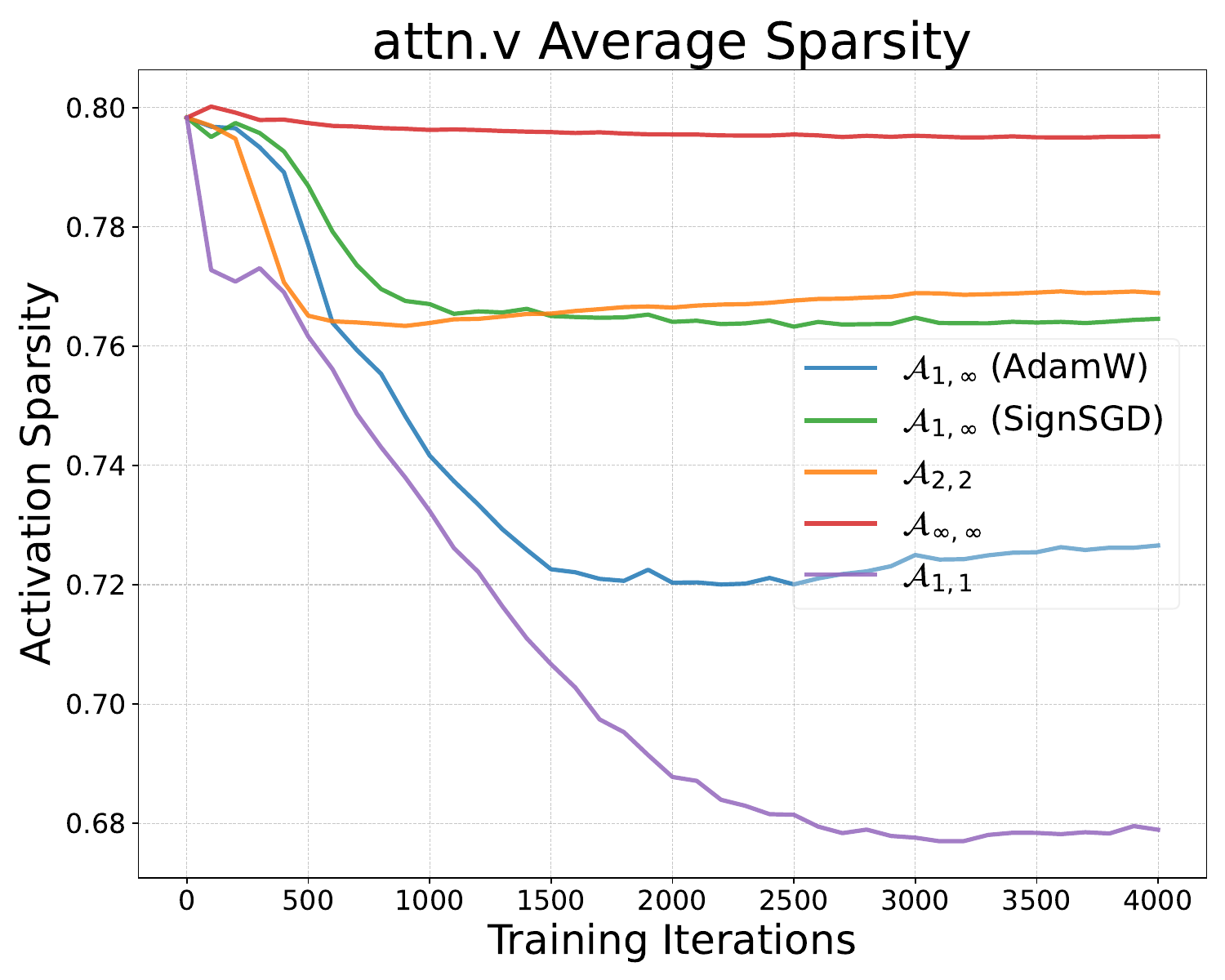} 
\end{tabular} 

\vskip-0.2cm
\caption{The average activation sparsity of some split activations under the same settings of Figure~\ref{fig:optimizer_activation_sparsity} and \ref{fig:optimizer_activation_sparsity_module_wise}, including q,k,v in attention modules and gate,value in FFN modules introduced by the SwiGLU activation function.}
\label{fig:optimizer_activation_sparsity_activation_split} \vskip-0.15cm
\end{figure}

\section{Proof of Section~\ref{sec:understanding_phenomenon}}

\subsection{Comparison between Norms}
We begin with some auxiliary lemmas for comparing different norms. Lemma~\ref{lem:comparison_vector_norms} is a standard result for vector norms, and Lemma~\ref{lem:comparing_matrix_norms} is a simple extension of it.

\begin{lemma}\label{lem:comparison_vector_norms}
    For an arbitrary vector $z \in \RR^d$ and $\alpha_1,\alpha_2 \in [1,\infty]$ such that $\alpha_1 \le \alpha_2$, it holds that
    \begin{align*}
        \Norm{z}_{\alpha_2} \le \Norm{z}_{\alpha_1} \le d^{\frac{1}{\alpha_1} - \frac{1}{\alpha_2}}\Norm{z}_{\alpha_2} .
    \end{align*}
\end{lemma}

\begin{lemma}\label{lem:comparing_matrix_norms}
    For an arbitrary matrix $A \in \RR^{m\times n}$ and $\alpha_1,\alpha_2 \in [1,\infty]$ such that $\alpha_1 \le \alpha_2$ and $\beta_1,\beta_2 \in [1,\infty]$ such that $\beta_1 \le \beta_2$, it holds that
    \begin{align*}
        \Norm{A}_{\alpha_1,\beta} \le \Norm{A}_{\alpha_2,\beta} \le n^{\frac{1}{\alpha_1} - \frac{1}{\alpha_2}} \Norm{A}_{\alpha_1,\beta} 
        \quad \mathrm{and} \quad
        \Norm{A}_{\alpha,\beta_2} \le \Norm{A}_{\alpha,\beta_1} \le m^{\frac{1}{\beta_1} - \frac{1}{\beta_2}} \Norm{A}_{\alpha,\beta_2} 
    \end{align*}
\end{lemma}
\begin{proof}
    Let us denote 
    \begin{align*}
        z_1 \in \argmax{z\neq 0} \frac{\Norm{A z}_\beta}{\Norm{z}_{\alpha_1}} \quad \mathrm{and} \quad z_1 \in \argmax{z\neq 0} \frac{\Norm{A z}_\beta}{\Norm{z}_{\alpha_2}} ,
    \end{align*}
    which implies that
    \begin{align*}
        \Norm{A z_1}_\beta = \Norm{A}_{\alpha_1, \beta} \Norm{z_1}_{\alpha_1} \quad \mathrm{and} \quad \Norm{A z_2}_\beta = \Norm{A}_{\alpha_2, \beta} \Norm{z_2}_{\alpha_2} .
    \end{align*}
    Combining with Lemma~\ref{lem:comparison_vector_norms}, we therefore have
    \begin{align*}
        \Norm{A}_{\alpha_1, \beta} = \frac{\Norm{A z_1}_\beta}{\Norm{z_1}_{\alpha_1}} \le \frac{\Norm{A z_1}_\beta}{\Norm{z_1}_{\alpha_2}} \le \frac{\Norm{A z_2}_\beta}{\Norm{z_2}_{\alpha_2}} = \Norm{A}_{\alpha_2, \beta} 
    \end{align*}
    and
    \begin{align*}
        \Norm{A}_{\alpha_2, \beta} = \frac{\Norm{A z_2}_\beta}{\Norm{z_2}_{\alpha_2}} \le n^{\frac{1}{\alpha_1} - \frac{1}{\alpha_2}} \frac{\Norm{A z_2}_\beta}{\Norm{z_2}_{\alpha_1}} \le n^{\frac{1}{\alpha_1} - \frac{1}{\alpha_2}} \frac{\Norm{A z_1}_\beta}{\Norm{z_1}_{\alpha_1}} = n^{\frac{1}{\alpha_1} - \frac{1}{\alpha_2}} \Norm{A}_{\alpha_1, \beta} ,
    \end{align*}
    which concludes the proof for the first inequality. We can prove the second inequality similarly.
\end{proof}

\subsection{Proof of Theorem~\ref{thm:same_opt_forgets_less}}
\begin{proof}
    Consider an arbitrary SFT optimizer choice $\cA_{\alpha_2, \beta_2}$ with $\alpha_2,\beta_2 \in [1,\infty)$ and $\Delta W$ is the corresponding SFT update. 
    From Assumption~\ref{asm:weight_activation_alignment}, for a fixed $\beta_* \in [1,\infty)$, there exists $\alpha_* \in [1,\infty)$ such that
    \begin{align*}
        \EE\left[ \Norm{\Delta W x}_{\beta_*}^2 \right] = \Theta \left( \Norm{\Delta W}_{\alpha_*,\beta_*}^2 \EE\left[ \Norm{x}_{\alpha_*}^2 \right] \right) .
    \end{align*}
    Since for all $\alpha,\beta \in [1,\infty]$, it always holds that
    \begin{align*}
        \Norm{\Delta W x}_{\beta} \le \Norm{\Delta W}_{\alpha,\beta} \Norm{x}_\alpha 
    \end{align*}
    based on the definition of matrix induced norms, $\alpha_*,\beta_*$ should correspond to the tightest upper bound to yield Assumption~\ref{asm:weight_activation_alignment}.
    We then discuss the best value of $\alpha$, i.e., $\alpha_*$, in this case. 
    \begin{itemize}
        \item If $\alpha \in (\alpha_1,\infty]$, based on Lemma~\ref{lem:comparing_matrix_norms} and Assumption~\ref{asm:optimizer_choice_activation_property}, it holds that
        \begin{align*}
            \Norm{\Delta W}_{\alpha_1,\beta_*} \le \Norm{\Delta W}_{\alpha,\beta_*} \quad \mathrm{and} \quad \EE\left[ \Norm{x}_{\alpha_1}^2 \right] = \Theta\left( \EE\left[ \Norm{x}_{\alpha}^2 \right] \right).
        \end{align*}
        Thus, we have
        \begin{align*}
            \Norm{\Delta W}_{\alpha_1,\beta_*}^2 \EE\left[ \Norm{x}_{\alpha_1}^2 \right] = \cO\left( \Norm{\Delta W}_{\alpha,\beta_*}^2 \EE\left[ \Norm{x}_{\alpha}^2 \right] \right)
        \end{align*}
        
        \item If $\alpha \in [1,\alpha_1)$, based on Lemma~\ref{lem:comparing_matrix_norms} and Assumption~\ref{asm:optimizer_choice_activation_property}, it holds that
        \begin{align*}
            \Norm{\Delta W}_{\alpha_1,\beta_*} \le n^{\frac{1}{\alpha} - \frac{1}{\alpha_1}} \Norm{\Delta W}_{\alpha,\beta_*} \quad \mathrm{and} \quad \EE\left[ \Norm{x}_{\alpha_1}^2 \right] = \Theta\left( n^{\frac{2}{\alpha_1} - \frac{2}{\alpha}} \EE\left[ \Norm{x}_{\alpha}^2 \right] \right).
        \end{align*}
        Thus, we have
        \begin{align*}
            \Norm{\Delta W}_{\alpha_1,\beta_*}^2 \EE\left[ \Norm{x}_{\alpha_1}^2 \right] = \cO\left( \Norm{\Delta W}_{\alpha,\beta_*}^2 \EE\left[ \Norm{x}_{\alpha}^2 \right] \right)
        \end{align*}
    \end{itemize}
    From the two cases, we conclude that we should have $\alpha_*=\alpha_1$ to yield the tightest upper bound since it holds that $\Norm{\Delta W}_{\alpha_1,\beta_*}^2 \EE\left[ \Norm{x}_{\alpha_1}^2 \right] = \cO\left( \Norm{\Delta W}_{\alpha,\beta_*}^2 \EE\left[ \Norm{x}_{\alpha}^2 \right] \right)$ for all $\alpha \in [1,\infty]$.

    Moreover, based on Assumption~\ref{asm:optimizer_choice_activation_property} on $\Delta W x$, we have
    \begin{align*}
        \EE\left[ \Norm{\Delta W x}_{2}^2 \right] 
        &= \Theta \left( m^{-\max\{ \frac{2}{\beta_2} - 1,0 \} } \EE\left[ \Norm{\Delta W x}_{\beta_2}^2 \right] \right) \\
        &= \Theta \left( m^{-\max\{ \frac{2}{\beta_2} - 1,0 \} + \max\{ \frac{2}{\beta_2} - \frac{2}{\beta_*},0 \} } \EE\left[ \Norm{\Delta W x}_{\beta_*}^2 \right] \right) .
    \end{align*}
    This implies that by choosing $\beta_* = \beta_2$ and substituting $\alpha_* = \alpha_1$, we have
    \begin{align}\label{eq:proof_thm_lforget_equivalent}
        \Lforget(\Delta W) \triangleq \EE\left[ \Norm{\Delta W x}_{2}^2 \right] = \Theta \left( m^{\max\{ \frac{1}{\beta_2} - \frac{1}{2},0 \} }  \Norm{\Delta W}_{\alpha_1,\beta_2}^2 \EE\left[ \Norm{x}_{\alpha_1}^2 \right] \right) 
    \end{align}
    for $\Delta W$ derived with any choice of $\alpha_2,\beta_2 \in [1,\infty]$. 

    Then we consider the SFT update $\Delta W$. Based on \eqref{eq:steepest_descent} and Assumption~\ref{asm:sft_loss}, the $\Delta W$ derived by SFT optimizer $\cA_{\alpha_2,\beta_2}$ holds
    \begin{align*}
        \Delta W(\cA_{\alpha_2,\beta_2}) = \argmin{\Norm{Z}_{\alpha_2,\beta_2} \le \eta} \dotprod{Z}{G} ,
    \end{align*}
    where $\eta$ denotes the update scale. 
    Taking $C \triangleq H_0+\min_{\Norm{Z}_{\alpha_2,\beta_2} \le \eta}\dotprod{Z}{G}$, this is equivalent to 
    \begin{align}\label{eq:proof_thm_sft_optimizer}
        \Delta W(\cA_{\alpha_2,\beta_2}) = \argmin{\dotprod{Z}{G} \le C-H_0} \Norm{Z}_{\alpha_2,\beta_2} = \argmin{\Lsft(Z) \le C} \Norm{Z}_{\alpha_2,\beta_2} ,
    \end{align}
    which can be straightforwardly verified by the KKT condition.
    
    With \eqref{eq:proof_thm_lforget_equivalent} and \eqref{eq:proof_thm_sft_optimizer} in hand, we are ready to prove the theorem. For brevity, we define the following function as the bound derived in \eqref{eq:proof_thm_lforget_equivalent} for any $\alpha,\beta\in [1,\infty]$:
    \begin{align*}
        g(\alpha,\beta) \triangleq m^{\max\{ \frac{1}{\beta} - \frac{1}{2},0 \} }  \Norm{\Delta W(\cA_{\alpha,\beta})}_{\alpha_1,\beta}^2 \EE\left[ \Norm{x}_{\alpha_1}^2 \right],
    \end{align*}
    where $\Delta W(\cA_{\alpha,\beta})$ denotes the SFT update if $\cA_{\alpha,\beta}$ is employed as the SFT optimizer. Thus, it holds that $\Lforget(\Delta W(\cA_{\alpha,\beta})) = \Theta\left( g(\alpha,\beta) \right)$.
    Then we have that for any $\alpha, \beta \in [1,\infty]$ pair yielding $\beta < 2$, it holds that
    $
        g(\alpha_1,\beta_2) =\cO\left( g(\alpha, \beta) \right)
    $
    with $\beta_2 \ge 2$, since we have
    \begin{align*}
        \Norm{\Delta W(\cA_{\alpha,\beta})}_{\alpha_1,\beta} \ge \Norm{\Delta W(\cA_{\alpha,\beta})}_{\alpha_1,\beta_2} \ge \min_{\dotprod{\Delta W}{G} \le C-H_0} \Norm{\Delta W}_{\alpha_1,\beta_2} = \Norm{\Delta W(\cA_{\alpha_1,\beta_2})}_{\alpha_1,\beta_2} 
    \end{align*}
    based on Lemma~\ref{lem:comparing_matrix_norms} and \eqref{eq:proof_thm_sft_optimizer}.

    When taking $\alpha_2=\alpha_1$ and $\beta_2 \ge 2$, we simultaneously have
    \begin{align*}
        g(\alpha_1,\beta_2) =\cO\left( g(\alpha, \beta) \right) \quad \text{and} \quad \Lforget(\Delta W(\cA_{\alpha,\beta})) = \Theta\left( g(\alpha,\beta) \right)
    \end{align*}
    for all $\alpha,\beta$ pair such that $\alpha\in[1,\infty]$ and $\beta \in [1,2]$. Therefore, it holds that
    \begin{align*}
        \Lforget(\Delta W(\cA_{\alpha_1,\beta_2})) = \cO\left( \inf_{\alpha\in[1,\infty],\beta \in [1,2]} \Lforget(\Delta W(\cA_{\alpha,\beta})) \right) ,
    \end{align*}
    which concludes the proof.
    
\end{proof}

\section{Licenses}\label{appendix:licenses}
The Alpaca dataset~\citep{alpaca} (\url{https://huggingface.co/datasets/tatsu-lab/alpaca}) is released under the cc-by-nc-4.0 license. The MetaMathQA~\citep{yu2023metamath} (\url{https://huggingface.co/datasets/meta-math/MetaMathQA}) and Magicoder~\citep{wei2023magicoder} (\url{https://huggingface.co/datasets/ise-uiuc/Magicoder-OSS-Instruct-75K}) datasets and the NanoGPT~\citep{karpathy2022nanogpt} (\url{https://github.com/karpathy/nanogpt}) and Dion~\citep{ahn2025dion,ahn2025dion2} (\url{https://github.com/microsoft/dion}) codebases are under the MIT license. The LlamaFactory codebase~\citep{zheng2024llamafactory} (\url{https://github.com/hiyouga/LlamaFactory}) is under the Apache-2.0 license. The Llama-2-7b-hf model~\citep{touvron2023llama} (\url{https://huggingface.co/meta-llama/Llama-2-7b-hf}) is under the llama2 license.

\end{document}